\documentclass{article}

\usepackage{PRIMEarxiv}

\usepackage{amsmath,amsfonts}
\usepackage{algorithm}
\usepackage{array}
\usepackage[caption=false,font=normalsize,labelfont=sf,textfont=sf]{subfig}
\usepackage{textcomp}
\usepackage{stfloats}
\usepackage{url}
\usepackage{verbatim}
\usepackage{graphicx}
\def\BibTeX{{\rm B\kern-.05em{\sc i\kern-.025em b}\kern-.08em
    T\kern-.1667em\lower.7ex\hbox{E}\kern-.125emX}}

\usepackage{thmtools}
\usepackage{thm-restate}
\usepackage{graphicx,amsmath,amsfonts,amscd,amssymb,color,latexsym}
\usepackage{subfig,wrapfig}
\usepackage{bm}
\usepackage{comment}
\usepackage{mathdots,mathabx}
\usepackage{pgfplots}
\pgfplotsset{compat=1.18}
\usetikzlibrary{positioning}
\usetikzlibrary{decorations}
\usepackage{authblk}
\usepackage[inline]{enumitem}
\usepackage[numbers,sort&compress]{natbib}
\usepackage{xspace}
\usepackage{algpseudocode}

\newtheorem{theorem}{Theorem}[section]
\newtheorem{lemma}[theorem]{Lemma}

\newtheorem{defn}[theorem]{Definition}
\newtheorem{prop}[theorem]{Proposition}
\newenvironment{proof}{\noindent {\bf Proof.} }{\endprf\par}
\def \endprf{\hfill {\vrule height6pt width6pt depth0pt}\medskip}

\pgfkeys{/pgf/decoration/.cd,
         start color/.store in =\startcolor,
         end color/.store in   =\endcolor
}

\pgfdeclaredecoration{width and color change}{initial}{
 \state{initial}[width=0pt, next state=line, persistent precomputation={%
   \pgfmathdivide{50}{\pgfdecoratedpathlength}%
   \let\increment=\pgfmathresult%
   \def\x{0}%
 }]{}
 \state{line}[width=.5pt,   persistent postcomputation={%
     \pgfmathadd{\x}{\increment}%
     \let\x=\pgfmathresult%
   }]{%
   \pgfsetlinewidth{\x/40*0.075pt+\pgflinewidth}%
   \pgfsetarrows{-}%
   \pgfpathmoveto{\pgfpointorigin}%
   \pgfpathlineto{\pgfqpoint{.75pt}{0pt}}%
   \pgfsetstrokecolor{\endcolor!\x!\startcolor}%
   \pgfusepath{stroke}%
 }
 \state{final}{%
   \pgfsetlinewidth{\pgflinewidth}%
   \pgfpathmoveto{\pgfpointorigin}%
   \color{\endcolor!\x!\startcolor}%
   \pgfusepath{stroke}% 
 }
}

\pgfdeclaredecoration{width and color change legend}{initial}{
 \state{initial}[width=0pt, next state=line, persistent precomputation={%
   \pgfmathdivide{100}{\pgfdecoratedpathlength}%
   \let\increment=\pgfmathresult%
   \def\x{0}%
 }]{}
 \state{line}[width=.5pt,   persistent postcomputation={%
     \pgfmathadd{\x}{\increment}%
     \let\x=\pgfmathresult%
   }]{%
   \pgfsetlinewidth{\x/40*0.075pt+\pgflinewidth}%
   \pgfsetarrows{-}%
   \pgfpathmoveto{\pgfpointorigin}%
   \pgfpathlineto{\pgfqpoint{.75pt}{0pt}}%
   \pgfsetstrokecolor{\endcolor!\x!\startcolor}%
   \pgfusepath{stroke}%
 }
 \state{final}{%
   \pgfsetlinewidth{\pgflinewidth}%
   \pgfpathmoveto{\pgfpointorigin}%
   \color{\endcolor!\x!\startcolor}%
   \pgfusepath{stroke}% 
 }
}

\newcommand{\mb}{\boldsymbol}

\newcommand{\bb}{\mathbb}

\newcommand{\ktrue}{k^{*}}
\newcommand{\bbeta}{\mb\beta}
\newcommand{\bbetastar}{\mb\beta^\star}

\newcommand{\deltamin}{\delta_{\min}}
\newcommand{\Deltamax}{\Delta_{\max}}
\newcommand{\Deltamin}{\Delta_{\min}}
\newcommand{\muu}{\mb u}
\newcommand{\mv}{\mb v}
\newcommand{\mx} {\mb x}
\newcommand{\E}{\mathbb{E}}%
\newcommand{\real}{\mathbb{R}}%
\newcommand{\kmeans}{\text{$k$-means}\xspace}
\newcommand{\xmeans}{\text{$X$-means}\xspace}
\newcommand{\kmeanspp}{\text{$k$-means$++$}\xspace}
\newcommand{\kmeansmp}{\text{I-$k$-means$-+$}\xspace}
\newcommand{\kmeansmpx}{\text{I-$k$-means$-+^\star$}\xspace}
\newcommand{\ofm}{one-fit-many\xspace}
\newcommand{\mfo}{many-fit-one\xspace}
\newcommand{\Ofm}{One-fit-many\xspace}
\newcommand{\Mfo}{Many-fit-one\xspace}

\newcommand{\Cgap}{C_{\textup{gap}}}
\newcommand{\vor}{\mathcal{V}}
\newcommand{\supp}{\textup{supp}}
\newcommand{\ddup}{\textup{d}}
\newcommand{\mean}{\textup{mean}}
\newcommand{\ball}{\mathbb{B}}
\newcommand{\vol}{\textup{Vol}}%
\renewcommand{\paragraph}[1]{\textbf{#1.}}

\begin{document}

\title{A Geometric Approach to $k$-means}

\author[$\Diamond$]{Jiazhen Hong}
\author[$\sharp$]{Wei Qian}
\author[$\star$]{Yudong Chen}
\author[$\Diamond$]{Yuqian Zhang}

\affil[$\Diamond$]{Department of Electrical \& Computer Engineering, Rutgers University}
\affil[$\sharp$]{School of Operations Research and Information Engineering, Cornell University}
\affil[$\star$]{Department of Computer Sciences, University of Wisconsin-Madison}
\date{}

\maketitle

\begin{abstract}
\kmeans clustering is a fundamental problem in many scientific and engineering domains. The optimization problem associated with \kmeans clustering is nonconvex, for which standard algorithms are only guaranteed to find a local optimum. Leveraging the hidden structure of local solutions, we propose a general algorithmic framework for escaping undesirable local solutions and recovering the global solution or the ground truth clustering. This framework consists of iteratively alternating between two steps: (i) detect mis-specified clusters in a local solution, and (ii) improve the local solution by non-local operations. We discuss specific implementation of these steps, and elucidate how the proposed framework unifies many existing variants of \kmeans algorithms through a geometric perspective. We also present two natural variants of the proposed framework, where the initial number of clusters may be over- or under-specified. We provide theoretical justifications and extensive experiments to demonstrate the  efficacy of the proposed approach. 
\end{abstract}


\vspace{-.06in}
\section{Introduction}
\label{sec:intro}
% --- root file: removing k-means local min algorithm ---
\vspace{-.05in}

Clustering is a fundamental problem across machine learning, computer vision, statistics and beyond. The general goal of clustering is to group a large number of (potentially high dimensional) data points into a few clusters, each containing similar data points. Many clustering criteria have been proposed. One of the most widely used criteria is the $k$-means formulation, where one aims to find $k$ cluster centers such that the sum of squared distances between each data point and its nearest cluster center is minimized. The most popular algorithm for \kmeans is Lloyd’s algorithm \cite{lloyd1982least}, which is often referred to as the \kmeans algorithm. This algorithm iteratively updates the location of cluster centers and the cluster assignment for each data point. Minimizing the \kmeans criterion is a nonconvex optimization problem. Consequently, Lloyd's and other local search algorithms are sensitive to choice of the initial clustering and in general only guaranteed to find a local solution. 

With decades of extensive research and application, various improved algorithms have been proposed for \kmeans to address the sub-optimality of local solutions. One line of algorithms are based on careful initialization of the clusters. For example, the celebrated $k$-means++ initialization~\cite{arthur2007k} employs a probabilistic initialization scheme such that the initial cluster centers are spread out. See~\cite{celebi2013comparative} for a comprehensive review of different initialization methods. 
%\cite{dasgupta2007probabilistic} considers an over parametrization of the initial center estimates such that there exists at least one one cluster center initiated around each true cluster.
Another line of work focuses on fine-tuning a local solution to produce a better solution, using various heuristics based on empirical observations of the properties of local solutions \cite{pelleg2000x,muhr2009automatic,morii2006clustering,lei2016robust,franti2000randomised,franti2006iterative,ismkhan2018ik}. However, in the absence of a precise characterization of these properties, little can be guaranteed for the performance of these heuristics.
%Recent research on the landscape of nonconvex optimization problems \cite{} manifest a surprising phenomenon that for certain nonconvex optimization problem, {\em every local solution is global}. This observation of the local solution corroborate the wide application of nonconvex formulations and algorithms. These nonconvex problems include complete dictionary learning, Gaussian phase retrieval, and clustering problems with $2$ cluster centers \cite{}. However, the scope of nonconvex problems whose local minimizer is always global is scarce compared to total range of nonconvex optimization problems. Non-global local minimizers do prevailingly exist. Fortunately, even if the most notoriously nonconvex deep neural network do render some reasonable solutions in practice. Admittedly, this requires years of experimental exploration from researchers. 

On the theory side, recent years have witnessed exciting progress on demystifying the structure of local solutions in certain nonconvex problems \cite{sun2016complete,sun2018geometric,ge2016matrix,zhang2018structured,qu2019analysis,ge2020optimization,zhang2020symmetry,jin2021unique}, including \kmeans and related clustering problems. %\kmeans, or clustering problem, as one of the most widely applied nonconvex formulation received extensive investigation as well. 
%In particular, \kmeans and related clustering problems have received extensive investigation. 
It is known that when the data are sampled from two identical spherical Gaussians, the Expectation-Maximization (EM) algorithm with random initialization recovers the ground truth solution~\cite{xu2016global, daskalakis2017ten,qian2019global}. Similar results hold for Lloyd's algorithm when the two Gaussians  satisfy certain separation conditions~\cite{chaudhuri2009learning}. However, as soon as the number of Gaussian components exceeds two, additional local solutions emerge, whose quality can be arbitrarily worse than the global optimum~\cite{jin2016local}. Recent work has established an interesting positive result: under some separation conditions, all local solutions share the same geometric structure that provides partial information for the ground-truth, under both the \kmeans formulation~\cite{qian2021structures} and the maximum likelihood formulation~\cite{chen2020likelihood}.
%When the data sampled are drawn from only two identical clusters, every local solution is global \cite{balakrishnan2017statistical, xu2016global, daskalakis2017ten, qian2019global, kwon2019global}. In a more general $k$-means clustering problem with $k\ge 3$ clusters, \cite{qian2020structures} proved that any local solution is a structured version of the ground truth under some separation conditions . 

%In this paper, we are concerned with exploiting algorithmic implications drawn from the structure of local solutions, based on which we propose (1) a general algorithmic framework to recover the global minimizer (ground truth) from any local minimizer, or (2) execute \kmeans with over-parametrized number of clusters and then merge redundant clusters. The geometry-inspired framework is non-probabilistic and does not depend on the initialization. We prove that the proposed framework requires linear number of \kmeans executions to recover the ground truth under conditions, and demonstrate that this framework perform robustly on more challenging benchmark datasets more with numerous experiments.

In this paper, we exploit the algorithmic implications of the above structural results on the geometry of local \kmeans solutions. We propose a general algorithmic framework for recovering the global minimizer (or ground truth clusters) from a local minimizer. Our framework consists of iterating two steps: 
%(i) identify and tentatively correct the mis-specified clusters in a local solution and (2) refine the tentative solution by another execution of \kmeans algorithm. 
(i) detect mis-specified clusters in a local solution obtained by Lloyd's algorithm, and (ii) improve this local solution by non-local operations. This geometry-inspired framework is non-probabilistic and does not rely on a good initialization. Under certain mixture models with $k$ clusters, we prove that this method recovers the ground truth in $O(k)$ iterations, whereas standard Lloyd's algorithm would require $ e^{\Omega(k)} $ random initializations to achieve the same. Our framework is flexible and provides justifications for many existing heuristics. It can be naturally extended to settings where the initial number of clusters is mis-specified. Extensive experiments demonstrate that our approaches perform robustly on challenging benchmark datasets.

\vspace{-.06in}
\section{Structure of Local Solutions}
\label{sec:geo}
\vspace{-.05in}

We consider the \kmeans problem under a mixture model with $\ktrue$ components: each data point $\mx$ is sampled \textit{i.i.d.}\ from a true density $f^*:=\frac{1}{\ktrue} \sum_{s=1}^{\ktrue} f^*_s$, where  $f^*_s$ is the density of the $s$-th component with mean $\bbeta_s^* \in \real^d$. Under this generative model, the population \kmeans objective function is 
\begin{align}
G(\bbeta):=\mathbb{E}_{\mx \sim f^*} \min_{j\in [k]} \|\mx - \bbeta_j\|^2,  \label{eq:kmeans}
\end{align}
where $ \bbeta = (\bbeta_1, \ldots, \bbeta_k) $ denotes $ k $ fitted cluster centers, with $ k $ potentially different from $ \ktrue $, and $[k]:=\{1,2,\ldots, k\}$.
The objective function $G$ is non-convex, and standard algorithms like Lloyd's only guarantee finding a local minimizer. 

Despite non-convexity, a recent work~\cite{qian2021structures} shows that all local minima have the same geometric structure. In particular, under some separation condition, for every local minimizer $\bbeta$, there exists an association map $\mathcal{A}$ between a partition of the fitted centers $\{\bbeta_s\}_{s\in [k]}$ and a partition of the true centers $\{\bbetastar_s\}_{s\in [\ktrue]}$, such that each center must participate in exactly one of three types of association:
\begin{enumerate} 
\item \textbf{\em{\Ofm} association}: A fitted center $\bbeta_i$ is close to the average of several true cluster centers $ \left\{\bbetastar_j\right\}_{j\in S}$ for some $S\subseteq [\ktrue]$. That is, $\mathcal{A}(\left\{\bbeta_i\right\}) = \{\bbetastar_j \}_{j\in S}$. %Consequently, the corresponding fitted cluster center is close to the mean of contained true cluster centers. 
\item \textbf{\em{\Mfo} association}: Several fitted centers $\{ \bbeta_i\}_{i\in T}$ are simultaneously close to a true center $\bbetastar_j$ and thus split the corresponding true cluster, for some $T\subseteq [k]$ and $j\in [\ktrue]$. That is, $\mathcal{A}(\{ \bbeta_i\}_{i\in T}) = \{\bbetastar_j\}$. %and the centers of those fitted clusters are all close to the true cluster center.
\item \textbf{\textit{Almost empty} association}: A fitted center $\bbeta_i$ is not associated with any true cluster, and the corresponding fitted cluster has almost no data points. That is, $\mathcal{A}(\{ \bbeta_i\}) = \emptyset$.
\end{enumerate}

\begin{figure} 
    \centering\includegraphics[width=0.6\textwidth]{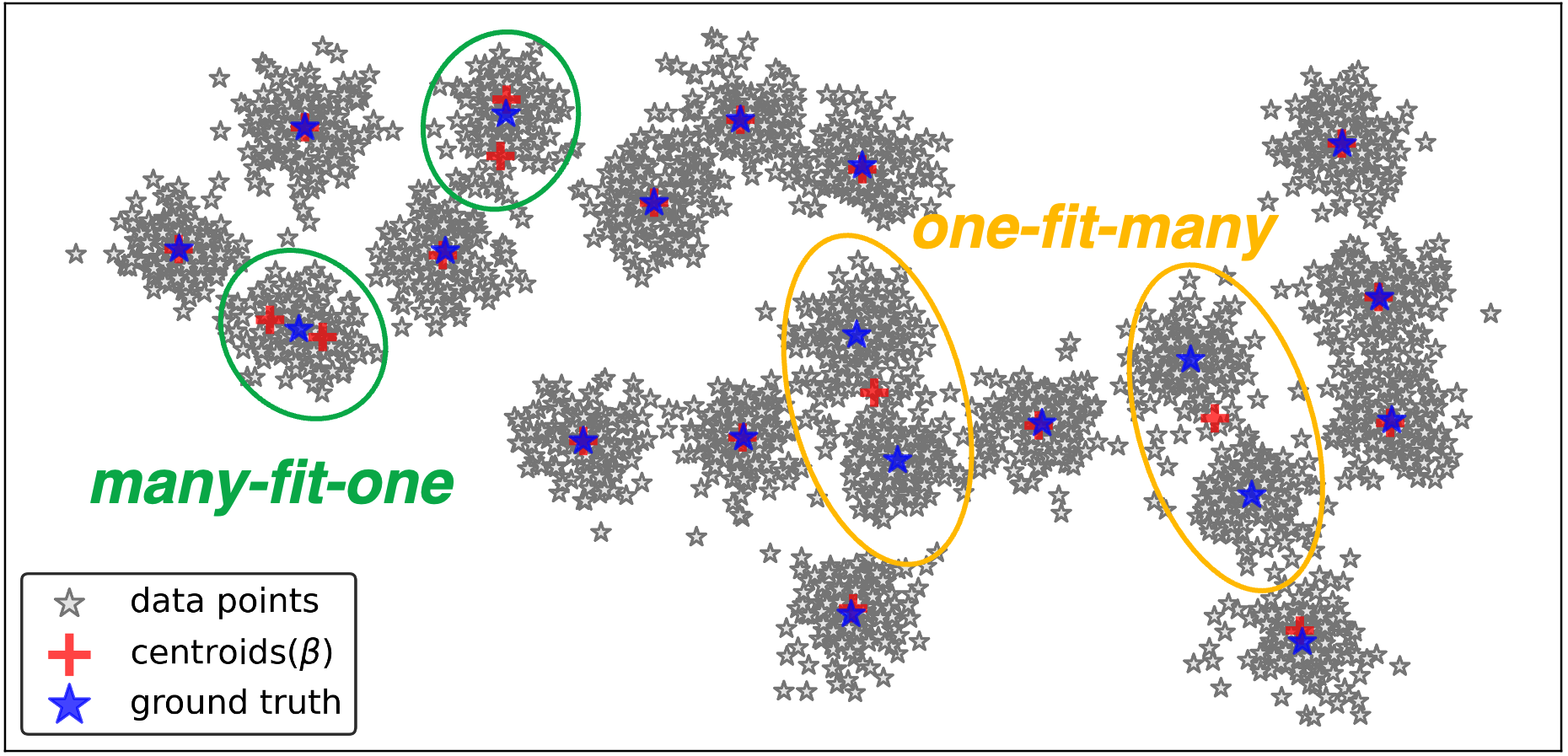} 
    % \vspace{-.1in}
    \caption{The \ofm and \mfo association relationships in a local minimizer of the \kmeans problem.} \label{fig:1}
    % \vspace{-.12in}
    \end{figure}

Figure~\ref{fig:1} illustrates these associations between the fitted centers in a local minimizer and the ground truth clusters. 

With the above characterization, we can deduce some geometric properties for each type of association within a local minimizer, particularly when the true clusters are separated and have identical shapes. For simple exposition, we start with the \emph{Stochastic Ball Model} (see Section 2.1 of~\cite{qian2021structures}), in which the mixture component $f_s^*$ satisfies
\begin{align}
f_{s}^*(\mx)=\frac{1}{\vol(\ball_{s}(r))} 1_{\ball_{s}(r)}(\mx),\quad s\in [\ktrue] \label{eq: SBM}, 
\end{align}
where $\ball_{s}$ denotes a ball with radius $r$ centered at $\bbeta_s$. In this case, we make the following observations.

\paragraph{Properties of \ofm association}
A fitted center with a \ofm association is approximately at the average center of multiple balls, thus the mean in-cluster $\ell_2$ distance to this fitted center is lower bounded by the minimum separation of the balls. On the other hand, for a fitted center with a \mfo association, the associated fitted cluster is contained in a ball, thus the mean in-cluster $\ell_2$ distance to that fitted center is upper bounded by the radius of the ball. When the balls are well-separated from each other, we infer that a fitted cluster with \ofm association has higher mean in-cluster $\ell_2$ distance.

\paragraph{Properties of \mfo association}
Since a fitted center with a \mfo association is contained in a ball, the pairwise distance between two such fitted centers that are associated with the same ball, is lower bounded by the radius of the ball. On the other hand, the distance between these fitted centers and any other fitted center not associated with the same ball, is lower bounded by the separation of the balls. We infer that the fitted centers associated to the same ball is characterized by a  small pairwise distance. 

\paragraph{Properties of \textit{almost empty} association} 
A fitted cluster with an \textit{almost empty} association has a negligible measure by Theorems 1 and 2 in~\cite{qian2021structures}. This means this cluster usually contains very few data points. For example, in an extreme case, some $\bbeta_j$ can be far away from all the data points and has an empty association with the data. We usually consider a \emph{non-degenerate} local minimum solution, in which almost empty associations do not occur.  

The above properties of the fitted clusters with \ofm and \mfo associations are derived under the ball models. In general, they may depend on the structure of the underlying data. As the properties for \ofm and \mfo associations are distinct, they can be leveraged to identify the exact type of association. Consequently, various methods can be designed to eliminate these associations and refine the fitted clusters. Since these associations are the only hurdles to recovering a global solution, eliminating them  helps escaping a local minimum solution. We pursue this idea in the next section.

\vspace{-.06in}
\section{From Structure to Algorithms}
\label{sec:alg}
\vspace{-.05in}

Motivated by the above geometric structure\footnote{While the geometric structure is established for the population \kmeans formulation in \cite{qian2021structures}, it can be shown that they are also present in the finite sample case.}---namely, the presence of \ofm and \mfo associations---in the local minimum solutions of \kmeans, we propose a general algorithmic framework that aims to escape local minimum solutions by detecting and correcting these undesirable associations.\footnote{For simplicity, we assume the local minimum is non-degenerate. In practice, degenerate local minima can usually be eliminated easily by examining the number of data points contained in a fitted cluster. } 

The proposed framework is based on (a) detecting \ofm and \mfo associations in the current solution, and (b) splitting a cluster with an \ofm association while merging clusters with a \mfo association. We call this general framework \emph{Fission-Fusion \kmeans}. After describing the framework (Section~\ref{sec:framework}), we discuss several concrete methods for detecting \ofm and \mfo associations (Section~\ref{sec:al.proposed}). Viewing \ofm and \mfo association as local model mis-specification, we further consider natural extensions of the framework, which allows one to start with any number $k$ of fitted clusters with $k\neq \ktrue$ (Section~\ref{sec:misspecification}). In addition, we discuss other related algorithmic approaches in literature and connect them to our framework (Section \ref{sec:al.discussion}).

\vspace{-.03in}
\subsection{Fission-Fusion \kmeans}\label{sec:framework}
\vspace{-.03in}

The proposed framework, Fission-Fusion \kmeans (FFkm), is presented in Algorithm~\ref{alg:1}. FFkm aims to iteratively improve the \kmeans solution. Each iteration of FFkm consists of four operations:
\begin{description}
    \item [Step 1] Detects a fitted cluster of \ofm association.
    \item [Step 2a] Replaces the fitted center with two centers from the 2-means solution (the Fission step);
    \item [Step 2b] Detects a pair of fitted clusters with a \mfo association and then merges these two fitted centers into one center (the Fusion step);
    \item [Step 3] A Lloyd's \kmeans step is used to update the modified solution.
\end{description}
Figure~\ref{fig:2} illustrates the above procedure. This procedure is iterated until the \kmeans objective no longer decreases. 
A visualization of each step of Algorithm~\ref{alg:1} (FFkm) is provided in Appendix~\ref{app:FFkm_detail}.

\begin{figure}
    \centering\includegraphics[width=0.3\textwidth]{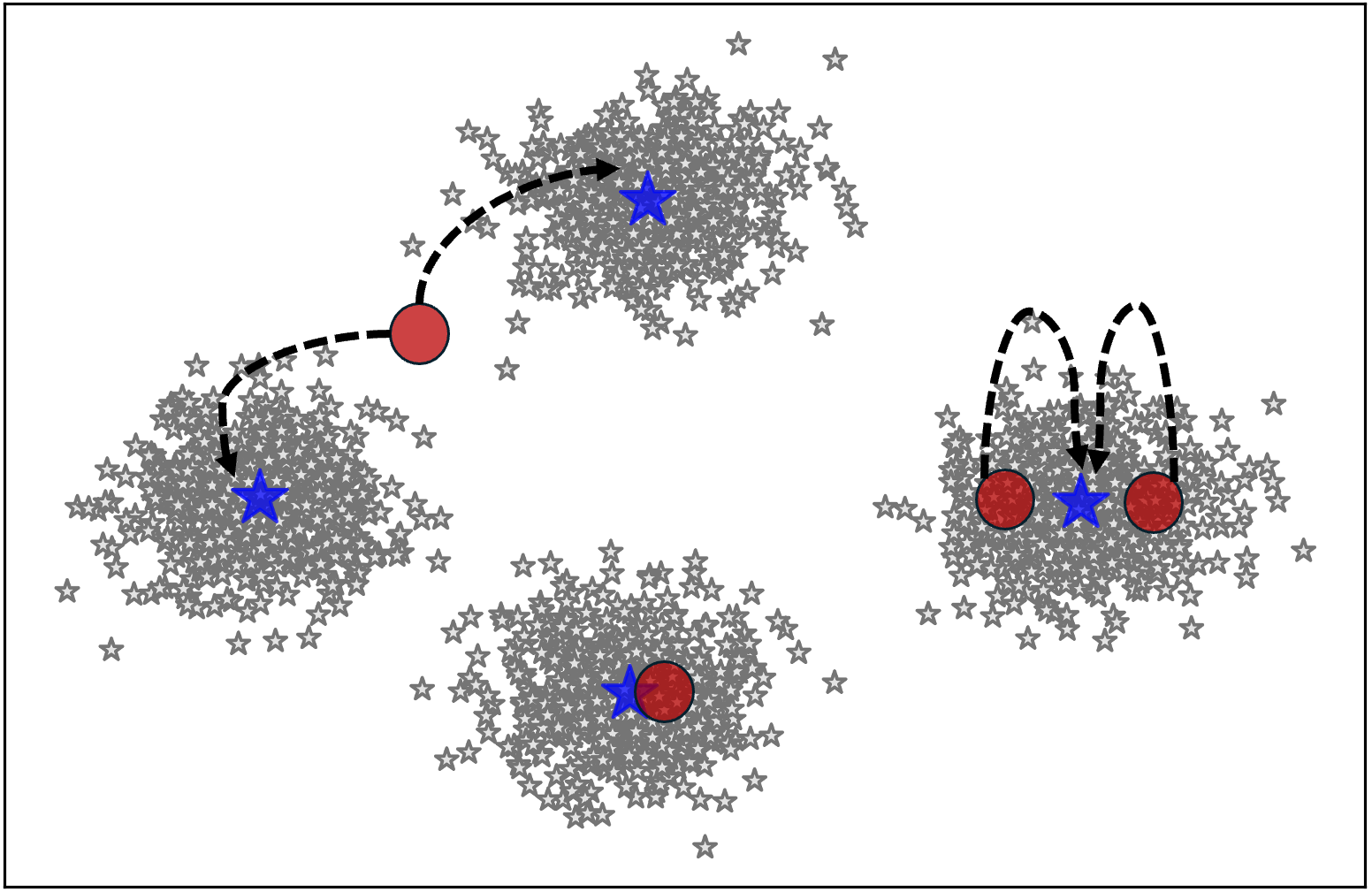}
    % \vspace{-.1in}
    \caption{Illustration of the Fission-Fusion \kmeans algorithm.}\label{fig:2}
    % \vspace{-.12in}
\end{figure}

\begin{algorithm}
    \caption{Fission-Fusion \kmeans (FFkm)} \label{alg:1}
    \hspace*{\algorithmicindent} \textbf{Input:} 
    data $\mathcal{D}$, number of fitted clusters $k$, initial solution $\mb\beta^{(\frac{1}{2})}\in{\real^{d\times k}}$, maximum number of iterations $L$. \\
    \hspace*{\algorithmicindent} \textbf{Output:} 
    $\mb\beta^{(L)}$
    \begin{algorithmic}[1]
        \State 
        Using $\mb\beta^{(\frac{1}{2})}$ as an initial solution, run Lloyd's algorithm to obtain a local minimum $\mb\beta^{(1)}$ with \kmeans objective value $G^{(1)}$. Set $G^{(0)}=\infty$ and $\ell=1$.
        \While{$\ell \le L$} 
            \State \textbf{Step 1:} Detect a cluster with tentative \ofm association, whose center is $\beta_{(1)}^{(\ell)}$.
            \State \textbf{Step 2:} Compute $\mb\beta^{(\ell + \frac{1}{2})}$ from $\mb\beta^{(\ell)}$ using the following procedure:
            \State
             - \textbf{Step 2a:} Split the center $\beta_{(1)}^{(\ell)}$ into two centers;
            \State
             - \textbf{Step 2b:} Detect two clusters with tentative \mfo association with the same true cluster, whose centers are $\beta_{(2)}^{(\ell)}$ and  $\beta_{(3)}^{(\ell)}$. Merge $\beta_{(2)}^{(\ell)}$ and  $\beta_{(3)}^{(\ell)}$ into one center.
            \State
            \textbf{Step 3:} Using $\mb\beta^{(\ell + \frac{1}{2})}$ as an initial solution, run Lloyd's algorithm to obtain a local minimum $\mb\beta^{(\ell+1)}$ with \kmeans objective value $G^{(\ell+1)}$.
            \State If $G^{(\ell+1)} \ge G^{(\ell)}$, set $\mb\beta^{(L)}:=\mb\beta^{(\ell)}$, terminate.
            \State $\ell \leftarrow \ell+1$
        \EndWhile
 %       \State Return $\mb\beta^{(L)}$.
    \end{algorithmic}
\end{algorithm}

Each iteration of FFkm maintains an invariance of the total number of fitted number of clusters: in Step 2a, the total number of fitted clusters is increased to $k+1$; in Step 2b, the total number of fitted clusters is decreased to $k$. Moreover, Step 3 guarantees that the output solution has a \kmeans objective value no worse than the input solution. 
FFKm is a general framework and works as long as the \ofm association and \mfo association can be correctly identified. One has the flexibility to adopt various methods for detecting \ofm association in Step 1 and \mfo association in Step 2b, and the best choices of these methods may be dependent on the data. In Section \ref{sec:al.proposed} we discuss several such methods, which harness the geometric properties of a local solution. 

\vspace{-.03in}
\subsubsection{Theoretical Guarantees}\label{sec:theory}
\vspace{-.03in}

We provide theoretical analysis for the proposed framework under the stochastic ball model~\eqref{eq: SBM}. These results illustrate the working mechanism of Fission Fusion \kmeans. 

For any current local minimum solution $\mb\beta^{(\ell)}$, there are two possibilities: either $\mb\beta^{(\ell)}$ is already a global optimal solution, or it is a local minimum with suboptimal objective value.  
In the first case, the algorithm simply returns a global optimal solution. In the second case, the current local solution $\mb\beta^{(\ell)}$ must contain at least one \ofm association, as shown in Theorem 1 of \cite{qian2021structures}. The Fission step (Step 2a) ensures that in the new solution $\mb\beta^{(\ell+1)}$, two (split) centers fit multiple (at least two) true clusters, which are contained in the cluster with \ofm association detected in Step 1. In particular, restricting to these true clusters, the \kmeans objective value at $\mb\beta^{(\ell+1)}$ strictly decreases. On the other hand, the Fusion step (Step 2b) reduces the number of centers to fit that single true cluster with which at least two fitted clusters are associated in $\mb\beta^{(\ell)}$. Restricting to this true cluster, the \kmeans objective value at $\mb\beta^{(\ell+1)}$ may increase compared with that evaluated at $\mb\beta^{(\ell)}$. One crucial observation here is that the decrement of the \kmeans objective value from the Fission step must exceed the increase of that from the Fusion step, by at least a constant. Therefore, Fission Fusion \kmeans must terminate at global optimal solution in a finite number of steps. 

The above argument is made precise in Theorem \ref{thm:main}.
\begin{restatable}[Main Theorem]{thm}{main}
\label{thm:main}
Let $\left\{\mb\beta^\star_{i}\right\}_{i\in[\ktrue]}$ be $\ktrue$ unknown
centers in $\real^{d}$, with maximum and minimum separations 
\begin{align*}
&\Delta_{\max} :=\max_{i,j\in[\ktrue]}\left\Vert \mb\beta^\star_{i}-\mb\beta^\star_{j}\right\Vert ,\\
&\Delta_{\min} :=\min_{i\neq j\in[\ktrue]}\left\Vert \mb\beta^\star_{i}-\mb\beta^\star_{j}\right\Vert .
\end{align*}
Suppose the data $\mx_{1},\ldots,\mx_{n}\in\real^{d}$ is generated
independently from the stochastic ball model~\eqref{eq: SBM}. Assume that $\frac{\Deltamin}{r}\ge 30$. With probability at least $1-2\ktrue\exp\left(-\frac{n}{2k^{*2}}\right)$, Algorithm \ref{alg:1} with $k=\ktrue$ terminates in $O\left(\ktrue \cdot\frac{\Deltamax^2}{\Deltamin^2}\right)$ iterations and outputs the global minimizer $\mb\beta^\star$.
\end{restatable}

Under the above setting, Algorithm \ref{alg:1} recovers the ground truth clusters with a linear (in $ \ktrue $) number of  executions of the Lloyd's algorithm.\footnote{Lloyd's algorithm itself takes polynomially many steps to terminate at a local solution under data generative models \cite{arthur2006slow}.}  
In sharp contrast, executing the Lloyd's algorithm alone from random initialization converges to the ground truth $\bbetastar$ with an exponentially small probability, hence it requires an exponential number of executions to find $\bbetastar$. This is shown in Theorem \ref{thm:converse} below.

\begin{restatable}[Lloyd's Converges to Bad Locals]{thm}{converse}
\label{thm:converse}
Consider the stochastic ball model setting. Let $\bbeta^{(t)}$ be the $t$-the iterate of the Lloyd's algorithm starting from $k$ random initial centers uniformly sampled from the data.  There exists a universal
constant $c$, for any $k\ge3$ and any constant $\Cgap>0$, such
that there is a well-separated stochastic ball model with $k$ true
centers satisfying
\begin{align*}
\bb P\left[\forall t\ge0:\frac{G(\bbeta^{(t)})-G(\bbetastar)}{G(\bbetastar)}\ge\Cgap\right]\ge1-e^{-ck},
\end{align*}
where $G$ is the \kmeans objective defined in Eq.\eqref{eq:kmeans}.
\end{restatable}

We defer the proofs of above theorems to the Appendix. 

\vspace{-.03in}
\subsection{Detection Subroutines}\label{sec:al.proposed}
\vspace{-.03in}

We propose several subroutines to detect \ofm association and \mfo association utilizing the geometric properties of the local solutions described in Section~\ref{sec:geo}.

\vspace{-.03in}
\subsubsection{Detect \ofm: Standard Deviation (SD)}\label{sec:al.proposed-sd}
\vspace{-.03in}

For each $i$-th fitted cluster with center $\bbeta_i$, we compute the mean squared $\ell_2$ distance to its center:
\begin{align}
\label{eq:al.std}
&\sigma_i^2:= \frac{1}{|C_i|}\sum_{j: \mx_j \in C_i} \|\mx_j-\bbeta_i\|^2, \quad \text{where}\\
&C_i = \big\{\mx_j\in \mathcal{D}: \|\mx_j - \bbeta_i\|\le \|\mx_j - \bbeta_{i'}\| \; \forall \; i'\neq i\big\}.\nonumber
\end{align}
The subroutine outputs $i^*$-th cluster that attains the maximal mean squared distance $i^*: = \text{argmax}_{i\in[k]} \sigma_i^2$. 

As discussed in Section~\ref{sec:geo}, when the true clusters are identical in size, a fitted cluster with a \ofm association contains multiple true clusters, thus having a larger mean squared distance. When the true clusters have varying sizes, we can adapt the above process accordingly. For example, before computing the mean squared distance for each cluster, we can normalize each cluster such that the radius (the maximal distance between a data in the cluster to the cluster center) of each fitted cluster is the same.  For a fitted cluster with a \ofm association, the mass of the data points will concentrate near the boundary after normalization, and will have a larger mean squared distance. 

\vspace{-.03in}
\subsubsection{Detect \ofm: $\epsilon$-Radius (RD)}\label{sec:al.proposed-radius}
\vspace{-.03in}

Fix $\epsilon > 0$. For each fitted cluster $i$, we compute the percentage of points contained in $\mathbb{B}_{\epsilon}(\bbeta_i)$, which denotes the ball centered at $\bbeta_i$ with radius $\epsilon$, among all the data contained in the fitted cluster $i$:
\begin{align}
\label{eq:al.radius}
p_i:= \frac{|B_i|}{|C_i|},\quad B_i = \left\{\mx_j: \|\mx_j - \bbeta_i\|\leq \epsilon, \; \mx_j \in C_i \right\}.
\end{align}
The subroutine outputs the $i^*$-th cluster that attains the smallest $B_i$ such that $i^*: = \text{argmin}_{i\in [k]} B_i$. 

For a fitted cluster with \ofm association, its center $\bbeta_i$ is in the middle of several true clusters. There are two possibilities, either there is no true cluster near the fitted center, or the fitted center coincides with a true cluster center. In the previous case, the set $B_i$ is almost empty as $\bbeta_i$ is not close to any true cluster when there are sufficient separation among the true clusters. In the latter case, the set $B_i$ has a small cardinality. However, $|C_i|$ is big as it contains multiple true clusters. In both cases, the ratio will be smaller for a cluster with a \ofm association (compared with a cluster with a \mfo association).  

\vspace{-.03in}
\subsubsection{Detect \ofm: Total Deviation (TD)}\label{sec:al.proposed-td}
\vspace{-.03in}

For each $i$-th fitted cluster with center $\bbeta_i$, we compute the summation of $\ell_2$ distance to its center:
\begin{align}
\label{eq:al.std2}
&v_i^2:= \sum_{j: \mx_j \in C_i} \|\mx_j-\bbeta_i\|^2, \quad \text{where}\\
&C_i = \big\{\mx_j\in \mathcal{D}: \|\mx_j - \bbeta_i\|\le \|\mx_j - \bbeta_{i'}\| \; \forall \; i'\neq i\big\}.\nonumber
\end{align}
The subroutine outputs $i^*$-th cluster that attains the maximal mean squared distance $i^*: = \text{argmax}_{i\in[k]} v_i^2$.

Compared with the standard deviation detection method, the total deviation is an unnormalized version of standard deviation. Indeed, the total deviation approximates the improvement in the $\kmeans$ objective value when a single fitted cluster is fitted with two centers; see section 3.1 of \cite{ismkhan2018ik}. This coincides with the observation that the \kmeans objective function decreases more when a fitted component with \ofm association is split into two centers in the stochastic ball model. 

\vspace{-.03in}
\subsubsection{Detect \mfo: Pairwise Distance (PD)}\label{sec:al.proposed-pairwise}
\vspace{-.03in}

For each pair of fitted cluster $(i,j)$, $i \neq j$, we compute the pairwise $\ell_2$ distance between fitted cluster center $\bbeta_i$ and $\bbeta_j$: $d_{i,j}:=\|\bbeta_i - \bbeta_j\|.$
The subroutine outputs $i_*$-th and $j_*$-th clusters whose pairwise distance attains the minimal:
\begin{align}
(i_*,j_*): = \text{argmin}_{(i,j), i\neq j}  d_{i,j}.
\end{align}
The method is also based on the inferred geometric properties in Section~\ref{sec:geo}: when  true clusters have similar shape or size, the pairwise distance between the fitted clusters with \mfo association is smaller.  %When the true clusters have varying size, we can scale the pairwise distance by the maximal distance between two fitted clusters. 

\vspace{-.03in}
\subsubsection{Detect \mfo: Objective Increment (OI)}\label{sec:al.proposed-obj-change}
\vspace{-.03in}

For each $i$-th fitted center, let us consider a modified $\kmeans$ clustering solution $\widehat{\bbeta}^{(i)}=(\bbeta_1,\ldots, \widehat{\bbeta_i},\ldots, \bbeta_k)$ by removing the $i$-th center. Denote the corresponding $k$-means objective function as $G_i$, in which we fit $k-1$ centers to the data compared with the original clustering solution. Let $(i^*,j^*)$ be such that
\begin{align*}
i^* =   \text{argmin}_{i}  G_{i},\quad
j^* =   \text{argmin}_{j, j\neq i^*}  \|\bbeta_j - \bbeta_i^*\|.
\end{align*} 
This method coincides with the observation that the \kmeans objective function increases the least when two fitted centers that have \mfo association with the same true center are merged in the stochastic ball model.

\vspace{-.03in}
\subsubsection{Other Detection Procedures}\label{subsec:related-algorithm}
\vspace{-.03in}

The idea of using split and merge type operations in clustering problems can be traced back to as early as the 1960s \cite{ball1967promenade}. This idea has been used to determine the correct number of fitted clusters when $k$ is unknown \cite{pelleg2000x, muhr2009automatic, lei2016robust}, or to escape local solutions when $k$ is known \cite{ueda2000smem, morii2006clustering}. Several criteria for split and merge steps have been proposed in the literature; see Table~\ref{tab:1} for a summary and Appendix \ref{app:related} for more details. 

\begin{table}
\centering\caption{Related Split and Merge criteria (details in Appendix \ref{app:related})} \label{tab:1}
% \vspace{-.15in}
\scalebox{1.00}{
\begin{tabular}{l l l}
\hline \textbf{Algorithm} & \textbf{Split Criteria} & \textbf{Merge Criteria} \\ \hline
\cite{pelleg2000x,muhr2009automatic} &  Reduction in BIC score & BIC score
\\
\cite{lei2016robust}& Max \& Min in-cluster distance & Pairwise distance  
\\
 \cite{morii2006clustering} &  Ratio of objective value with $k$ & Pairwise distance \\ \hline
\end{tabular}}
% \vspace{-.1in}
\end{table}

These existing criteria can be adapted and incorporated into our proposed framework, as we describe below.

The work \cite{pelleg2000x, muhr2009automatic} studies the \xmeans algorithm, which uses the Bayesian Information Criterion (BIC) score with respect to the current solution. A fitted cluster is to be split into two clusters, and a pair of clusters are to be merged, if doing so decreases the BIC score. To adapt the split criterion for detecting \ofm association in our framework, we can output the cluster that attains the maximal reduction in BIC score if it is split into two clusters. To adapt the merge criterion for detecting \mfo association, we can output the pair of clusters that attain the maximal reduction in BIC if they are to be merged.

The algorithm in \cite{lei2016robust} evaluates the intra-cluster and inter-cluster dissimilarity. A fitted cluster is to be split if the intra-cluster dissimilarity exceeds some threshold; a pair of clusters are to be merged if the inter-cluster dissimilarity falls below some threshold. The dissimilarities are measured in Euclidean distance. In particular, the intra-cluster dissimilarity for a fitted cluster is defined as the sum of maximal and minimal distance to that cluster center; the inter-cluster dissimilarity is the pairwise cluster center distance. Note that the merge criterion coincides with the pairwise distance described in Section~\ref{sec:al.proposed-pairwise}.
To adapt the split criterion for detecting \ofm association, we output the cluster with maximal intra-cluster dissimilarity; to detecting \mfo association, we output the pair of clusters with minimal inter-cluster dissimilarity.

The algorithm in \cite{morii2006clustering} aims to split a cluster into $2,\ldots,M$ clusters and compute the ratio of successive \kmeans objectives. The cluster will be split if the minimum of these ratios is smaller than a threshold. In the merge step, it retains the split cluster that is furthest from the neighboring regions and then merges the rest of the split clusters to the neighboring Voronoi regions. We can also adapt the split criterion for detecting \ofm association here --- we can split a cluster into $2$ clusters and compute the ratio between the local \kmeans objective with 2 clusters and the local \kmeans objective with only 1 cluster. Afterwards, we output the cluster that attains the smallest ratio.

\vspace{-.03in}
\subsection{Mis-specification of Initial Number of Clusters and Ablation Study}\label{sec:misspecification}
\vspace{-.03in}

We consider two variants of the proposed FFkm algorithm, where only the fission step or the fusion step is used. Recall the fission/fusion step only increases/decreases the number of clusters. To ensure our algorithm outputs $\ktrue$ clusters at the end, we under-specify the initial number of clusters  ($k<\ktrue$) for Fission-only \kmeans or over-specify ($k>\ktrue$) for Fusion-only \kmeans.
Considering these two variants also serve as an ablation study on the roles of the fission and fusion steps in the proposed algorithm.
 
Note that the structural result in Section~\ref{sec:geo} holds even when  $k\neq\ktrue$, i.e., the numbers of fitted and true clusters are not equal~\cite{qian2021structures}. An interpretation of \ofm association is that an insufficient number of parameters (in this case only one parameter, corresponding to one fitted cluster center) are used to fit multiple true components, resulting in local underfitting. On the other hand, \mfo association happens when too many parameters are used to fit a single component, resulting in local overfitting. When the fitted parameter $k$ is much smaller than the ground truth $\ktrue$, the local solutions are more likely to contain \ofm association. When the fitted parameter $k$ is larger than the ground truth $\ktrue$, the local solutions are more likely to contain \mfo association. %\textcolor{blue}{Based on these insights, there are two natural extensions of FFkm that start with an initial $k$ not necessarily equal to $\ktrue$: the {\em Fission \kmeans} algorithm dealing with under-parametrized setting, and the {\em Fusion \kmeans} algorithm dealing with over-parametrized setting.}

\paragraph{Fission-only \kmeans in Under-specified Setting}
For Fission-only \kmeans, we initially fit less clusters than the true number of clusters, i.e., $k < \ktrue$ and iteratively apply a \ofm detection subroutine and split the corresponding cluster. See Algorithm~\ref{alg:2}. %Note that this procedure is equivalent to FFkm without the Fusion step.
\begin{algorithm}
    \caption{Fission-only \kmeans}  \label{alg:2}
    \hspace*{\algorithmicindent} \textbf{Input:} 
    data points $\mx_1, ... , \mx_n \in \real^{d}$, number of fitted clusters $k$, number of true clusters $\ktrue$\\
    \hspace*{\algorithmicindent} \textbf{Output:} 
    $\mb\beta$
    \begin{algorithmic}[1]
        \State 
        Run Lloyd's algorithm initialized from $k$ randomly selected cluster centers.
        \While{$k> \ktrue$} 
            \State \textbf{Step 1:} Detect a cluster with \ofm association, whose center is $\bbeta_{(1)}$.
            \State \textbf{Step 2:} Split $\bbeta_{(1)}$ into two centers $\bbeta_{(1)}$ and $\bbeta_{(1)'}$ , $k\gets k+1$
            \State \textbf{Step 3:} Run Lloyd's algorithm on $k$ cluster centers initialized at the updated solution.
        \EndWhile
    \end{algorithmic}
\end{algorithm}

\paragraph{Fusion-only \kmeans in Over-specified Setting}
\label{sec:al.overparameterization}
For Fusion-only \kmeans, we initially fit more clusters than the true number of clusters, i.e., $k > \ktrue$ and only apply the \mfo detection subroutine to merge close clusters. See Algorithm~\ref{alg:3}. We defer the experiment results on these two algorithms to Section \ref{sec:expt}.
\begin{algorithm}
    \caption{Fusion-only \kmeans}  \label{alg:3}
    \hspace*{\algorithmicindent} \textbf{Input:} 
    data $\mathcal{D}$, number of fitted clusters $k$, the number of true clusters $\ktrue$\\
    \hspace*{\algorithmicindent} \textbf{Output:} 
    $\mb\beta$
    \begin{algorithmic}[1]
        \State 
        Run Lloyd's algorithm initialized from $k$ randomly selected cluster centers.
        \While{$k> \ktrue$} 
            \State \textbf{Step 1:} Detect two clusters with \mfo association, whose centers are $\bbeta_{(1)}$ and $\bbeta_{(2)}$.
            \State \textbf{Step 2:} Merge $\bbeta_{(1)}$ and $\bbeta_{(2)}$ into one center $\bbeta_{(1,2)}$ by averaging, $k\gets k-1$.
            \State \textbf{Step 3:} Run Lloyd's algorithm on $k$ cluster centers initialized at the updated solution.
        \EndWhile
    \end{algorithmic}
\end{algorithm}

\vspace{-.06in}
\section{Related Work and Connection}\label{sec:al.discussion}
\vspace{-.05in}

Fission Fusion \kmeans (FFkm) is a general framework which iteratively eliminates \ofm and \mfo associations and decreases the \kmeans objective value. This framework allows us to unify many existing algorithmic designs for \kmeans, from the perspective of the structural properties of local solutions. Below, we discuss other variants of \kmeans algorithms in literature; we elucidate their connection to our framework and to the structures of local solutions, and highlight the differences.

\vspace{-.03in}
\subsection{Swap Operation}
\vspace{-.03in}

One variant of our framework is to use a Swap operation, which moves the center of one cluster in \mfo association to the neighborhood of the center with \ofm association; see Figure~\ref{fig:3} for an illustration, which can be compared with Figure~\ref{fig:2}. The Swap operation can also be viewed as performing the Fusion step before the Fission step in the FFkm framework. Using Swap, a cluster with \mfo association and a cluster with \ofm association need to be identified simultaneously. One such randomized procedure is considered in \cite{franti2000randomised}, in which a random center and a random cluster are swapped. Other deterministic procedures have been proposed \cite{franti2006iterative, ismkhan2018ik,kaukoranta1998iterative, frigui1997clustering,franti2009efficiency}. To select a center to be swapped, an objective value based criterion is considered in \cite{franti2006iterative, ismkhan2018ik}; a merge based criterion is used in \cite{kaukoranta1998iterative, frigui1997clustering}. To select a cluster to which a center is moved, an objective value based criterion is considered in \cite{ismkhan2018ik}; other heuristic criteria are proposed, e.g., selecting a cluster with the largest variance \cite{fritzke1997lbg,franti2009efficiency}.  

\begin{figure}
    \centering
    \includegraphics[width=0.3\textwidth]{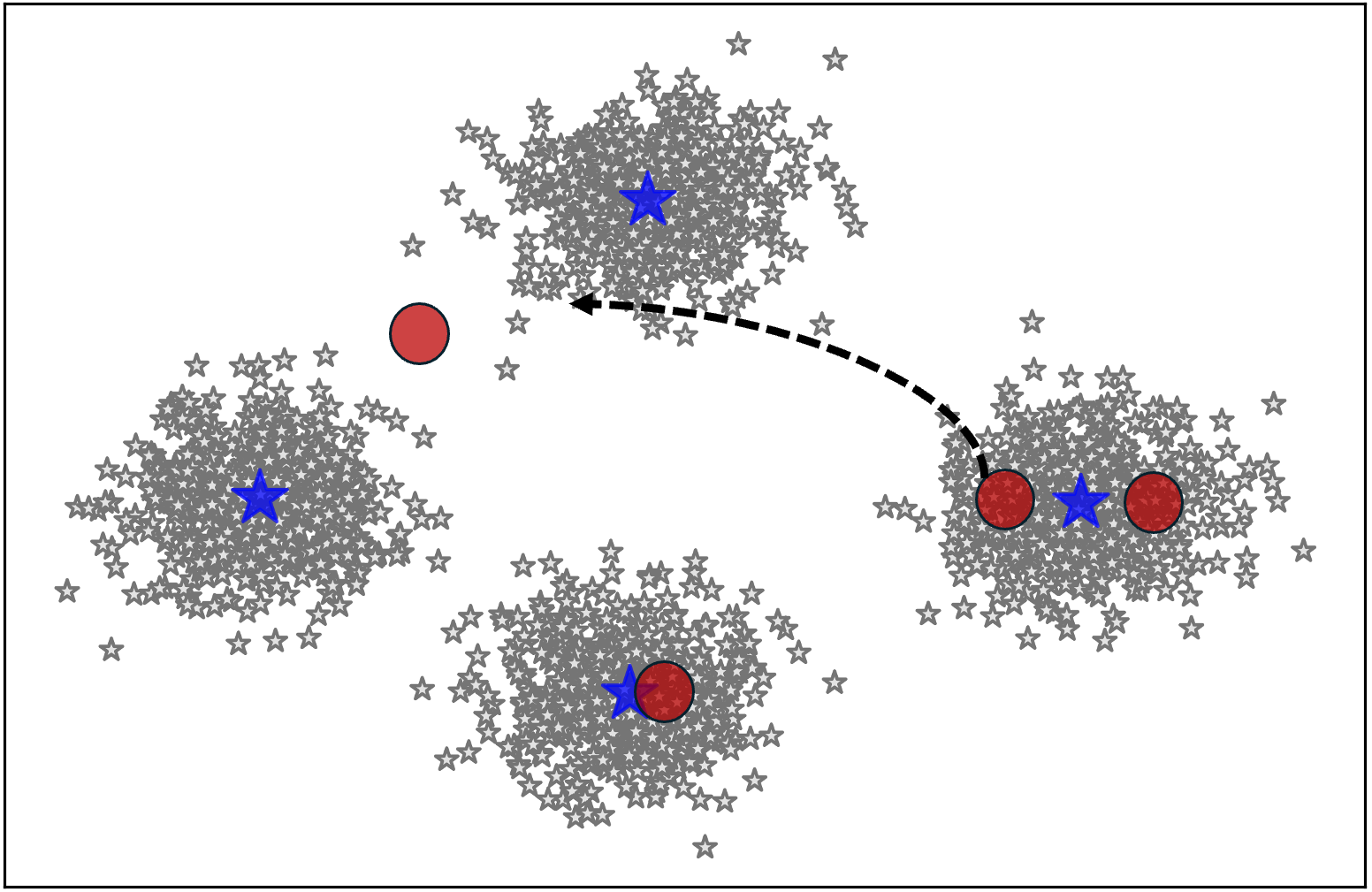} 
    \centering{\caption{Illustration of the Swap operation.} }\label{fig:3}
\end{figure}

\vspace{-.03in}
\subsubsection{Geometry-based versus Objective-based Algorithms}\label{sec:kmsmp}
\vspace{-.03in}

The proposed FFkm approach is \emph{geometry-based}, which escapes local minima by harnessing their geometric properties. In particular, this is the case when FFkm employs the Standard Deviation (SD) and $\epsilon$-Radius (RD) subroutines to detect \ofm, and the Pairwise Distance (PD)  subroutine to detect \mfo.
In contrast, \emph{objective-based algorithms} focus solely on the \kmeans objective value when trying to improve the clustering solution \cite{franti2006iterative, ismkhan2018ik, zhu2021improved}. 
FFkm with the Total Deviation (TD) subroutine for \ofm detection and the Objective Increment (OI) subroutine for \mfo detection, can be classified into this category. 

A representative objective-based algorithm in the literature is \emph{\kmeansmp} \cite{ismkhan2018ik}, which identifies a cluster to be removed (minus) and a cluster to be divided (plus) with the goal of improving the \kmeans objective value. In particular, \kmeansmp finds the ``min-cost'' cluster whose total objective value minus the cluster's partial objective value is minimal, as well as the ``max-gain'' cluster whose new partial objective value after adding one center is maximum. These criteria are similar to those described in Sections \ref{sec:al.proposed-td} and \ref{sec:al.proposed-obj-change}.
One can also view \kmeansmp as a variant of FFkm using the Swap operation discussed above.

In general, one can expect that objective-based algorithms like \kmeansmp perform well for datasets that are balanced, where different clusters have similar numbers of data points. However, real-world datasets often have highly unbalanced clusters. In this case, even when the clusters have well-defined boundaries, objective-based algorithms often overly focus on large clusters (those with many data points) while ignore small clusters. In particular, these algorithms may incorrectly split a large cluster as doing so leads to a local improvement of the objective value, resulting in a local minimum. We corroborate these observations with experiment results on unbalanced datasets in Section \ref{sec:synthetic}, where we find that geometry-based FFkm outperforms objective-based methods like \kmeansmp.

\vspace{-.03in}
\subsection{Additional Related Work}
\vspace{-.03in}

A different direction for improving the quality of the \kmeans solution is to design better initialization schemes. The work by Celebi et al~\cite{celebi2013comparative} provides a comprehensive review of initialization methods. Many of these methods coincide with the intuition of reducing the \ofm association and \mfo association. We discuss a few illustrating examples below; an exhaustive comparison is beyond the scope of the current work. One approach is to sequentially choose the initial centers so that they are spread out, which avoids the \mfo association. To this end, \kmeanspp~\cite{arthur2007k} uses a probabilistic procedure, and maxmin method~\cite{katsavounidis1994new} and Hartigan method~\cite{hartigan1979algorithm} use a deterministic procedure. Astrahan's method~\cite{astrahan1970speech,cao2009initialization} selects centers such that the data near each center has a relative high density and successive centers are far apart from each other. 

The proposed FFkm framework can be viewed as going beyond the initialization step to further improve the clustering solution. In particular, the above existing initialization schemes aim to reduce the \ofm and \mfo associations at the start of the algorithm; our framework reduces them continuously throughout the iterations. Importantly, our framework can be applied on top of any existing initialization schemes.

\vspace{-.06in}
\section{Experiments}
\label{sec:expt}
\vspace{-.05in}

We implement Algorithm~\ref{alg:1}, Fission-Fusion \kmeans, which incorporates the \ofm and \mfo association detection methods described in Section~\ref{sec:al.proposed}. For \ofm detection, we consider the standard deviation (SD), total deviation (TD), and $\epsilon$-radius (RD) methods. For \mfo detection, we include the pairwise distance (PD) method and the objective increment (OI) method. There are six combinations of these subroutines. The resulting FFkm implementations are called FFkm (SD+PD), FFkm (SD+OI), FFkm (TD+PD), FFkm (TD+OI), FFkm (RD+PD), and FFkm (RD+OI), respectively. Our experiments employ the benchmark datasets used in~\cite{franti2018k}. In Section~\ref{sec:real-world}, we consider additional real-world datasets.

%Some additional implementation details are in order. 
For the $\epsilon$-radius (RD) method, the radius of the ball is determined adaptively as follows. We first compute the minimum median $\ell_2$ distance to the cluster centers among all fitted clusters. This distance serves as the base radius $r$. Subsequently, we set the radius to $\delta \cdot r$, where $\delta$ is chosen from $\{0.01, 0.1, 1, 5\}$, with  $\delta = 0.1$ as the default value. 
%%%%%%%%%%%%%%%
%code foot note
%%%%%%%%%%%%%%%
%\footnote{All the code can be found \url{https://github.com/h128jj/Fission-and-Fusion-kmeans.git}.}
\vspace{-.03in}
\subsection{Benchmark Datasets}\label{exp:data}
\vspace{-.03in}

\begin{table}
    \centering
    \caption{Characteristics of the Benchmark Datasets.}\label{tab:2}
    % \vspace{-.15in}
    %\footnote{Source: \href{http://cs.uef.fi/sipu/datasets/}}
    \begin{tabular}{l l l l r}
    \hline 
    \textbf{Dataset} & \textbf{Varying} & \textbf{Size} & \textbf{Clusters} & \textbf{Per cluster} \\ 
    \hline
    A-sets & \#Clusters & $3000$--$7500$ & $20$--$50$ & 150 \\
    S-sets & Overlap & 5000 & 15 & 333 \\
    Dim032 & Dimensions & $1024\times 100$ & 16 & 64 \\
    Birch1 & Structure & 100,000 & 100 & 1000 \\
    Unbalance & Balance & 6500 & 8 & 100, 2000 \\ \hline
    \end{tabular}
    % \vspace{-.12in}
\end{table}

We use the synthetic benchmark datasets from~\cite{franti2018k}, which are  widely employed for assessing clustering algorithms. These datasets have several categories with varying cluster numbers (A-sets), degrees of separation (S-sets), dimensionalities (DIM032), and levels of unbalance (Unbalance). 
For an overview of these datasets' properties, see Table~\ref{tab:2}. For a visual representation, see Appendix~\ref{app:benchmark}.

Below we offer a brief description of these datasets.
\begin{enumerate}
    \item\textbf{A-sets} consist of three sets, $A_1$, $A_2$, and $A_3$ ($A_1 \subset A_2 \subset A_3$), corresponding to $20, 35$ and $50$ spherical clusters in $\real^2$ respectively, all with $20\%$ overlap. %Each cluster contains 150 points.
    \item\textbf{S-sets} contain four sets,  $S_1$, $S_2$, $S_3$ and $S_4$, which correspond to $15$ Gaussian clusters in $\real^2$ with varying overlap percentages of $9\%$, $22\%$, $41\%$ and $44\%$. While most clusters are spherical, a few have been truncated and become non-spherical.
    \item\textbf{Unbalance} includes a single set with eight clusters in $\real^2$, divided into two well-separated groups (left and right). The left group consists of three dense clusters with 2000 vectors each, while the right group comprises five sparse clusters with 100 vectors each.
    \item\textbf{DIM032} features a single set with 16 well-separated Gaussian clusters in $\real^{32}$. \footnote{To prevent artifacts (e.g., a center fitting a single data point) due to small sample sizes, we increased the number of data points from $1024$ to $102400$. Specifically, random sampling was performed from Gaussian distributions with means at the ground truth centers and uniform standard deviations.}
    \item\textbf{Birch1} includes a single set with 100 Gaussian clusters in $\real^2$, with centers arranged in a regular $10\times 10$ grid.
\end{enumerate}

\vspace{-.03in}
\subsection{Evaluation Metrics} \label{sec:evalaution}
\vspace{-.03in}

Three metrics are used for evaluating the clustering quality. 

The first two metrics are based on a modified version of the \textit{centroid index} (CI)~\cite{franti2014centroid}. CI allows one to compare two clustering solutions with different numbers of clusters, as some algorithms like \cite{pelleg2000x, muhr2009automatic} do not necessarily return a solution with $\ktrue$ clusters.
To compute the CI, we first identify the index of the closest ground truth center to each fitted cluster center. Then, we count the total number of ground truth centers whose indices are not mapped to any fitted cluster center in the first step. This count yields the CI, which approximately measures the total number of true centers contained in \ofm associations. It does not penalize \mfo associations since the true center associated with that \mfo association has been identified. A zero CI indicates successful clustering in the sense that all ground truth centers have been identified. 

Based on CI, we consider two more fine-grained metrics.
\begin{enumerate}
\item \textbf{Success rate (SR)}: defined as the percentage of trials in which an algorithm succeeds in returning a zero-CI solution~\cite{franti2018k}. Different trials differ by random initialization and other internal randomness of the algorithm.
\item \textbf{Average missing rate (AMR)}: defined as the mean CI (normalized by the number of true clusters) over multiple trials of an algorithm. Compared to SR, AMR accounts for the quality of the solution when the success rate is not 100\%. A higher AMR indicates a lower solution quality.
\end{enumerate}
When an algorithm assumes knowledge of the number of true clusters $\ktrue$, we further use the relative \kmeans objective value, described below, as a third evaluation metric:

\begin{enumerate}[resume]
\item $\boldsymbol{\rho}$\textbf{-ratio}: defined as the ratio between the objective value of the solution returned by an algorithm and the optimal \kmeans objective value. 
\end{enumerate}

\vspace{-.03in}
\subsection{Results for Benchmark Datasets}\label{exp:result}
\vspace{-.03in}
In Section~\ref{sec:synthetic}, we investigate the differences between geometry-based algorithms and objective-based algorithms (cf.\ Section~\ref{sec:kmsmp}). In Section~\ref{sec:param}, we conduct an ablation study and examine the performance of Fission-only \kmeans and Fusion-only \kmeans (cf. Section~\ref{sec:misspecification}). In Section~\ref{exp:com}, we compare FFkm against other algorithms, including Lloyd's algorithm using both random and $k$-means$++$ initializations, as well as more recent algorithms from~\cite{pelleg2000x, muhr2009automatic, morii2006clustering, lei2016robust, ismkhan2018ik}. 

In Section~\ref{sec:real-world} to follow, we validate the effectiveness of FFkm on real-world datasets.

\vspace{-.03in}
\subsubsection{A Challenging Unbalanced Dataset} \label{sec:synthetic}
\vspace{-.03in}

We use a challenging synthetic dataset to demonstrate the difference between geometry-based algorithms (including variants of FFkm) and objective-based algorithms (including \kmeansmp). The dataset is visualized in Figure \ref{fig:4}. The data points, shown as gray stars, are generated from Gaussian distributions centered at the true centers from the benchmark dataset \textbf{Unbalance}. The three smaller clusters on the left each contain $200$ data points with a standard deviation of $3$; the five larger clusters on the right each contain $2000$ data points with a standard deviation of~$7$.

The following algorithms are considered: the standard Lloyd's \kmeans algorithm, the objective-based algorithm \kmeansmp \cite{ismkhan2018ik}, and FFkm with the aforementioned six combinations of subroutines. 
The original paper \cite{ismkhan2018ik} discusses six versions of \kmeansmp with different initialization schemes and different values of a hyperparameter $\alpha$. For a fair and consistent comparison, we use a re-implemented version 5 of \kmeansmp with $\alpha = 3/4$ and random initialization, which aligns with how we initialize FFkm; we refer to this implementation as \kmeansmpx. 
%finds a pair of Min-Cost and Max-Gain clusters and apply the minus-plus (Swap) operation. Among these pairs, the Gain should always be greater than the Cost, with a coefficient $\alpha = 3/4$ used in the Gain calculation, referred to as version 1 in the original setting. There are six versions of \kmeansmp, each with different settings based on various $\alpha$ values or initialization methods, such as Lloyd's, \kmeanspp, and the method in \cite{ismkhan2017initialization}. For a direct comparison, we have re-implemented \kmeansmp, referred to as \kmeansmpx, based on version 5 with $\alpha = 3/4$ and random initialization to align with the approach used for FFkm. 
Among the six variants of FFkm, we consider FFkm (SD+PD) and FFkm (RD+PD) as geometry-based, FFkm (TD+OI) as objective-based, and FFkm (TD+PD), FFkm (SD+OI), and FFkm (RD+OI) as hybrid combining the geometry- and objective-based approaches. 

\begin{table}
  \centering
  \caption{Experiment results on the challenging synthetic dataset}\label{tab:3}
  % \vspace{-.15in}
  \scalebox{0.8}{
  \begin{tabular}{|c | c |c | c | c | c | c | c |}
  \hline 
  \textbf{Algorithms} & \textbf{Strategy} & \textbf{SR (\%)} & \textbf{AMR} & \textbf{$\rho$-ratio (\%)} & \textbf{Average SSE} & \textbf{Ground Truth SSE} &  \textbf{Time (s)} \\ 
  \hline
  Lloyd \kmeans & objective & 0          & 0.26         & $2.3\pm2.84$  & $2284151.16$ & $991139.25$ & $0.0425$\\
  \hline
  FFkm (SD+PD)   & geometry  & $\bm{100}$ & $\bm{0.00}$  & $\bm{1.00\pm0.00}$ & $\bm{991139.25}$  & $991139.25$ & $\bm{0.0686}$\\
  FFkm (RD+PD)   & geometry  & $\bm{100}$ & $\bm{0.00}$  & $\bm{1.00\pm0.00}$ & $\bm{991139.25}$  & $991139.25$ & $0.0852$\\
  FFkm (TD+PD)   & hybrid  & 2          & 0.12         & $1.11\pm0.02$ & $1099421.01$ & $991139.25$ & $0.1780$\\
  FFkm (SD+OI)   & hybrid  & $\bm{100}$ & $\bm{0.00}$  & $\bm{1.00\pm0.00}$ & $\bm{991139.25}$  & $991139.25$ & $0.0948$\\
  FFkm (RD+OI)   & hybrid  & $\bm{100}$ & $\bm{0.00}$  & $\bm{1.00\pm0.00}$ & $\bm{991139.25}$  & $991139.25$ & $0.1152$\\
  FFkm (TD+OI)   & objective & 2          & 0.12         & $1.12\pm0.02$ & $1106059.38$ & $991139.25$ & $0.0769$\\
  \kmeansmpx     & objective & 0          & 0.13         & $1.17\pm0.05$ & $1155022.97$ & $991139.25$ & $0.1872$\\  
  \hline
  \end{tabular}}
% \vspace{-.1in}
\end{table}

For each algorithm, we conducted $100$ independent trials. The results are summarized in Table~\ref{tab:3}, which present the performance metrics as well as the sum of squared errors (SSE) averaged across trials, the SSE of the ground truth clustering, and the execution time averaged across trials.\footnote{The execution time was recorded on the same machine.} The best results in each column (excluding Lloyd's algorithm) are marked in bold. As observed, the geometry-based algorithms, FFkm (SD+PD) and FFkm (RD+PD), recover the ground truth clustering and achieve a $100\%$ success rate, with FFkm (SD+PD) using fewer iterations and hence the fastest execution time.  Two of the FFkm variants with combined strategies and the geometry-based subroutines SD and RD, also achieve a 100\% success rate. 
In comparison, the objective-based algorithms, FFkm (TD+OI) and \kmeansmpx, failed to recover the ground truth, with success rates of only $2\%$ and $0\%$, respectively. 

\begin{figure*}
  \centering
  % First row of images
  \subfloat[Initial Solution by Lloyd \kmeans]{  
    \includegraphics[width=0.328\linewidth, height=0.17\textheight]{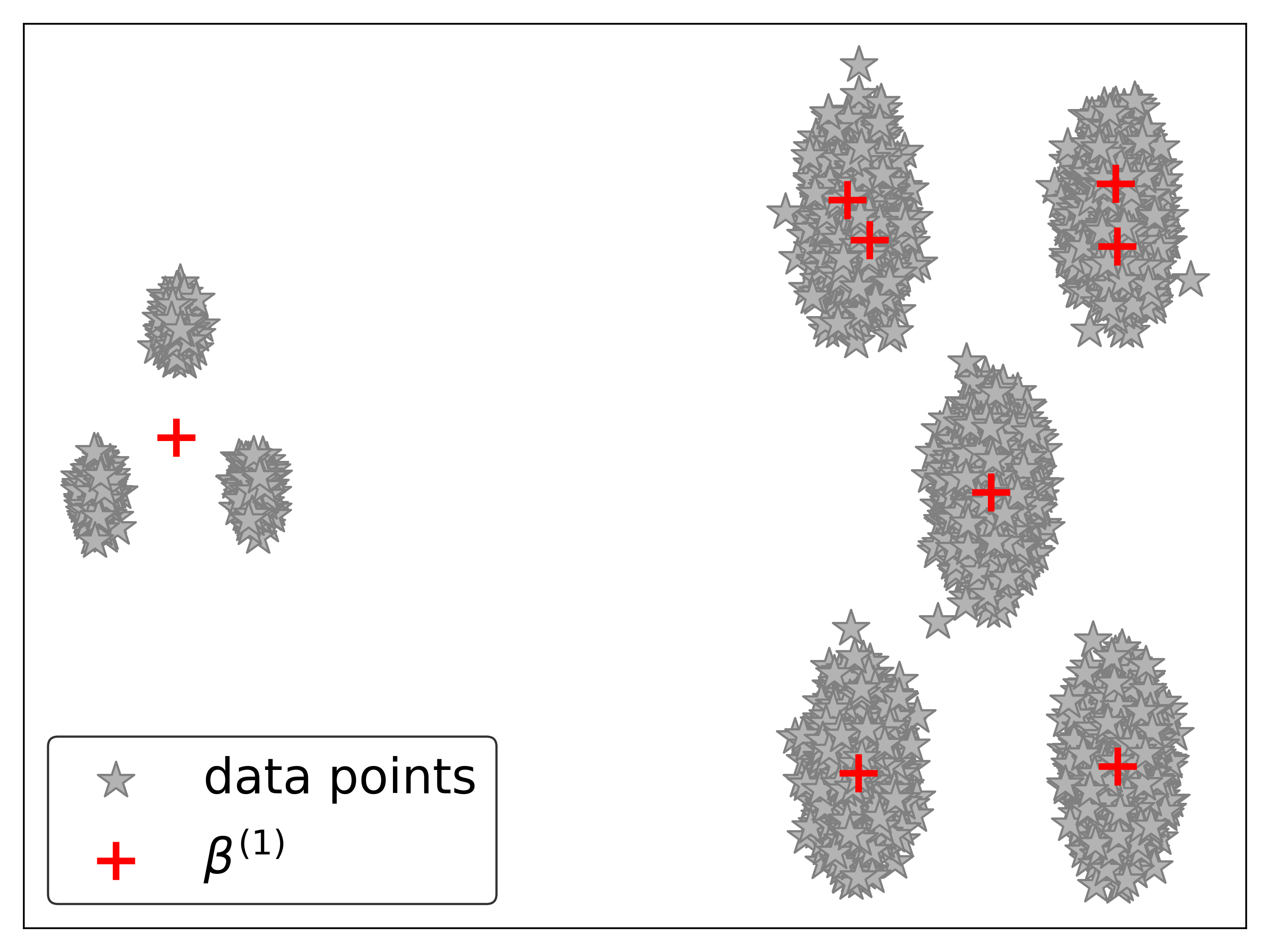}
    \label{fig:4a}
  }
  \subfloat[Detecting suitable paris ($\ell=1$)]{  
    \includegraphics[width=0.328\linewidth, height=0.17\textheight]{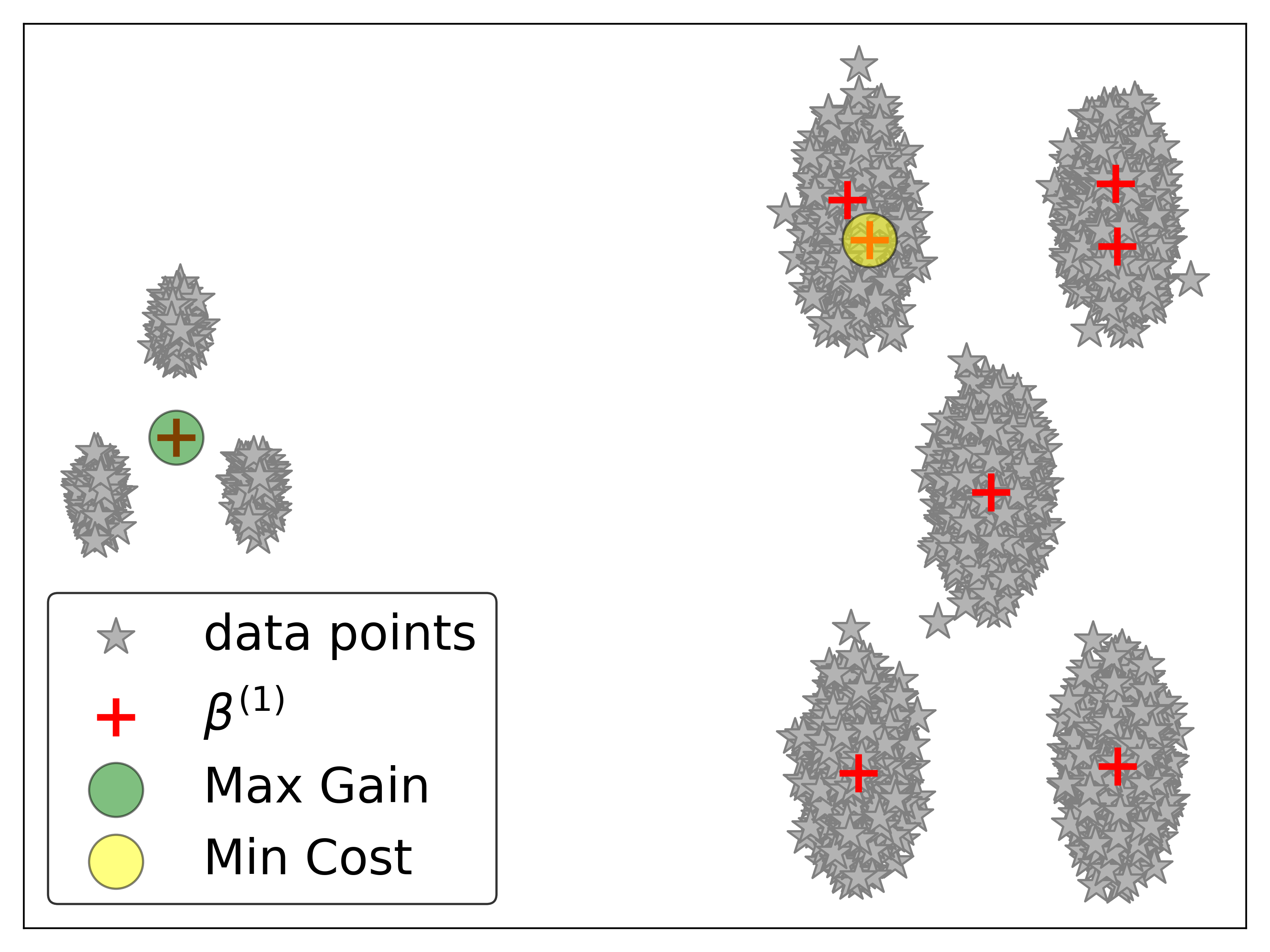}
    \label{fig:4b}
  }
  \subfloat[Solution by \kmeansmpx ($\ell=1$)]{  
    \includegraphics[width=0.328\linewidth, height=0.17\textheight]{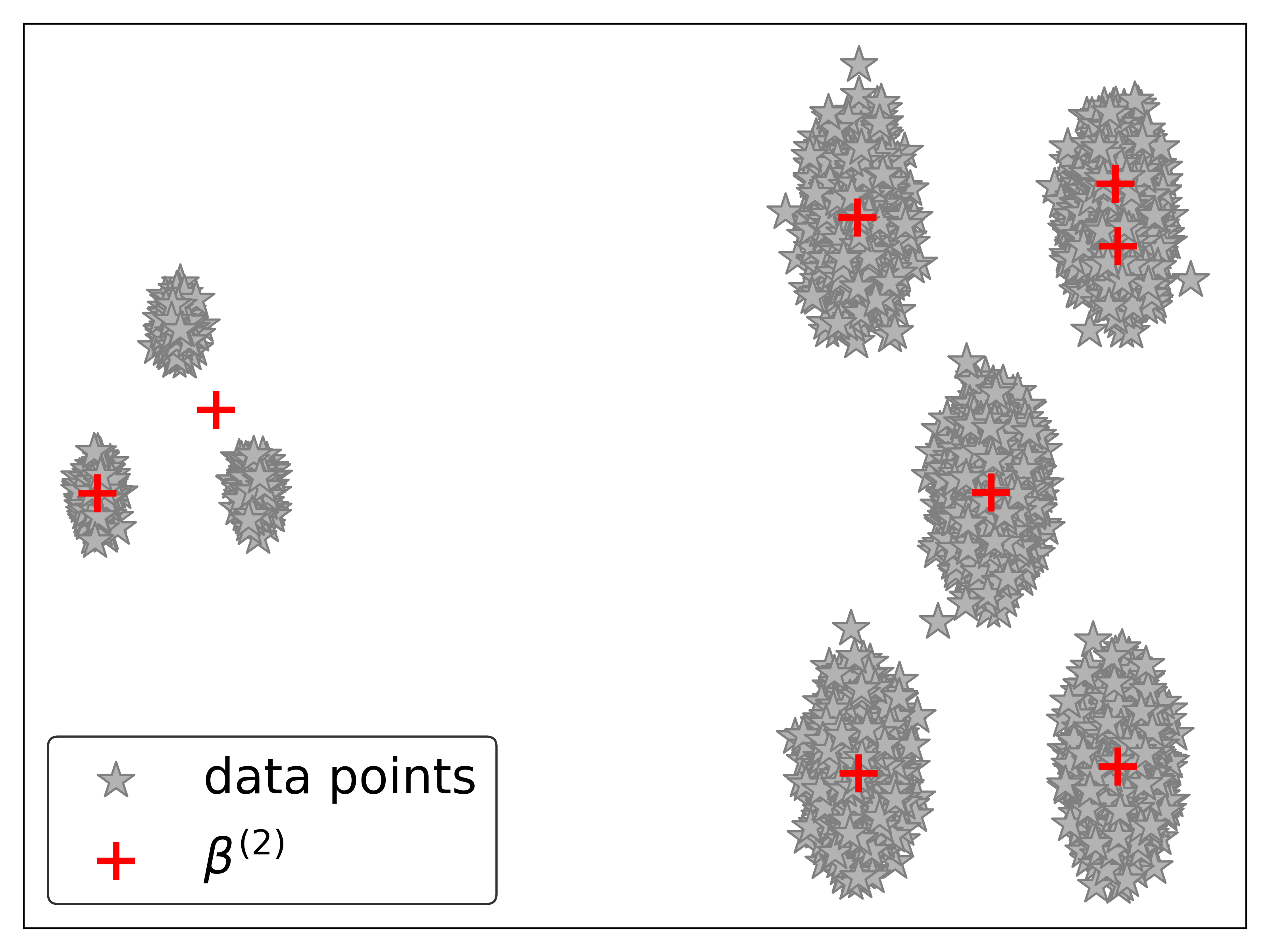}
    \label{fig:4c}
  }

  % Second row of images
  \subfloat[Detecting suitable paris ($\ell=2$)]{  
    \includegraphics[width=0.328\linewidth, height=0.17\textheight]{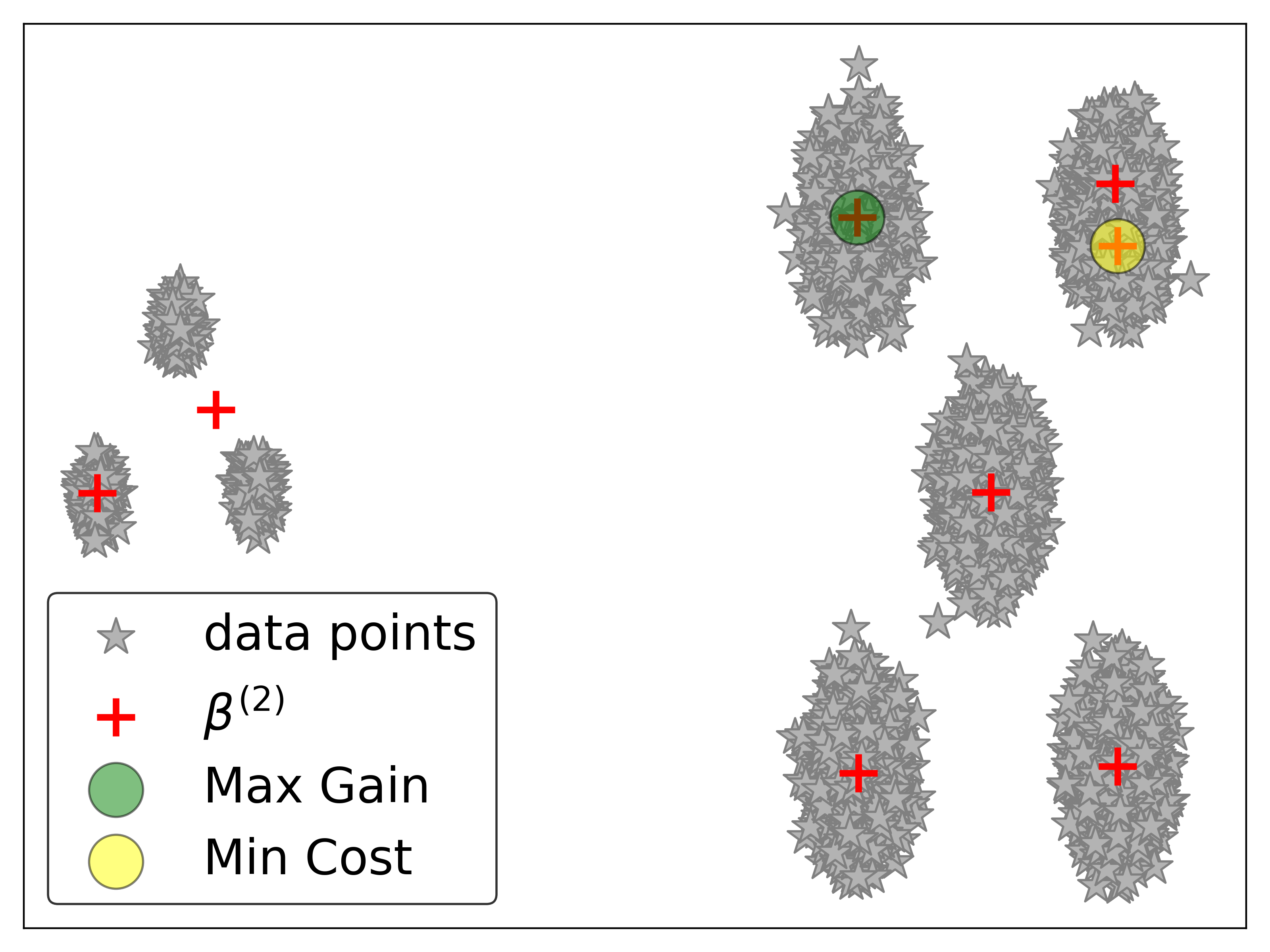}
    \label{fig:4d}
  }
  \subfloat[Solution by \kmeansmpx ($\ell=2$)]{  
    \includegraphics[width=0.328\linewidth, height=0.17\textheight]{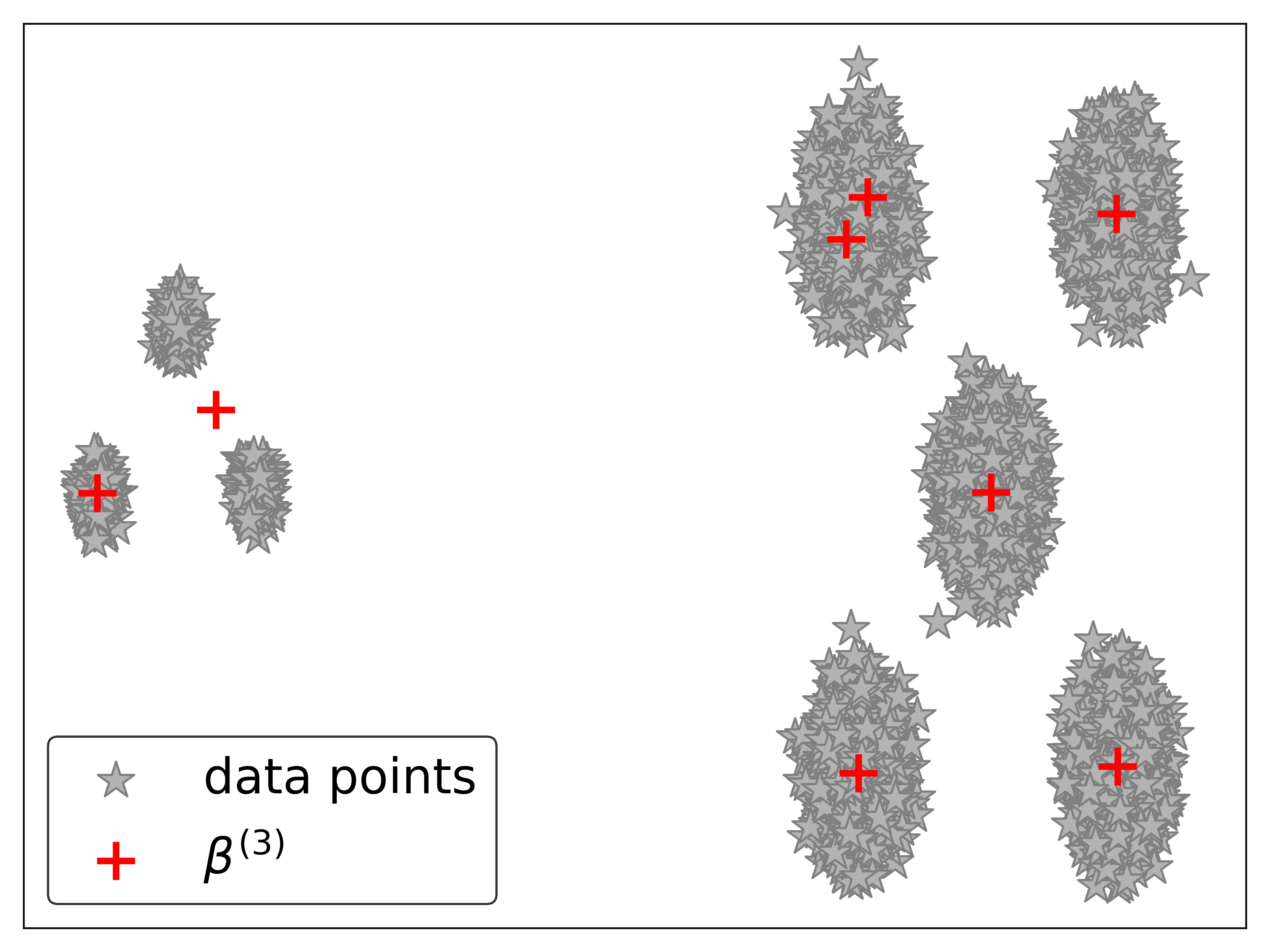}
    \label{fig:4e}
  }
  \subfloat[Detecting OFM and MFO ($\ell=1$)]{  
    \includegraphics[width=0.328\linewidth, height=0.17\textheight]{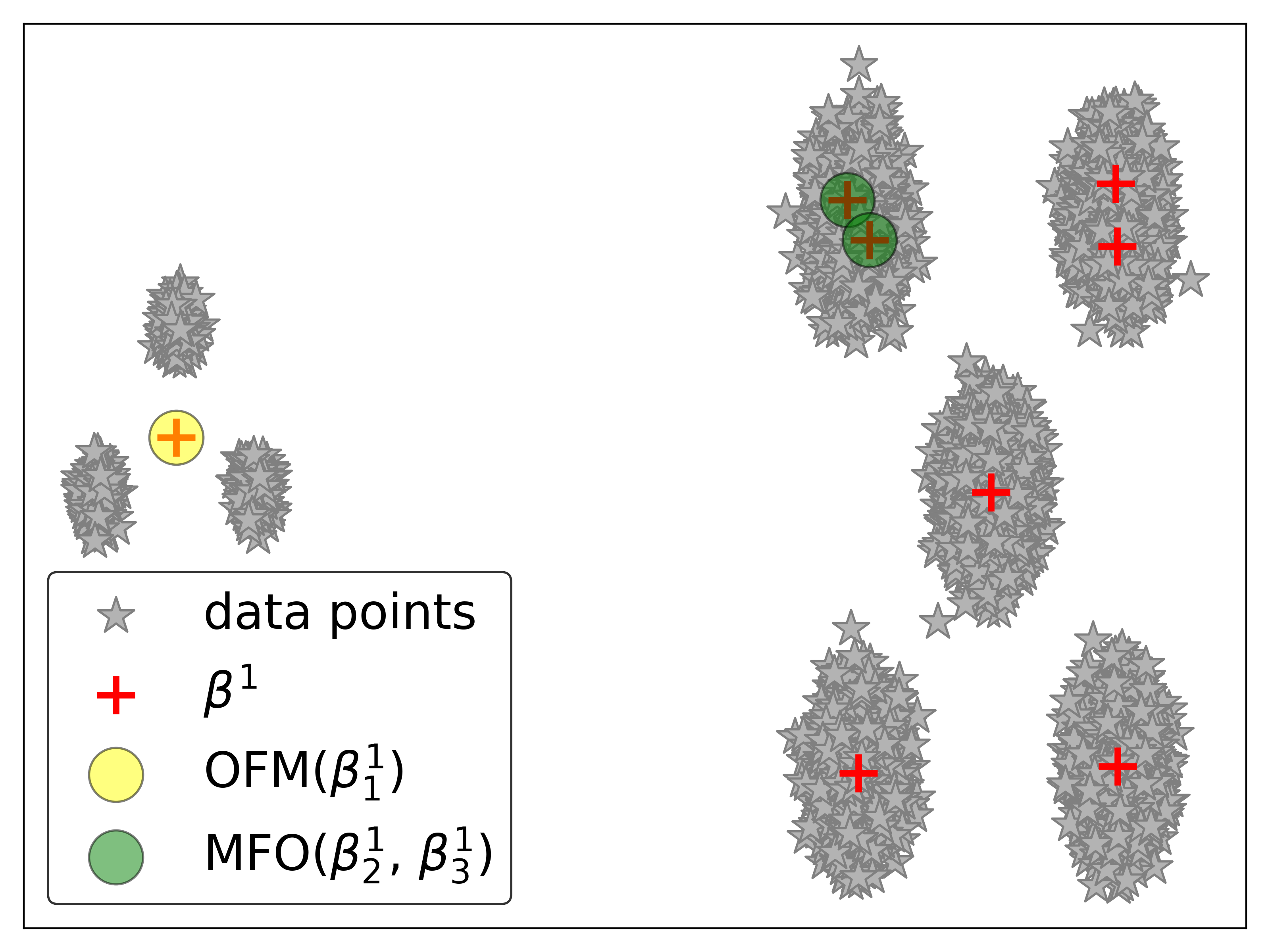}
    \label{fig:4f}
  }

  % Second row of images
  \subfloat[Solution by FFkm ($\ell=1$)]{  
    \includegraphics[width=0.328\linewidth, height=0.17\textheight]{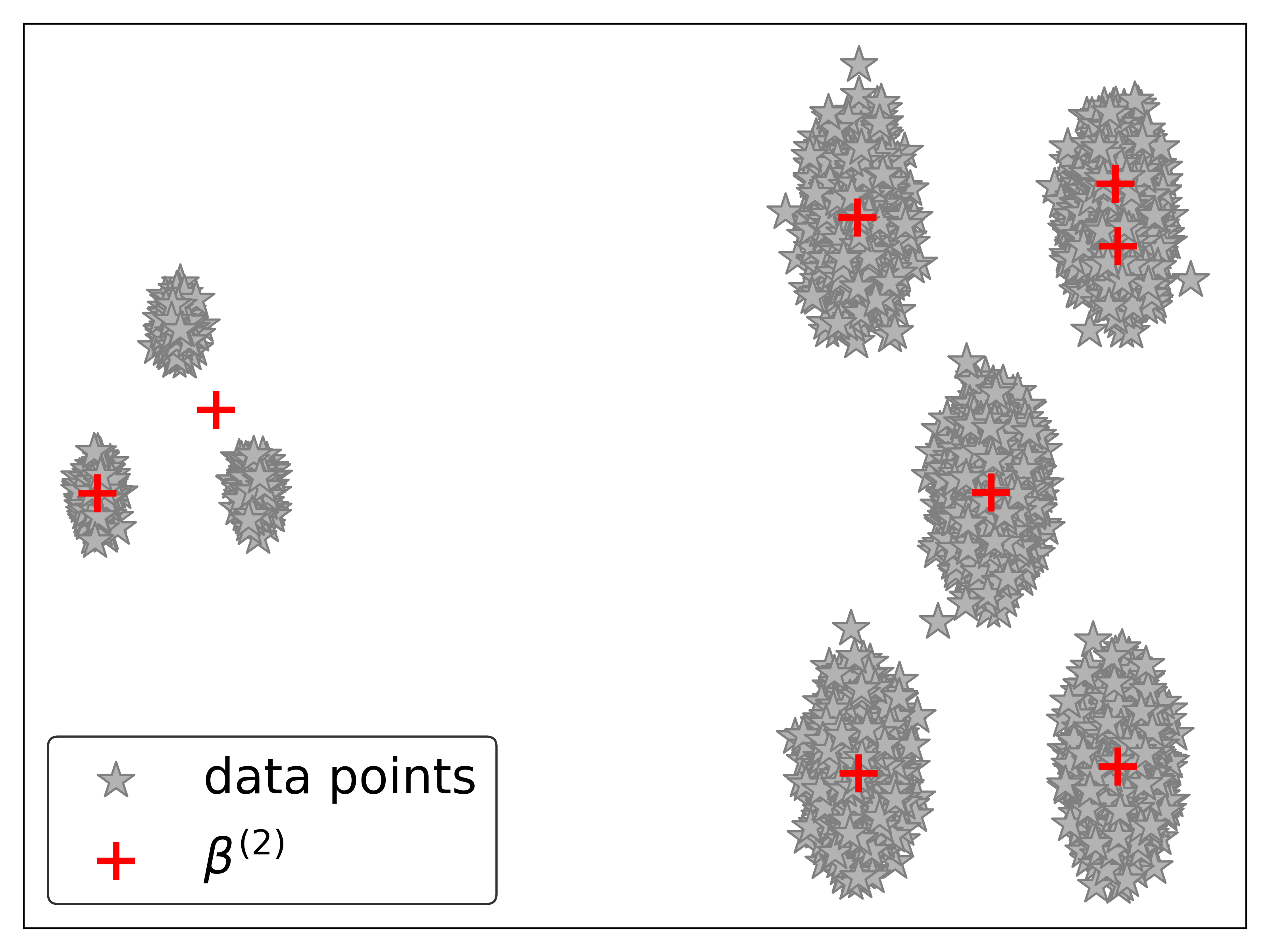}
    \label{fig:4g}
  }
  \subfloat[Detecting OFM and MFO ($\ell=2$)]{  
    \includegraphics[width=0.328\linewidth, height=0.17\textheight]{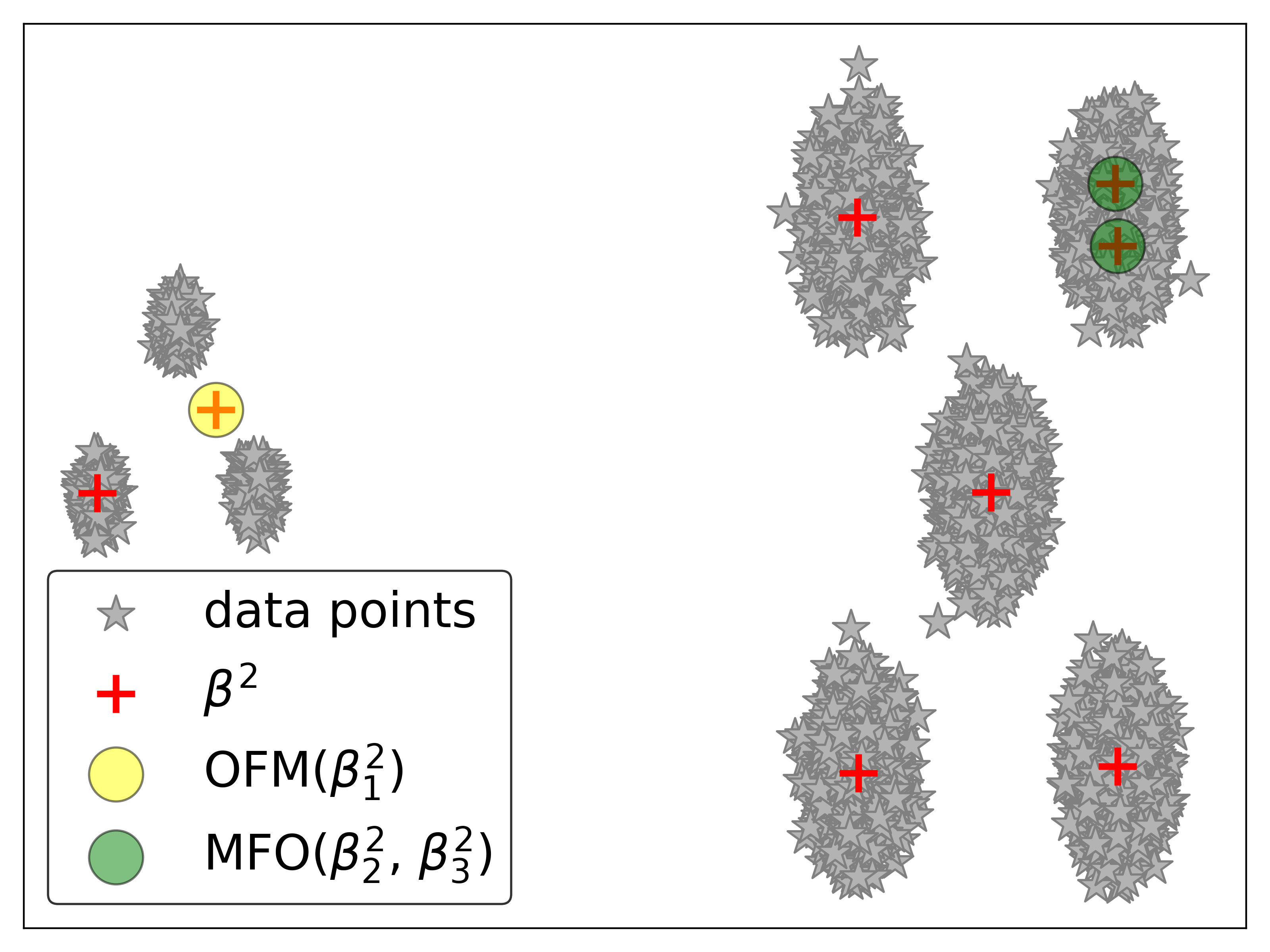}
    \label{fig:4h}
  }
  \subfloat[Solution by FFkm ($\ell=2$)]{  
    \includegraphics[width=0.328\linewidth, height=0.17\textheight]{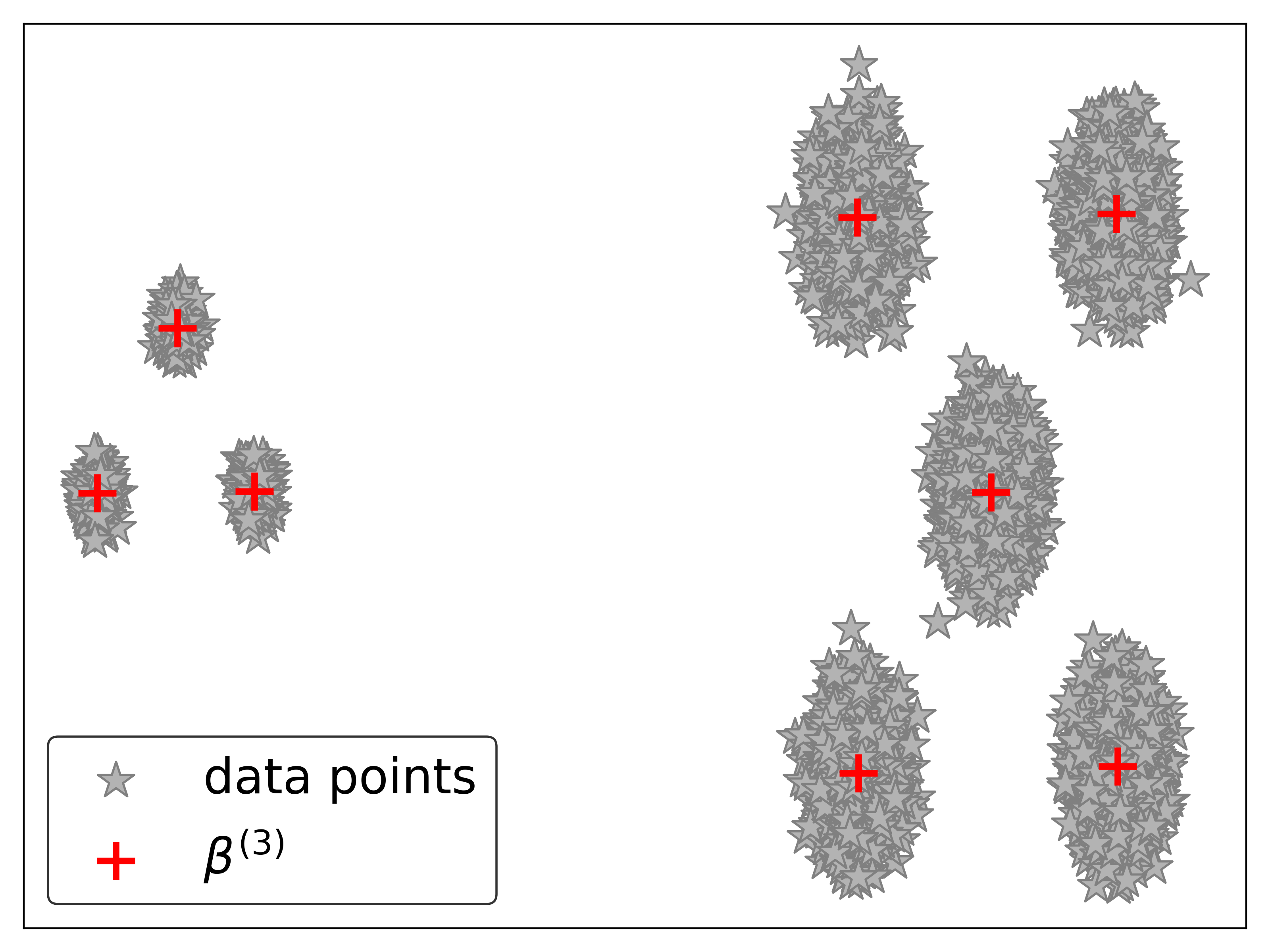}
    \label{fig:4i}
  }
  \caption{Comparison of the objective-based algorithm \kmeansmpx and the geometry-based algorithm FFkm (SD+PD). Here $\ell$ is the iteration number, and $\beta^{(\ell)}$ is the cluster centers in iteration $\ell$. (a) shows the solution by Lloyd \kmeans, which also serves as the initial solution for both \kmeansmpx and FFkm (SD+PD). (b--e) show each iteration of \kmeansmpx, while (f--i) show each iteration of FFkm (SD+PD).} \label{fig:4}
% \vspace{-.1in}
\end{figure*}

Figure~\ref{fig:4} shows a single trial of FFkm (SD+PD) and \kmeansmpx, both of which terminate after 2 iterations. Both algorithms use the same initial solution produced by Lloyd's \kmeans; see Figure~\ref{fig:4a}. The detection and solution steps of \kmeansmpx are shown in Figures~\ref{fig:4b} to~\ref{fig:4d}, and those for FFkm (SD+PD) in Figures~\ref{fig:4f} to~\ref{fig:4i}. As discussed in Section~\ref{sec:kmsmp}, each iteration of \kmeansmpx finds a suitable pair of Min-Cost cluster and Max-Gain cluster, followed by a minus-and-plus (or swap) step. These terminologies are used in Figures~\ref{fig:4b} and~\ref{fig:4d}.

From Figure~\ref{fig:4}, we see that \kmeansmpx converges to a sub-optimal local minimizer after two iterations ($\ell=2$). As shown in Figures~\ref{fig:4b} and \ref{fig:4d}, the reason is that the objective-based strategy of \kmeansmp focuses on finding the Min-Cost cluster, which is from the clusters on the right with many data points and large SSEs. As a result, the small clusters on the left are under-represented by the centers produced by \kmeansmp.   In contrast, FFkm (SD+PD) benefits from its geometry-based SD and PD detection subroutines, correctly clustering the data points on the left (see Figures~\ref{fig:4f} and \ref{fig:4h}). These results demonstrate the superior performance of FFkm (SD+PD) in clustering complex, unbalanced datasets and avoiding bad local minima that arise due to the heterogeneity of the data.

\vspace{-.03in}
\subsubsection{Ablation Study and Model Mis-specification}\label{sec:param}
\vspace{-.03in}
  
  We evaluate two variants of our framework: Fission-only \kmeans (Algorithm~\ref{alg:2}) with an under-specified initial number of clusters,  and Fusion-only \kmeans (Algorithm~\ref{alg:3}) with over-specification. This experiment serves as an ablation study on the roles of the fusion operation and the fission operation.
  We execute these two algorithms for 100 trials on each benchmark dataset discussed in Section \ref{exp:data}. For the under-parameterized Fission-only \kmeans, we consider $2$, $\lceil\frac{\ktrue}{4}\rceil$, and $\lceil\frac{\ktrue} {2}\rceil$ as the initial value of $k$. The standard deviation (SD) method is used to detect \ofm associations. For the over-parameterized Fusion-only \kmeans, the initial $k$ is $2\ktrue$, $3\ktrue$, and $4\ktrue$. The pairwise distance (PD) is used to detect \mfo associations. Both algorithms terminate with $\ktrue$ fitted clusters, and we use $\rho$-ratio as the performance metric. The experiment results are summarized in Table~\ref{tab:4} and Table~\ref{tab:5}.
  
  \begin{table}
      \centering
      \caption{Fission-only \kmeans (Algorithm~\ref{alg:2}) with Under-specified $k$}  \label{tab:4}
    % \vspace{-.15in}
      \scalebox{0.75}{
      \begin{tabular}{|c | c | c | c | c | c | c | c | c | c | c | c |c |}
      \hline \textbf{Dataset} & 
     \multicolumn{3}{|c|}{\textbf{$k=k^\star$}} & 
      \multicolumn{3}{|c|}{\textbf{$k=2$}}  & \multicolumn{3}{|c|}{\textbf{$k=\lceil\frac{\ktrue}{4}\rceil$}} & \multicolumn{3}{|c|}{\textbf{$k=\lceil\frac{\ktrue}{2}\rceil$}}\\ 
      \cline{2-13}
      \small  & \small \textbf{SR(\%)} &\small \textbf{AMR} &\small \textbf{$\rho$-ratio} &\small \textbf{SR(\%)} &\small \textbf{AMR} &\small \textbf{$\rho$-ratio} &\small \textbf{SR(\%)} &\small \textbf{AMR} &\small \textbf{$\rho$-ratio} &\small \textbf{SR(\%)} &\small \textbf{AMR} &\small \textbf{$\rho$-ratio} \\ \hline
    A1      & 1   & 0.13 &  $1.67\pm0.31$ & 100 & 0.00 &  $1.00\pm0.00$ & 100 & 0.00 &  $1.00\pm0.00$ & 99  & 0.00 &  $1.00\pm0.02$\\
    A2      & 0   & 0.13 &  $1.69\pm0.24$ & 100 & 0.00 &  $1.00\pm0.00$ & 100 & 0.00 &  $1.00\pm0.00$ & 97  & 0.00 &  $1.00\pm0.02$\\
    A3      & 0   & 0.13 &  $1.73\pm0.25$ & 100 & 0.00 & 	$1.00\pm0.00$ & 100 & 0.00 & 	$1.00\pm0.00$ & 92  & 0.00 & 	$1.01\pm0.02$\\
    S1      & 1   & 0.14 &  $2.23\pm0.55$ & 100 & 0.00 & 	$1.00\pm0.00$ & 100 & 0.00 & 	$1.00\pm0.00$ & 100 & 0.00 & 	$1.00\pm0.00$\\
    S2      & 3   & 0.11 &  $1.56\pm0.39$ & 100 & 0.00 & 	$1.00\pm0.00$ & 100 & 0.00 & 	$1.00\pm0.00$ & 100 & 0.00 & 	$1.00\pm0.00$\\
    S3      & 8   & 0.09 &  $1.18\pm0.10$ & 100 & 0.00 & 	$1.00\pm0.00$ & 100 & 0.00 & 	$1.00\pm0.00$ & 100 & 0.00 & 	$1.00\pm0.00$\\
    S4      & 20  & 0.07 &  $1.10\pm0.08$ & 0   & 0.13 & 	$1.15\pm0.00$ & 0   & 0.13 & 	$1.15\pm0.00$ & 0   & 0.07 & 	$1.08\pm0.02$\\
    Unbalance   & 0   & 0.48 &  $9.62\pm1.62$ & 100 & 0.00 & 	$1.00\pm0.00$ & 100 & 0.00 & 	$1.00\pm0.00$ & 61  & 0.05 & 	$4.33\pm4.25$\\
    Dim032      & 1   & 0.21 &  $51.99\pm19.56$ & 100 & 0.00 & 	$1.00\pm0.00$ & 99  & 0.00 & 	$1.12\pm1.17$ & 68  & 0.02 & 	$5.25\pm6.75$\\
    Birch1      & 0   & 0.07 &  $1.20\pm0.04$ & 100 & 0.00 & 	$1.00\pm0.00$ & 100 & 0.00 & 	$1.00\pm0.00$ & 100 & 0.00 & 	$1.00\pm0.00$\\
       \hline
      \end{tabular}}
    % \vspace{-.1in}
    \end{table}
  
   \begin{table}
      \centering
      \caption{Fusion-only \kmeans (Algorithm~\ref{alg:3}) with Over-specified $k$}  \label{tab:5}
      % \vspace{-.15in}
      \scalebox{0.75}{
      \begin{tabular}{|c | c | c | c | c | c | c | c | c | c | c | c |c |}
      \hline \textbf{Dataset} & 
     \multicolumn{3}{|c|}{\textbf{$k=k^\star$}} & \multicolumn{3}{|c|}{\textbf{$k=2k^\star$}}  & \multicolumn{3}{|c|}{\textbf{$k=3k^\star$}} & \multicolumn{3}{|c|}{\textbf{$k=4k^\star$}}\\ 
      \cline{2-13}
      \small  & \small \textbf{SR(\%)} &\small \textbf{AMR} &\small \textbf{$\rho$-ratio} &\small \textbf{SR(\%)} &\small \textbf{AMR} &\small \textbf{$\rho$-ratio} &\small \textbf{SR(\%)} &\small \textbf{AMR} &\small \textbf{$\rho$-ratio} &\small \textbf{SR(\%)} &\small \textbf{AMR} &\small \textbf{$\rho$-ratio} \\ \hline
    A1      & 1   & 0.13 &  $1.67\pm0.31$ & 96  & 0.00 & 	$1.01\pm0.05$ & 100 & 0.00 & 	$1.00\pm0.00$ & 100 & 0.00 & 	$1.00\pm0.00$\\
    A2      & 0   & 0.13 &  $1.69\pm0.24$ & 88  & 0.00 & 	$1.01\pm0.04$ & 99  & 0.00 & 	$1.00\pm0.01$ & 100 & 0.00 &  $1.00\pm0.00$\\
    A3      & 0   & 0.13 &  $1.73\pm0.25$ & 89  & 0.00 &  $1.01\pm0.03$ & 100 & 0.00 &  $1.00\pm0.00$ & 100 & 0.00 &  $1.00\pm0.00$\\
    S1      & 1   & 0.14 &  $2.23\pm0.55$ & 97  & 0.00 &  $1.02\pm0.10$ & 100 & 0.00 &  $1.00\pm0.00$ & 100 & 0.00 &  $1.00\pm0.00$\\
    S2      & 3   & 0.11 &  $1.56\pm0.39$ & 100 & 0.00 &  $1.00\pm0.00$ & 100 & 0.00 &  $1.00\pm0.00$ & 100 & 0.00 &  $1.00\pm0.00$\\
    S3      & 8   & 0.09 &  $1.18\pm0.10$ & 97  & 0.00 &  $1.00\pm0.02$ & 100 & 0.00 &  $1.00\pm0.00$ & 100 & 0.00 &  $1.00\pm0.00$\\
    S4      & 20  & 0.07 &  $1.10\pm0.08$ & 85  & 0.01 &  $1.01\pm0.03$ & 50  & 0.03 &	$1.04\pm0.04$ & 25  & 0.05 &  $1.06\pm0.03$\\
    Unbalance   & 0   & 0.48 &  $9.62\pm1.62$ & 0	  & 0.44 &  $8.23\pm2.46$ & 4	  & 0.38 &  $6.89\pm2.99$ & 6	  & 0.34 &  $5.90\pm2.93$ \\
    Dim032      & 1   & 0.21 &  $51.99\pm19.56$ & 62  & 0.03 &  $6.88\pm8.48$ & 93  & 0.00 &  $1.92\pm3.40$ & 99  & 0.00 &  $1.10\pm1.01$ \\
    Birch1      & 0   & 0.07 &  $1.20\pm0.04$ & 100 & 0.00 & $1.00\pm0.00$ & 100 & 0.00 & $1.00\pm0.00$ & 100 & 0.00 & $1.00\pm0.00$ \\
       \hline
      \end{tabular}}
    % \vspace{-.1in}
    \end{table}

  One observes that both algorithms returned near-optimal solutions for most datasets, with the exceptions of \textbf{S4}, \textbf{Unbalance}, and \textbf{DIM032}. %Moreover, they are much better than the solutions with \kmeans++ initialization.
  For Fission-only \kmeans, setting $k=2$ achieves the best performance, with all datasets except S4 having a 100\% success rate; the performance is slightly worse with $k=\lceil\frac{\ktrue}{2}\rceil$. 
  For Fusion-only \kmeans, all choices of $k$ lead to worse performance on S4 and Unbalance. We attribute this performance to the lack of the fission step as well as the use of the pairwise distance (PD) subroutine for the fusion step, which face challenges when the data has overlapping or unbalanced clusters.
  %Due to the high overlap of S4 and the limitations of \kmeans, it makes sense that the performance is worse when we apply pairwise distance (PD). The results can be improved by utilizing alternative methods like objective increment (OI); the choice of methods in general depends on the application scenarios.

\vspace{-.03in}
\subsubsection{Comparison with Related Algorithms}\label{exp:com}
\vspace{-.03in}
    
In Tables~\ref{tab:6} and~\ref{tab:7}, we compare the Success Rates (SR) and $\rho$-ratios, respectively of Algorithm~\ref{alg:1} (FFkm), Lloyd's algorithm, \kmeansmpx, and other related algorithms~\cite{pelleg2000x, morii2006clustering, lei2016robust}, using $100$ independent trials on the benchmark datasets. When the success rate is less than 100\%, the Average Missing Rate (AMR) is given in parentheses (cf.\ Section~\ref{sec:evalaution}). We present results for three combinations of subroutines for FFkm in Tables~\ref{tab:6} and~\ref{tab:7}. Results for all combinations of subroutines are available in Appendix~\ref{app:six_detect}. In Tables~\ref{tab:6} and~\ref{tab:7}, only SR and AMR are reported for the algorithms in \cite{lei2016robust} and \cite{pelleg2000x}, because they may use a different initial number of clusters than the ground truth clustering.

\begin{table}
  \centering
  \caption{Success rate (\%) comparison (best results in boldface)}   \label{tab:6}
% \vspace{-.15in}
  \scalebox{0.75}{
  \begin{tabular}{|c | c | c | c | c | c | c | c || c | c |}
  \hline 
  \textbf{Dataset} & \textbf{Lloyd} & \textbf{\kmeanspp}  & \textbf{\kmeansmpx} & \textbf{SD+OI} & \textbf{TD+OI} & \textbf{RD+PD} &  \textbf{\cite{morii2006clustering}} & \textbf{\cite{lei2016robust}} & \textbf{\cite{pelleg2000x}}  \\ 
  \hline
  A1        & 1 (0.13)  & 49 (0.03)    & \textbf{100}  & \textbf{100}  & \textbf{100}   & \textbf{100} & 99          & 66 (0.02)      & \textbf{100}  \\
  A2        & 0 (0.13)  & 6 (0.04)     & \textbf{100}  & \textbf{100}  & \textbf{100}   & \textbf{100} & 96          & 5 (0.07)       & \textbf{100} \\
  A3        & 0 (0.13)  & 4 (0.03)     & \textbf{100}  & \textbf{100}  & \textbf{100}   & \textbf{100} & 99          & 0 (0.14)       & 0 (0.92)\\
  S1        & 1 (0.14)  & 71 (0.02)    & \textbf{100}  & \textbf{100}  & \textbf{100}   & \textbf{100} & \textbf{100}& \textbf{100}   & \textbf{100}\\
  S2        & 3 (0.11)  & 61 (0.03)    & \textbf{100}  & \textbf{100}  & \textbf{100}   & \textbf{100} &  90 (0.01)  & \textbf{100}   & \textbf{100}\\
  S3        & 8 (0.09)  & 48 (0.04)    & \textbf{100}  & 89(0.00)      & 96(0.00)       & 89(0.01)     & 72  (0.02)  & \textbf{100}   & \textbf{100}\\
  S4        & 20 (0.07) & 52 (0.03)    & 93(0.00)      & 39(0.04)      & 90(0.01)       & 41(0.05     )& 29  (0.05)  & \textbf{100}   & 98 (0.01)\\
  Unbalance & 0 (0.48)  & 97           & \textbf{100}  & \textbf{100}  & \textbf{100}   & \textbf{100} & 17 (0.5)    & 99             & 50 (0.25)\\
  Dim032    & 1 (0.21)  & \textbf{100} & \textbf{100}  & \textbf{100}  & \textbf{100}   & \textbf{100} & 94          & \textbf{100}   & \textbf{100}\\
  birch1    & 0 (0.07)  & 0            & \textbf{100}  & \textbf{100}  &\textbf{100}    & \textbf{100} & \textbf{100}& 0 (0.08)       & 0 (0.96)  \\
  \hline
  \end{tabular}}
  % \vspace{-.1in}
\end{table}

\begin{table}
  \centering
  \caption{$\rho$-ratio comparison (best results in boldface)}  \label{tab:7}
% \vspace{-.15in}
  \scalebox{0.88}{
  \begin{tabular}{|c | c | c | c | c | c | c | c |}
  \hline 
  \textbf{Dataset} & \textbf{Lloyd} & \textbf{\kmeanspp} & \textbf{\kmeansmpx} & \textbf{SD+OI} & \textbf{TD+OI} & \textbf{RD+PD} & \textbf{\cite{morii2006clustering}} \\ 
  \hline
  A1        & $1.67 \pm 0.31$ &  $1.12\pm0.14$  & $\bm{1.00\pm0.00}$ & $\bm{1.00\pm0.00}$ & $\bm{1.00\pm0.00}$ & $\bm{1.00\pm0.00}$&  $1.00 \pm 0.02 $   \\
  A2        & $1.69\pm0.24$   & $1.14\pm0.08$   & $\bm{1.00\pm0.00}$ & $\bm{1.00\pm0.00}$ & $\bm{1.00\pm0.00}$ & $\bm{1.00\pm0.00}$&  $1.01\pm 0.03 $  \\
  A3        & $1.73\pm0.25$   & $1.13\pm0.06$   & $\bm{1.00\pm0.00}$ & $\bm{1.00\pm0.00}$ & $\bm{1.00\pm0.00}$ & $\bm{1.00\pm0.00}$&  $1.00 \pm 0.01 $  \\
  S1        & $2.23\pm0.55$   & $1.16\pm0.25$   & $\bm{1.00\pm0.00}$ & $\bm{1.00\pm0.00}$ & $\bm{1.00\pm0.00}$ & $\bm{1.00\pm0.00}$&  $\bm{1.00 \pm 0.00} $  \\
  S2        & $1.56\pm0.39$   & $1.10\pm0.13$   & $\bm{1.00\pm0.00}$ & $\bm{1.00\pm0.00}$ & $\bm{1.00\pm0.00}$ & $\bm{1.00\pm0.00}$&  $1.02 \pm 0.06 $   \\
  S3        & $1.18\pm0.10$   & $1.07\pm0.07$   & $1.01\pm0.02$      & $1.01\pm0.04$      & $\bm{1.00\pm0.00}$ & $1.01\pm0.04$     &  $1.03 \pm 0.06 $   \\
  S4        & $1.10\pm0.08$   & $1.04\pm0.05$   & $1.01\pm0.01$      & $1.05\pm0.05$      & $1.01\pm0.02$      & $1.06\pm0.07$     &  $1.06 \pm 0.05 $   \\
  Unbalance & $9.62\pm1.62$   & $1.03\pm0.18$   & $\bm{1.00\pm0.00}$ & $\bm{1.00\pm0.00}$ & $\bm{1.00\pm0.00}$ & $\bm{1.00\pm0.00}$&  $5.14\pm 1.92 $   \\
  Dim032    & $51.99\pm19.56$ & $1.10\pm1.01$  & $\bm{1.00\pm0.00}$  & $\bm{1.00\pm0.00}$ & $\bm{1.00\pm0.00}$ & $\bm{1.00\pm0.00}$&  $1.70\pm2.76$  \\
  birch1    & $1.20\pm0.04$   & $1.09\pm0.02$  & $\bm{1.00\pm0.00}$  & $\bm{1.00\pm0.00}$ & $\bm{1.00\pm0.00}$ & $\bm{1.00\pm0.00}$&  $\bm{1.00\pm0.00}$  \\
   \hline
  \end{tabular}}
  % \vspace{-.1in}
\end{table}

As seen from Tables~\ref{tab:6} and~\ref{tab:7}, FFkm with subroutines (SD+OI), (TD+OI), (RD+PD) reliably recovers the ground truth on all benchmark datasets except S3 and S4. Given that these datasets vary in the number, shapes and separation of clusters, this performance demonstrates the robustness and effectiveness of FFkm. The S3 and S4 datasets have highly overlapping clusters. As demonstrated in Section \ref{sec:param}, in these scenarios the geometry-based subroutines---Standard Deviation (SD), $\epsilon$-Radius (RD), and Pairwise Distance (PD)---may not be effective. Instead, using the subroutines Total Deviation (TD) and Objective Increment (OI), one can improve the success rate of 41\% for the geometry-based FFkm (RD+PD) to 90\% for the objective-based FFkm (TD+OI). Among other \kmeans algorithms, only \kmeansmpx achieves a success rate and $\rho$-ratio comparable to FFkm. Specifically, through an objective-based strategy, both \kmeansmpx and FFkm (TD+OI) achieve over 90\% SR for the dataset S4, with \kmeansmpx somewhat higher than FFkm (TD+OI). For the dataset S3, FFkm (TD+OI) achieves a better $\rho$-ratio than \kmeansmpx.

In light of the observations above on \kmeansmp, a more detailed comparison is given in Table~\ref{tab:8} between Lloyd's, \kmeanspp, \kmeansmp and FFkm. Following the experimental setup in \kmeansmp paper \cite{ismkhan2018ik}, 50 independent trials were conducted and the sum of squared errors (SSE) was calculated. (The dataset Unbalance and DIM032 were not considered in \cite{ismkhan2018ik}.) We report the results quoted from \cite{ismkhan2018ik} (which reports two decimal places) as well as those from our own re-implementation of \kmeansmpx and FFkm (with four decimal places). The best SSE for each dataset is highlighted in bold. As observed in Table~\ref{tab:8}, all six FFkm subroutines significantly outperform both Lloyd's algorithm and \kmeanspp. Except for datasets S3 and S4, FFkm (SD+PD) performs better than or equally well as  \kmeansmp; similar performance can be seen from FFkm subroutines with SD+OI, TD+PD, and TD+OI. For the dataset S3, FFkm with the objective-based subroutines  TD+OI achieves the best SSE; for S4, \kmeansmpx (our re-implementation) has the best SSE. These two datasets have highly overlapping clusters, which present challenges for geometry-based algorithms, whereas the TD+OI subroutine may mitigate these challenges.
Finally, we note that FFkm (RD+PD) did not achieve the ideal SSE, possibly due to the radius settings discussed in Section \ref{sec:disc}.

\begin{table}
  \centering
  \caption{Sum of squared errors (SSE) comparison between results taken from \cite{ismkhan2018ik} the and FFkm (best results in boldface)}\label{tab:8}
  % \vspace{-.15in}
\scalebox{0.68}{
  \begin{tabular}{|c | c | c | c | c | c | c | c | c| c| c| c|c|}
  \hline 
  \textbf{Dataset} & \multicolumn{3}{c|}{\textbf{SSE of \kmeansmp in paper ~\cite{ismkhan2018ik}}}  & \multicolumn{7}{c|}{\textbf{SSE computed using the same machine}}  \\ 
  \cline{2-11}
  \textbf{} & \textbf{Lloyd} & \textbf{\kmeanspp} & \textbf{\kmeansmp} &  \textbf{\kmeansmpx} & \textbf{SD+PD} & \textbf{SD+OI} & \textbf{TD+PD} & \textbf{TD+OI} & \textbf{RD+PD}& \textbf{RD+OI}\\ 
  \hline
  A1        & 2.08E+10 & 1.73E+10 & 1.22E+10 & \textbf{1.2146E+10} & \textbf{1.2146E+10} & \textbf{1.2146E+10} & \textbf{1.2146E+10} & 1.2149E+10 & 1.2186E+10 & 1.2225E+10 \\
  A2        & 3.47E+10 & 2.99E+10 & 2.03E+10 & 2.0311E+10 & \textbf{2.0287E+10} & \textbf{2.0287E+10} & \textbf{2.0287E+10} & \textbf{2.0287E+10} & 2.2053E+10 & 2.2389E+10 \\
  A3        & 5.23E+10 & 4.29E+10 & 2.90E+10 & 2.8943E+10 & \textbf{2.8938E+10} & \textbf{2.8938E+10} & \textbf{2.8938E+10} & \textbf{2.8938E+10} & 3.0684E+10 & 3.1048E+10 \\
  S1        & 1.85E+13 & 1.67E+13 & 8.92E+12 & \textbf{8.9177E+12} & \textbf{8.9177E+12} & \textbf{8.9177E+12} & \textbf{8.9177E+12} & \textbf{8.9176E+12} & 13.3627E+12 & 11.2245E+12 \\
  S2        & 2.01E+13 & 1.82E+13 & 1.33E+13 & 1.3290E+13 & \textbf{1.3279E+13} & \textbf{1.3279E+13} & \textbf{1.3279E+13} & \textbf{1.3279E+13} & 1.4659E+13 & 1.4810E+13 \\
  S3        & 1.94E+13 & 1.90E+13 & 1.69E+13 &1.6893E+13& 1.7641E+13 & 1.7146E+13 & 1.7365E+13 & \textbf{1.6889E+13} & 1.8280E+13 & 1.8447E+13 \\
  S4        & 1.70E+13 & 1.67E+13 & 1.57E+13 & \textbf{1.5740E+13} & 1.6330E+13 & 1.6330E+13 & 1.6332E+13 &1.5748E+13 & 1.6866E+13 & 1.6327E+13 \\
  Birch1    & 1.13E+14 & 1.06E+14 & 9.28E+13 & 9.2815E+13 &\textbf{9.2772E+13} & \textbf{9.2772E+13} & \textbf{9.2772E+13} & \textbf{9.2772E+13} & 9.5754E+13 & 9.5725E+13\\
   \hline
  \end{tabular}}
  % \vspace{-.1in}
\end{table}

Combining with the findings from Section~\ref{sec:synthetic}, we observe that relying on a single strategy (geometry- or objective-based) often leads to limited performance. FFkm demonstrates effectiveness on a broad spectrum of datasets as well as the flexibility to incorporate different strategies/subroutines. The geometry-based FFkm (SD+PD) can handle highly unbalanced datasets, while the objective-based FFkm (TD+OI) is effective with overlapping datasets. In general, the choice of detection routines in FFkm can adapt to the specific characteristics of the dataset.

\vspace{-.03in}
\subsection{Experiments on Real-world data}\label{sec:real-world}
\vspace{-.03in}

\begin{figure}
 \centering\includegraphics[width=\textwidth]{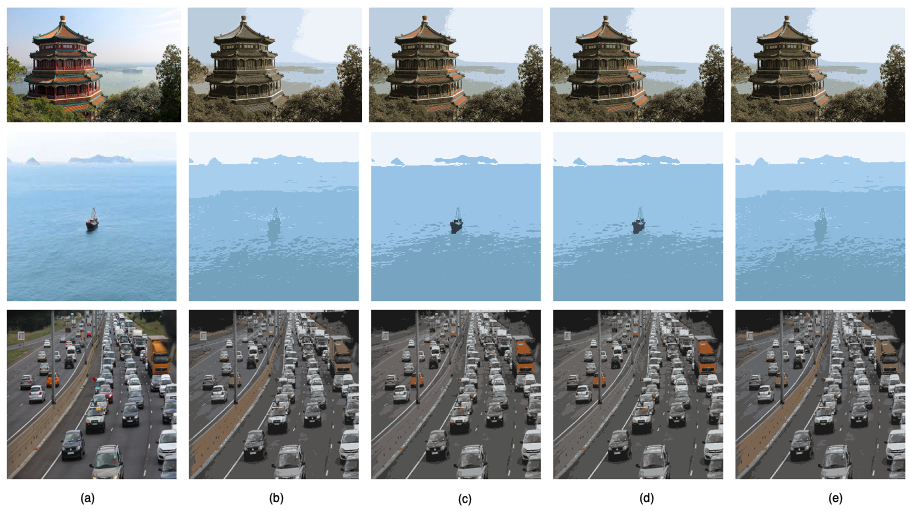} 
  % \centering\includegraphics[width=\textwidth, clip, trim = 0 400 0 0]{Figure/fig5.png} 
  % \vspace{-.1in}
  %     \includegraphics[width=\textwidth, height = 0.12\textheight, clip, trim = 0 190 0 110]{Figure/fig5.png}
  %     \includegraphics[width=\textwidth, height = 0.10\textheight, clip, trim = 0 50 0 320]{Figure/fig5.png}
  %   \includegraphics[width=\textwidth, clip, trim = 0 0 0 485]{Figure/fig5.png}
    \caption{Results of unsupervised color quantization using different numbers of clusters ($k$ values for colors).  The images are organized in rows from top to bottom: Palace ($k=8$), Boat ($k=4$), Traffic ($k=8$).  Each column shows (a) The original image ($k$ is provided in Table~\ref{tab:9}).  (b) The result of Lloyd \kmeans.  (c) The result of FFkm (SD+PD).  (d) The result of FFkm (TD+OI).  (e) The result of \kmeansmpx.}
  \label{fig:5}
  % \vspace{-.03in}
\end{figure}

\begin{table}
      \centering\caption{Results of unsupervised color quantization using different numbers of clusters ($k$ values for colors)}\label{tab:9} %Sum of Squared Errors (SSE) is provided for comparison between four \kmeans algorithms. (best results in boldface)}
      % \vspace{-.15in}
      \scalebox{0.75}{
      \begin{tabular}{|c | c | c | c | c | c | c | c |}
      \hline 
      \textbf{Image} & \textbf{dimension} & \textbf{\# points} &   \textbf{\# of clusters} & \multicolumn{4}{c|}{\textbf{SSE}}  \\ 
      \cline{5-8}
                       &             &                   &  \textbf{(colors)} & \textbf{Lloyd} &\textbf{FFkm (SD+PD)} & \textbf{FFkm (TD+OI)} & \textbf{\kmeansmp$^\star$}\\ 
      \hline
      Palace (k=8)    & 3 & 273280 & 966154 & $2874.01$ & $2660.61$ & $\bm{2655.26}$ & $2685.48$ \\
      Boat (k=4)      & 3 & 65536  & 5498   & $286.38$  & $\bm{270.70}$  & $273.67$  & $286.38$ \\
      Traffic (k=8)   & 3 & 65536  & 29792  & $364.25$  & $\bm{348.30}$  & $\bm{348.30}$  & $364.25$ \\
      Flower (k=8)    & 3 & 65536  & 49178  & $801.34$  & $796.99$       & $\bm{789.26}$ & $801.34$\\
      Red Panda (k=8) & 3 & 65536  & 52215  & $815.62$  & $\bm{781.07}$  & $809.82$  & $815.62$ \\
      Babbon (k=10)   & 3 & 65536  & 59951  & $859.97$  & $\bm{855.33}$ & $859.78$ & $859.97$\\
      Peppers (k=10)   & 3 & 65536  & 53527  & $699.78$  & $685.87$ & $\bm{685.81}$ & $699.78$\\
      Earch (k=5)     & 3 & 65536  & 28917  & $896.44$  & $756.08$  & \bm{$755.80$}  & $758.37$ \\
      \hline
      \end{tabular}}
      % \vspace{-.1in}
    \end{table}
    
\paragraph{Color Quantization} Color Quantization (CQ) is a fundamental image processing operation that reduces the number of distinct colors in a true-color image. CQ has been used to benchmark and visualize clustering algorithms \cite{abernathy2022incremental}. 
%We consider the use of CQ technology to visualize the results of FFkm. 
For \kmeans-based image segmentation applications, \cite{wibowo2023performances} has considered Flower, Red Panda, and Traffic images from the Bing Image Downloader library, and \cite{mujica2020color} has explored the Berkeley Segmentation Data Set 500 (BSD500). The work \cite{abernathy2022incremental, frackiewicz2022efficient} has considered applications of CQ using images like Baboon, Parrots, Fruits, and Peppers. Here we consider images of Palace, Boat, Traffic, Flower, Red Panda, Earth, Baboon, and Peppers. The Palace image is from \cite{sklearn2011}; all other images are from \cite{HelkinTestImages} and resized to \(256 \times 256\) with quantized RGB color values.

In Figure~\ref{fig:5}, we show results for a subset of the images for the objective-based algorithms Lloyd's, FFkm (TD+OI), and \kmeansmpx and the geometry-based algorithm FFkm (SD+PD). (For the rest of the images, the difference between results produced by different algorithms are not discernible by human eyes; these results are given in Appendix \ref{app:images}). As seen from Figure~\ref{fig:5}, both FFkm (SD+PD) and FFkm (TD+OI) can clearly reveal a red roof in the image Palace, a boat in image Boat, and an orange truck in image Traffic, demonstrating the algorithms' ability to avoid local minima and finding a better solution than Lloyd's and \kmeansmpx. Table~\ref{tab:9} summarizes the properties of the original image and the objective values (SSE) of different algorithms. The best SSEs, marked in bold, are achieved by the geometry-based algorithm FFkm (SD+PD) for images Boat, Traffic, Red Panda, and Baboon, and by the objective-based algorithm FFkm (TD+OI) for images Palace, Traffic, Flower, Peppers, and Earth. These results demonstrate the effectiveness of the proposed FFkm framework in real-world scenarios.

\noindent\paragraph{Other real-world datasets} For further comparison, we consider five additional real-world datasets that are widely used in the literature on \kmeans for evaluating alternatives of Lloyd's algorithms~\cite{ismkhan2018ik}. The IRIS dataset, which includes three types of flowers (true $\ktrue=3$) with four features, is used to study the classification accuracy of \kmeans \cite{chakraborty2020comparative}. The Human Activity Recognition Using Smartphones (HAR) dataset has six recorded activities ($\ktrue=6$). In the ISOLET dataset, one aims to predict which letter was spoken ($\ktrue=26$). The Letter Recognition (LR) dataset contains 26 capital letters from the English alphabet ($\ktrue=26$). In the Musk version 2 dataset, one seeks to predict whether new molecules will be musks or non-musks ($\ktrue=2$).

We compare the objective values (SSE) achieved by Lloyd's \kmeans, FFkm (SD+PD), FFkm (TD+OI), and \kmeansmpx.
With $50$ independent trials, the average SSEs and execution times are reported in Table~\ref{tab:10}. For each dataset, the table also gives the number of ground truth clusters $\ktrue$ and data size (number of data points $\times$ number of features/dimensions). Note that although the Musk dataset has 168 features, only 166 integer features are utilized. %The best SSE results are marked in bold. 
As observed in Table~\ref{tab:10}, FFkm (SD+PD), FFkm (TD+OI), and \kmeansmpx all achieve better SSE than Lloyd's \kmeans; in particular, they avoid the local minima that trap Lloyd's \kmeans. Among them, the objective-based algorithms FFkm (TD+OI) and \kmeansmpx have the best SSE in three and two datasets, respectively. FFkm (SD+PD), which is geometry-based, achieves the best results in the IRIS and Musk datasets. Moreover, FFkm (SD+PD) reports the fastest execution time (excluding Lloyd's \kmeans) in the HAR, ISOLET, LR, and Musk datasets.

\begin{table}
  \centering
  \caption{Results of SSE and average execution time in seconds (Time (s)) (best results of SSE in boldface)}\label{tab:10}
  % \vspace{-.15in}
  \scalebox{0.8}{
  \begin{tabular}{|c | c | c | c | c | c | c | c | c | c | c| c|}
  \hline 
  \textbf{Dataset} & \textbf{Data Size} & \multicolumn{2}{c|}{\textbf{Lloyd}}     &  \multicolumn{2}{c|}{\textbf{FFkm (SD+PD)}}   &  \multicolumn{2}{c|}{\textbf{FFkm (TD+OI)}}  &  \multicolumn{2}{c|}{\textbf{\kmeansmpx}} \\ 
  \cline{3-10}
                   &                    & \textbf{ave. SSE}          & \textbf{Time(s)}   & \textbf{ave. SSE}        & \textbf{Time(s)}          & \textbf{ave. SSE}        & \textbf{Time(s)}         & \textbf{ave. SSE} & \textbf{Time(s)} \\ 
  \hline
  Iris ($\ktrue=3$)       & $150\times 3$      & $9.308\times 10^1$   &  0.0021          & $\bm{7.885\times 10^1}$  & 0.0053                 & $\bm{7.885\times 10^1}$  & 0.0048                & $\bm{7.885\times 10^1}$  &0.0032 \\
  HAR ($\ktrue=6$)        & $10299\times 561$  & $1.851\times 10^5$   &  0.3021          & $1.851\times 10^5$  & 0.8207                 & $1.825\times 10^5$  & 2.4546                & $\bm{1.823\times 10^5}$  &4.0628 \\
  ISOLET ($\ktrue=26$)    & $7797\times 617$   & $4.465\times 10^5$   &  0.6423          & $4.454\times 10^5$  & 1.5728                 & $\bm{4.406\times 10^5}$  & 4.8862                & $4.414\times 10^5$  &17.8700 \\
  LR ($\ktrue=26$)        & $20000\times 16$   & $6.201\times 10^5$   &  0.0874          & $6.196\times 10^5$  & 0.2435                 & $\bm{6.183\times 10^5}$  & 0.9948                & $6.184\times 10^5$  &0.9166 \\
  Musk ($\ktrue=2$)       & $6598\times 166$   & $6.090\times 10^9$   &  0.0247          & $\bm{5.922\times 10^9}$  & 0.1670                 & $5.923\times 10^9$  & 0.1978                & $5.983\times 10^9$  &0.3165\\
  \hline
  \end{tabular}}
  % \vspace{-.1in}
\end{table}

\vspace{-.06in}
\section{Discussion}
\label{sec:disc}
\vspace{-.05in}
%In this section, we provide several additional remarks.

\paragraph{Alternative Improvements on the Unbalanced Dataset} The {Unbalance} dataset presents unique challenges due to the significant differences in cluster densities and sizes. These characteristics sometimes lead to poor performance when the number of clusters is over-specified, e.g., for the Fusion-only \kmeans. In particular, this dataset contains two groups of clusters that are far from each other. The first group contains 3 clusters with smaller variances but more data points, while the second group contains 5 clusters with larger variances but fewer data points. Consequently, with random initialization, more initial centers (than the true number of clusters) are associated with the first group of clusters, and fewer initial centers are associated with the second group. Since Fusion-only \kmeans only performs fusion steps, the fitted centers can over-merge, leading to a suboptimal solution. This issue can be mitigated by further increasing the initial cluster number, so that there are a sufficient number of initial centers in the right part of the clusters. Indeed, as seen in Table \ref{tab:11} (Appendix \ref{app:discuss}). Fusion-only \kmeans with $k=20\ktrue$ achieves the best performance.

\noindent\paragraph{The $\epsilon$-Radius (RD) and  the S4 Dataset} The choice of radius in the $\epsilon$-radius (RD) method plays an important role on the highly overlapping S4 dataset. In an additional experiment, the radius is set to $\delta \cdot r$, where $\delta$ is selected from $\{0.001, 0.01, 0.1, 0.25, 0.5, 1, 2, 5\}$. For the merge step (2b) in FFkm, we merge two detected centroids (using either the PD or OI method) to their mean. Table~\ref{tab:S4com} (Appendix \ref{app:discuss}) shows a solution produced by the geometry-based algorithm FFkm (RD+PD), which finds the ground truth clusters in this challenging S4 dataset. This result demonstrates that decreasing the $\epsilon$ of RD can be helpful in such scenarios.

\noindent\paragraph{Further Optimization of FFkm} Spectral clustering is another class of clustering algorithms with strong performance. One such method is the rank-constrained spectral clustering with flexible embedding proposed by \cite{li2018rank}. This approach introduces rank constraints to the clustering problem, producing an accurate and robust low-dimensional embedding with enhanced geometric properties. Using such embeddings within our FFkm framework can be particularly beneficial for detecting mis-specified clusters and may lead to better performance. A related work \cite{li2018dynamic} discusses dynamic affinity graph construction for spectral clustering using multiple features, which can capture complex data relationships. 
%Incorporating such dynamic, feature-based graphs into our geometric approach can significantly improve the initial identification of clusters and the refinement of solutions.
Incorporating dynamic affinity graphs into our FFkm framework has the potential of enhancing the efficacy of our non-local operations and the quality of the initial clustering solution, both are crucial for FFkm.

%These works demonstrate the potential for combining different methods to enhance the robustness and effectiveness of clustering. By integrating rank constraints and dynamic affinity graphs, we can further refine the FFkm framework and achieve superior clustering performance.

\vspace{-.06in}
\section{Conclusion}
\label{sec:con}
\vspace{-.05in}

We propose a flexible framework for \kmeans problem by harnessing the geometric structure of local solutions. It provides a theoretical foundation for future work to design detection routines for varying cluster distributions. Future work includes analyzing the Fission-Fission \kmeans under the more general setting with empirical success: (i) clusters could be of different sizes and shapes; (ii) clusters have moderate or heavy overlaps with each other.
\section*{Acknowledgement}
% \vspace{-.05in}
 J.\ Hong and Y.\ Zhang acknowledges support from the Department of Electrical and Computer Engineering at Rutgers University. W.\ Qian and Y.\ Chen are partially supported by NSF grants CCF-1657420, CCF-1704828 and CCF-2233152.

%\newpage
% \vspace{-.06in}
{
\bibliographystyle{IEEEtran}
\bibliography{clustering}
}

\clearpage
\newpage 
\appendix

\section{Proof for Theorem 3.1}
\label{app:proof}

\begin{proof}
The $\kmeans$ objective function is 
\begin{align}
\label{app:eq-kmeans-obj}
G_{n}(\bbeta) = \frac{1}{n}\sum_{t=1}^{n} \min_{s\in [k^*]} \|\mx_t - \bbeta_s\|^2.
\end{align} 
Note that the output solution of the algorithm is no worse than the input solution, it suffices to consider the case in which the initial solution $\bbeta^{(0)}$ is a non-degenerate local minimum with a strictly worse objective value than that of a global optimal solution. Note that we fit the data with $\ktrue$ cluster, there exists a fitted cluster with \ofm association. Assume this is not the case, then $\bbeta^{(0)}$ only has \mfo associations. This means that for each true cluster center, there is at least a fitted center close to it. Since $k=\ktrue$, each true cluster must be associated with exactly one fitted cluster. In particular, $\bbeta^{(0)}$ is global optimal, a contradiction.  Similarly, $\bbeta^{(0)}$ must have two fitted clusters with \mfo associations. 

We first characterize the amount of improvement of the \kmeans objective in each fission-and-fusion step. Specifically, Lemma~\ref{lem:improve} shows that 
\[
G_{n}(\bbeta^{(\frac{1}{2})})\le G_{n}(\bbeta^{(0)})-\frac{\Deltamin^{2}}{72\ktrue}.
\]
Since the Lloyd step (Step 3 of the algorithm) does not increase the \kmeans objective (Lemma~\ref{lem:monotone}), it follows that 
\[
G_{n}(\bbeta)^{(1)}) \le G_{n}(\bbeta^{(\frac{1}{2})})\le G_{n}(\bbeta^{(0)})-\frac{\Deltamin^{2}}{72\ktrue}.
\]

Thus after each the first iteration of the Fission-Fussion algorithm, the decrement in the \kmeans objective function is at least a constant. Moreover, $\bbeta^{(1)}$ is non-degenerate local minimum from the argument in Lemma~\ref{lem:monotone}. We can then apply the above argument recursively, obtaining
\[
G_{n}(\bbeta^{(\ell)})\le G_{n}(\bbeta^{(0)})- \ell \cdot \frac{\Deltamin^{2}}{72\ktrue}.
\] 

Note that the \kmeans objective value is upper bounded uniformly at every non-degenerate local minimum by Lemma~\ref{lem:upper}, and lower bounded by $0$. In particular, 
\begin{align*}
0 \leq G_{n}(\bbeta^{(0)}) \le 4\Delta_{\max}^{2}.
\end{align*} 
This means that there are at most $\frac{288\ktrue \Deltamax^2}{\Deltamin^2}$ Fission-Fussion iterations in the algorithm. At the termination of the algorithm, we must obtain a global optimal solution. 

\end{proof}

\begin{lemma}[Upper Bound]
\label{lem:upper}
Under the assumption of Theorem 3.1, suppose $\bbeta$ is a non-degenerate local minimum of $G_{n}$. 
Then 
\[
G_{n}(\bbeta)\le4\Deltamax^{2}.
\]
\end{lemma}
\begin{proof}
When $\bbeta$ is a non-degenerate local minimum solution, each fitted cluster is non-empty. From Lemma 2 of \cite{qian2021structures}, each $\bbeta_i$ is at the center of the  corresponding Voronoi set. In particular, each $\bbeta_i$ must be in the convex hull of all the balls. For each $i\in [\ktrue]$, we conveniently represent $\bbeta_i$ as
\begin{align*}
\bbeta_i = \sum_{s\in [\ktrue]} \alpha_{i,s} (\bbetastar_s+\mv_{i,s}),
\end{align*}
where $\left\{\mv_{i,s}\right\}_{s\in [\ktrue]}$ are vectors in $\real^d$ with norm bounded by $r$; $\alpha_{i,s}$s' are non-negative scalars satisfying $\sum_{s\in [\ktrue]} \alpha_{i,s} =1$. 
For each $t\in [n]$, assuming without loss of generality that $\mx_{t}$ is generated from $f_{1}^*$, we have
\begin{align}
\big \|\mx_t - \bbeta_i \big \|= & \Big \Vert \bbetastar_1 + \muu - \sum_{s\in [\ktrue]} \alpha_s (\bbetastar_s+\mv_s) \Big \Vert \label{eq:represent}\\
                      = & \Big \|\sum_{s\in [\ktrue]} \alpha_s ( \bbetastar_1  - \bbetastar_s) + \sum_{s\in [\ktrue]} \alpha_s  (\muu - \mv_s) \Big\| \nonumber \\
                      \leq & \Deltamax + 2r \leq 2\Deltamax \label{eq:triangle}.
\end{align}
In equation~\ref{eq:represent}, we represent $\mx_t = \bbetastar_1 + \muu$ for some vector $\muu\in \real^d$ with norm bounded by $r$. In inequality~\ref{eq:triangle}, we apply the triangle inequality and the SNR assumption that $\frac{\Deltamax}{r}\ge \frac{\Deltamin}{r}\ge 30$. 

It follows that $G_{n}(\bbeta)\le\frac{1}{n}\sum_{t\in[n]}\left(r+\Deltamax\right)^{2}\le 4\Deltamax^{2}$.

\end{proof}

\begin{lemma}[Improvement]
\label{lem:improve} Under the assumption of Theorem 3.1, for each $\ell=0, 1,2,\ldots$,
\[
G_{n}(\bbeta^{(\ell+\frac{1}{2})})\le G_{n}(\bbeta^{(\ell)})-\frac{\deltamin^{2}}{72\ktrue}.
\]
Moreover, each $\bbeta^{(\ell)}$ is non-degenerate. 
\end{lemma}
\begin{proof}
Using the structural result of a local minimum, if $\bbeta^{(\ell)}$ is a non-degenerate local minima that is not globally optimal, then $\bbeta^{(\ell)}$ must have both \ofm and \mfo association. Without loss of generality, suppose that in the local minimizer $\bbeta^{(\ell)}$, the center
$\bbeta_{1}$ fits multiple true centers $\bbetastar_{1},\ldots,\bbetastar_{m}$,
where $m\ge2$, and the centers $\bbeta_{2}$ and $\bbeta_{3}$ (potentially
along with some other centers) fit the true center $\bbetastar_{m+1}$.
For $i=1,\ldots,s+1$, let $C_{i}$ index the data points generated
from $f_{i}^{*}$.

Suppose that in the new solution $\bbeta^{(\ell+\frac{1}{2})},$ $\bbeta_{1}$
is split into two centers $\bbeta_{<}$ and $\bbeta_{>}$, where $\bbeta_{<}=\bbeta_1$ and $\bbeta_{>}$
is moved to a data point $\mb x_{0}=\arg\max_{\mb x\in\vor(\bbeta_{1})}\left\Vert \mb x-\bbeta_{1}\right\Vert $, where $\vor(\bbeta_{1})$
is the Voronoi set associated with $\bbeta_{1}$;
moreover, $\bbeta_{2}$ and $\bbeta_{3}$ are merged into one center
$\bbeta_{=}=\frac{1}{2}\left(\bbeta_{2}+\bbeta_{3}\right)$. Suppose
that $\mb x_{0}$ is from the cluster $C_{1}$ generated from $f_{1}^{*}$.

We illustrate the above notations in Figure \ref{fig:split_merge}.

\begin{figure*}
\noindent \begin{centering}
\includegraphics[width=0.7\textwidth]{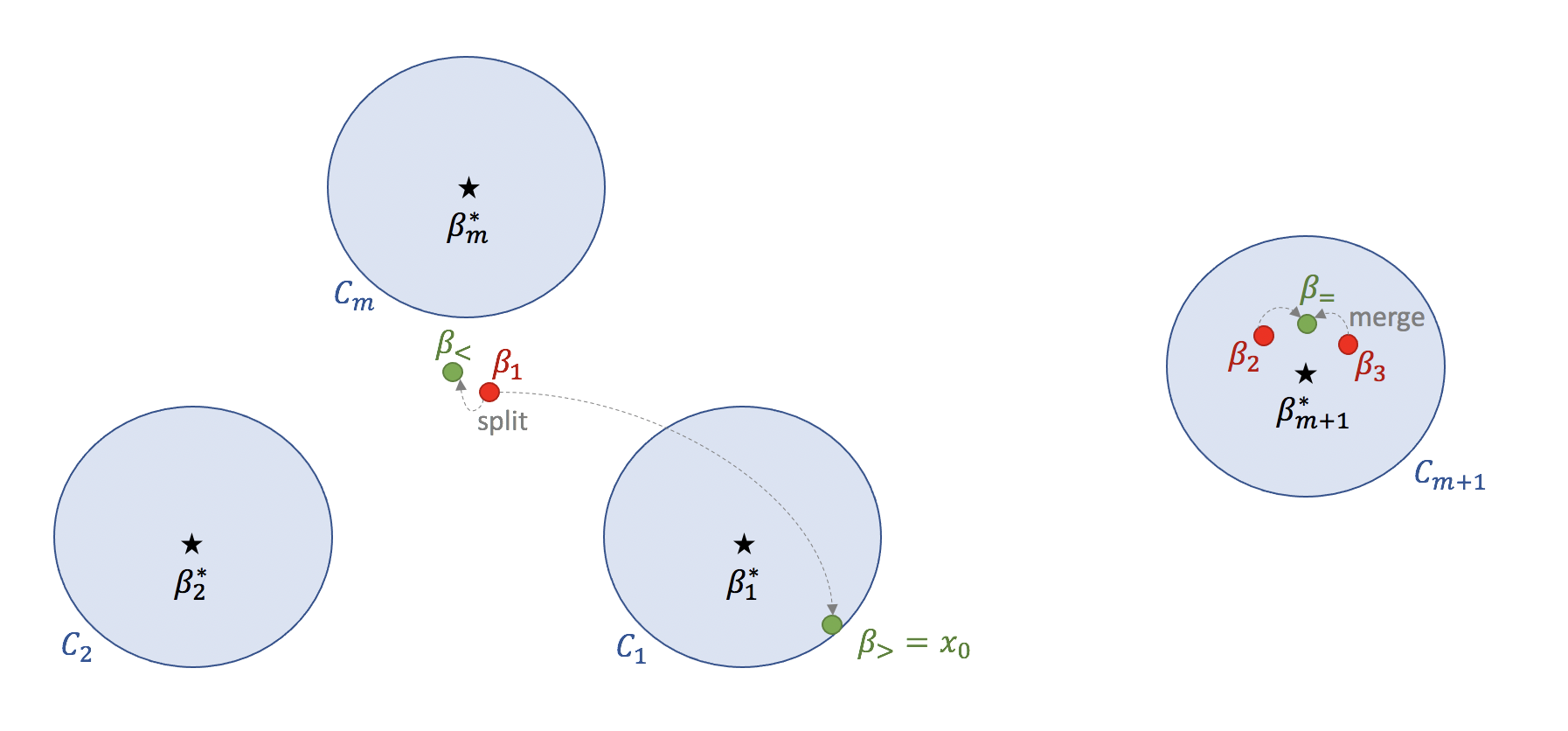}\bigskip{}
\par\end{centering}
\caption{Illustration for Lemma~\ref{lem:improve}\label{fig:split_merge}}
\end{figure*}

The objective values before and after Fission and Fusion are 
\begin{align*}
&G_{n}(\bbeta^{(\ell)}) \\
&=\frac{1}{n}\Bigg(\sum_{i\in[m]}\sum_{t\in C_{i}}\left\Vert \mb x_{t}-\bbeta_{1}\right\Vert ^{2} \\
&\quad +\sum_{t\in C_{m+1}}\min\left\{ \left\Vert \mb x_{t}-\bbeta_{2}\right\Vert ^{2},\left\Vert \mb x_{t}-\bbeta_{3}\right\Vert ^{2}\right\} \Bigg)+B
\end{align*}
and
\begin{align*}
&G_{n}(\bbeta^{(\ell+\frac{1}{2})})\\
&\le\frac{1}{n}\Bigg(\sum_{i\in[m]}\sum_{t\in C_{i}}\min\left\{ \left\Vert \mb x_{t}-\bbeta_{<}\right\Vert ^{2},\left\Vert \mb x_{t}-\bbeta_{>}\right\Vert ^{2}\right\} \\
&\quad +\sum_{t\in C_{m+1}}\left\Vert \mb x_{t}-\bbeta_{=}\right\Vert ^{2}\Bigg)+B,
\end{align*}
where $B$ is the part of the objective value that involves data points indexed by $\left\{C_i\right\}_{i>m+1}$. $B$ does not change
by Fission and Fusion step. Therefore, the improvement can be lower bounded
as
\begin{align*}
G_{n}(\bbeta^{(\ell)})-G_{n}(\bbeta^{(\ell+\frac{1}{2})})
\ge T_{+} - T_{-} \,,
\end{align*}
where
\begin{align*}
T_{+} := &\frac{1}{n}\left(\sum_{i\in[m]}\sum_{t\in C_{i}}\left\Vert \mb x_{t}-\bbeta_{1}\right\Vert ^{2}\right. \\
&\quad \quad \left.-\sum_{i\in[m]}\sum_{t\in C_{i}}\min\left\{ \left\Vert \mb x_{t}-\bbeta_{<}\right\Vert ^{2},\left\Vert \mb x_{t}-\bbeta_{>}\right\Vert ^{2}\right\} \right)
\end{align*}
and
\begin{align*}
T_{-} := & \frac{1}{n}\left(\sum_{t\in C_{m+1}}\left\Vert \mb x_{t}-\bbeta_{=}\right\Vert ^{2}\right. \\
&\quad \quad \left.-\sum_{t\in C_{m+1}}\min\left\{ \left\Vert \mb x_{t}-\bbeta_{2}\right\Vert ^{2},\left\Vert \mb x_{t}-\bbeta_{3}\right\Vert ^{2}\right\} \right).
\end{align*}
% \begin{align*}
% &G_{n}(\bbeta^{(\ell)})-G_{n}(\bbeta^{(\ell+\frac{1}{2})})\\
% &\ge \underbrace{\frac{1}{n}\left(\sum_{i\in[m]}\sum_{t\in C_{i}}\left\Vert \mb x_{t}-\bbeta_{1}\right\Vert ^{2}\right.}_{T_{+}} \\
% &\quad \quad \underbrace{\left.-\sum_{i\in[m]}\sum_{t\in C_{i}}\min\left\{ \left\Vert \mb x_{t}-\bbeta_{<}\right\Vert ^{2},\left\Vert \mb x_{t}-\bbeta_{>}\right\Vert ^{2}\right\} \right)}_{T_{+}} \\
% &\quad -\underbrace{\frac{1}{n}\left(\sum_{t\in C_{m+1}}\left\Vert \mb x_{t}-\bbeta_{=}\right\Vert ^{2}\right.}_{T_{-}} \\
% &\quad \quad \underbrace{\left.-\sum_{t\in C_{m+1}}\min\left\{ \left\Vert \mb x_{t}-\bbeta_{2}\right\Vert ^{2},\left\Vert \mb x_{t}-\bbeta_{3}\right\Vert ^{2}\right\} \right)}_{T_{-}}.
% \end{align*}
Note that the $T_{+}$ term is the improvement in objective due to
splitting $\bbeta_{1}$ into $\bbeta_{<}$ and $\bbeta_{>}$, where
the $T_{-}$ term is the potential loss in objective value due to
merging $\bbeta_{2}$ and $\bbeta_{3}$ into $\bbeta_{=}$. Let us
control these terms separately.

First consider the $T_{+}$ term. We claim that $\left\Vert \mb x_{0}-\bbeta_{1}\right\Vert \ge\frac{1}{3}\Deltamin.$
Proof of claim: Let $\mb x'$ be generated from $f_{2}^{*}$. By triangle
inequality we have 
\begin{align*}
\Deltamin & \le\left\Vert \bbetastar_{1}-\bbetastar_{2}\right\Vert \\
 & \le\left\Vert \bbetastar_{1}-\mb x_{0}\right\Vert +\left\Vert \mb x_{0}-\bbeta_{1}\right\Vert +\left\Vert \bbeta_{1}-\mb x'\right\Vert +\left\Vert \mb x'-\bbetastar_{2}\right\Vert \\
 & \le r +\left\Vert \mb x_{0}-\bbeta_{1}\right\Vert +\left\Vert \mb x_{0}-\bbeta_{1}\right\Vert +r,
\end{align*}
where the last step follows from the choice of $\mb x_{0}$. Rearranging
the above equation gives $$\left\Vert \mb x_{0}-\bbeta_{1}\right\Vert \ge\frac{\Deltamin}{2}-r \ge\frac{\Deltamin}{3}$$ since $\frac{\Deltamin}{r} \ge 30$. 
thereby proving the claim. 

% It follows that 
% \begin{align}
% T_{+} &=\frac{1}{n}\left(\sum_{i\in[m]}\sum_{t\in C_{i}}\left\Vert \mb x_{t}-\bbeta_{1}\right\Vert ^{2}\right. \label{eq:choice_beta} \\
% &\quad \left.-\sum_{i\in[m]}\sum_{t\in C_{i}}\min\left\{ \left\Vert \mb x_{t}-\bbeta_{1}\right\Vert ^{2},\left\Vert \mb x_{t}-\mb x_{0}\right\Vert ^{2}\right\} \right) \\
%  & \ge\frac{1}{n}\left(\sum_{t\in C_{1}}\left\Vert \mb x_{t}-\bbeta_{1}\right\Vert ^{2}-\sum_{t\in C_{1}}\left\Vert \mb x_{t}-\mb x_{0}\right\Vert ^{2}\right) \label{eq:dropC}\\
% &\ge\frac{1}{n}\left(\sum_{t\in C_{1}}\left(\left\Vert \mb x_{0}-\bbeta_{1}\right\Vert -\left\Vert \mb x_{t}-\mb x_{0}\right\Vert \right)^{2}\right. \nonumber \\
% &\quad \left.-\sum_{t\in C_{1}}\left\Vert \mb x_{t}-\mb x_{0}\right\Vert ^{2}\right) \label{eq:triangle_ineq} \\
% & \ge\frac{1}{n}\left(\sum_{t\in C_{1}}\left(\frac{1}{3}\Deltamin-2r\right)^{2}-\sum_{t\in C_{1}}(2r)^{2}\right) \label{eq:claim}\\
%  & \ge\frac{|C_1|}{n}\cdot\frac{1}{18}\Deltamin^{2} \ge \frac{\Deltamin^2}{36\ktrue} \label{eq:h-snr}.
% \end{align}

From the choice of $\bbeta_{<}$ and $\bbeta_{>}$, we have 
\begin{align}
T_{+} &=\frac{1}{n}\left(\sum_{i\in[m]}\sum_{t\in C_{i}}\left\Vert \mb x_{t}-\bbeta_{1}\right\Vert ^{2}\right. \label{eq:choice_beta} \\
&\quad \left.-\sum_{i\in[m]}\sum_{t\in C_{i}}\min\left\{ \left\Vert \mb x_{t}-\bbeta_{1}\right\Vert ^{2},\left\Vert \mb x_{t}-\mb x_{0}\right\Vert ^{2}\right\} \right) 
\end{align}
Noting that 
$$\sum_{t\in C_{i}}\left\Vert \mb x_{t}-\bbeta_{1}\right\Vert ^{2}-\sum_{t\in C_{i}}\min\left\{ \left\Vert \mb x_{t}-\bbeta_{1}\right\Vert ^{2},\left\Vert \mb x_{t}-\mb x_{0}\right\Vert ^{2}\right\}$$ is non-negative for $i=2,\ldots, m$,
we obtain
\begin{align}
T_{+} 
 & \ge\frac{1}{n}\left(\sum_{t\in C_{1}}\left\Vert \mb x_{t}-\bbeta_{1}\right\Vert ^{2}-\sum_{t\in C_{1}}\left\Vert \mb x_{t}-\mb x_{0}\right\Vert ^{2}\right).
 \label{eq:dropC}
\end{align}
It follows that
\begin{align}
T_{+}
&\ge\frac{1}{n}\left(\sum_{t\in C_{1}}\left(\left\Vert \mb x_{0}-\bbeta_{1}\right\Vert -\left\Vert \mb x_{t}-\mb x_{0}\right\Vert \right)^{2}\right. \nonumber \\
&\quad \left.-\sum_{t\in C_{1}}\left\Vert \mb x_{t}-\mb x_{0}\right\Vert ^{2}\right) \label{eq:triangle_ineq} \\
& \ge\frac{1}{n}\left(\sum_{t\in C_{1}}\left(\frac{1}{3}\Deltamin-2r\right)^{2}-\sum_{t\in C_{1}}(2r)^{2}\right) \label{eq:claim}\\
 & \ge\frac{|C_1|}{n}\cdot\frac{1}{18}\Deltamin^{2} \ge \frac{\Deltamin^2}{36\ktrue} \label{eq:h-snr}. 
\end{align}
Above, inequality~\eqref{eq:triangle_ineq} follows from the triangle inequality. Inequality~\eqref{eq:claim} follows from the proved claim, and the fact that both $\mx_0$ and $\mx_t$ are from the ball centered at $\bbetastar_1$. Finally, inequality~\eqref{eq:h-snr} follows from $\frac{\Deltamin}{r}\ge 30$ and the probabilistic bound on $|C_i|$ from Lemma~\ref{lem:equal_size}.

Turning to the $T_{-}$ term, we have 
\begin{align*}
T_{-} & =\frac{1}{n}\left(\sum_{t\in C_{m+1}}\left\Vert \mb x_{t}-\bbeta_{=}\right\Vert ^{2}\right.\\
 &\quad -\left.\sum_{t\in C_{m+1}}\min\left\{ \left\Vert \mb x_{t}-\bbeta_{2}\right\Vert ^{2},\left\Vert \mb x_{t}-\bbeta_{3}\right\Vert ^{2}\right\} \right) \\
 & \le\frac{1}{n}\sum_{t\in C_{m+1}}\left\Vert \mb x_{t}-\frac{1}{2}\left(\bbeta_{2}+\bbeta_{3}\right)\right\Vert ^{2}\\
 & \le\frac{|C_{m+1}|}{n}\cdot(2r)^{2}. 
\end{align*}
In the last inequality, we note that $\bbeta_{2}$ and $\bbeta_{3}$ are both inside the ball since they split the ball centered at $\bbetastar_{m+1}$. Therefore, $\frac{\bbeta_{2}+\bbeta_{3}}{2}$ is also inside the ball.  Applying the upper bound on $|C_{m+1}|$ from Lemma~\ref{lem:equal_size}, we have that $T_{-} \leq \frac{8r^2}{\ktrue}$. 

Combining pieces, we have 
\begin{align*}
G_{n}(\bbeta^{(\ell)})-G_{n}(\bbeta^{(\ell+\frac{1}{2})}) & \ge T_{+}-T_{1}\\
 & \ge\frac{1}{36\ktrue}\Deltamin^{2}-\frac{8r^2}{\ktrue}\\
 & \ge\frac{1}{72\ktrue}\Deltamin^{2},
\end{align*}
thereby proving the lemma. We note that the modified solution $\bbeta^{(\ell+\frac{1}{2})}$ is non-degenerate: $\bbeta_{>}$ is at least associated with the data points generated by the ball centered at $\bbetastar_{1}$, and $\bbeta_{<}$ is associated the rest of the data points. After applying Lloyd's algorithm, $\bbeta^{(\ell+1)}$ is also non-degenerate. 
\end{proof}
%Our last lemma follows from the monotonicity property of the Lloyd's
%algorithm.
\begin{lemma}[Monotonicity]
\label{lem:monotone}For each $\ell=1,2,\ldots,$ we have
\[
G_{n}(\bbeta^{(\ell+1)})\le G_{n}(\bbeta^{(\ell+\frac{1}{2})}).
\]
\end{lemma}
\begin{proof}
    This follows from the monotonicity property of the Lloyd's algorithm: after each iteration of Lloyd's algorithm, the objective value does not increasing.
\end{proof}
\begin{lemma}[Almost Equal Size]
\label{lem:equal_size}
Under the assumption of Theorem 3.1, let $C_i$ be the indices of data points generated by $f_i^*$ for each $i\in [\ktrue]$. With probability at least $1-2k\exp\left(-\frac{n}{2k^{*2}}\right)$,
\[
\frac{n}{2\ktrue}\le |C_i| \le \frac{3n}{2\ktrue} \quad \forall i\in [\ktrue].
\]
\end{lemma}
\begin{proof}
Fix $i\in [\ktrue]$. For each $t\in [n]$, define 
\begin{align*}
Z_t=\begin{cases} 1 & t\in C_i \\ 0 & t\notin C_i \end{cases}.
\end{align*}
$Z_1, Z_2,\ldots, Z_n$ are i.i.d Bernoulli random variables with $p=\frac{1}{\ktrue}$. Hoeffding's bounds shows that 
\begin{align*}
\mathbb{P}\left(\sum_{t=1}^{n} Z_t \ge \frac{n}{2\ktrue}\right) \leq \exp\left(-\frac{n}{2k^{*2}}\right) \\
\mathbb{P}\left(\sum_{t=1}^{n} Z_t \le \frac{3n}{2\ktrue}\right) \leq \exp\left(-\frac{n}{2k^{*2}}\right).
\end{align*}
Applying the union bound gives us the result. 
\end{proof}

\section{Proof of Theorem 3.2}
\label{app:proof_thm_converse}
%~\ref{thm:converse}
%We restate the theorem for reader's convenience. 
%\converse*
Throughout the proof we consider the one-dimensional setting of a
stochastic ball model with unit radius. We introduce some notations.
For a probability density function $f$, denotes its support by $\supp(f):=\left\{\mb x\in\real:f(\mb x)>0\right\} $.
For a set of $k$ centers $\bbeta=(\bbeta_{1},\ldots,\bbeta_{k})$,
let 
\[
\vor_{i}(\bbeta):=\left\{ \mb x\in\supp(f^{*}):\left|\mb x-\bbeta_{i}\right|\le\left|\mb x-\bbeta_{j}\right|,\forall j\in[k]\right\} 
\]
denote the Voronoi set of the center $\bbeta_{i}$. For each set $S\in\real$,
let 
\[
\mean(S):=\frac{\int\mb x\mb 1\left\{ \mb x\in S\right\} f^{*}(\mb x)\ddup\mb x}{\int\mb 1\left\{ \mb x\in S\right\} f^{*}(\mb x)\ddup\mb x}
\]
denote its center of mass. With these notations, the Lloyd's k-means
algorithm is given by the iteration
\[
\bbeta_{i}^{(t+1)}=\mean\left(\vor_{i}(\bbeta^{(t)})\right),\qquad i\in[k],t=0,1,\ldots,
\]
with the convention that $\bbeta_{i}^{(t+1)}=\bbeta_{i}^{(t+1)}$
if $\vor_{i}(\bbeta^{(t)})=\emptyset.$

Let $\ball_{x}(\delta):=[x-\delta,x+\delta]$ and $\overline{\ball_{x}(\delta)}:=(-\infty,x-\delta)\cup(x+\delta,\infty)$.
We need the following definition.
\begin{defn}[Diffuse Stochastic Ball Model]
\label{def:diffuseSBM}We say that a stochastic ball model with $\widetilde{k}$
components is $(c,\delta)$-diffuse if
\begin{enumerate}
\item For some $k\le\widetilde{k}$, there are $k$ true centers contained
in $\ball_{c\delta}(\delta)\cup\ball_{-c\delta}(\delta);$
\item Each of the sets $\ball_{c\delta}(\delta)$ and $\ball_{-c\delta}(\delta)$
contain at least one true center;
\item The remaining $\widetilde{k}-k$ centers are all in $\overline{\ball_{0}(20c\delta)}$.
\end{enumerate}
Consider the Lloyd's algorithm, and denote by $k_{1}^{(t)},k_{2}^{(t)}$
and $k_{3}^{(t)}$ the number of fitted centers in the $t$-th iteration
in the sets $\ball_{c\delta}(\delta)$, $\ball_{-c\delta}(\delta)$
and $\overline{\ball_{0}(20c\delta)}$, respectively. With these definitions,
we establish a key technical lemma on the behaviors of Lloyd's algorithm.
\end{defn}
\begin{lemma}
\label{lem:trap}Suppose that the true stochastic ball model has $\widetilde{k}$
true centers and is $(c,\delta)$-diffuse with $c>20$ and $\delta>3$,
and that the Lloyd's algorithm is initialized so that $k_{1}^{(0)}\ge1,k_{2}^{(0)}\ge1$.
\begin{enumerate}
\item If $k=\widetilde{k}$, then 
\begin{equation}
k_{1}^{(t)}=k_{1}^{(0)}\quad\text{and}\quad k_{2}^{(t)}=k_{2}^{(0)},\qquad\forall t\ge0.\label{eq:trap}
\end{equation}
\item If $k<\widetilde{k}$, suppose further that for each true center $\bbetastar_{s}$
in $\overline{\ball_{0}(20c\delta)}$, there is an initial fitted
center $\bbeta_{i}^{(0)}$ such that $\left|\bbeta_{i}^{(0)}-\bbetastar_{s}\right|\le\left|\bbetastar_{s}\right|/10$.
Then the same conditions (\ref{eq:trap}) hold.
\end{enumerate}
\end{lemma}
We prove this lemma in Section~\ref{sec:proof_lem_trap} to follow.
We note that our proof differs substantially from that in \cite{jin2016local},
which considers Gaussian mixtures and the EM algorithm. Our setting
requires a different set of arguments due to the nonsmoothness of
the k-means objective as well as Voronoi sets being involved in the
Lloyd's algorithm.

Equipped with Lemma~\ref{lem:trap}, we can follow the same arguments
in the proof of \cite[Theorem 2]{jin2016local}, with Lemma 1 therein
replaced by Lemma~\ref{lem:trap}, to establish the following:
\begin{prop}
\label{prop:bad_initialization}For each integer $k\ge3$ and real
number $R>0$, there exists a stochastic ball model with $k$ components
such that the event 
\[
\mathcal{E}:=\left\{ \forall t\ge0:\max_{s\in[k]}\min_{i\in[k]}\left|\bbeta_{i}^{(t)}-\bbetastar_{s}\right|\ge\frac{R}{k^{7}}\right\} 
\]
holds with probability at least $1-e^{-ck}$ under random initialization
of Lloyd's algorithm, where $c>0$ is a universal constant.
\end{prop}
We are ready to complete the proof of Theorem 3.2.
Under the event $\mathcal{E}$ in Proposition~\ref{prop:bad_initialization},
for each $t\ge 0$, the objective value of iterate $\bbeta^{(t)}$
of Lloyd's algorithm satisfies 
\begin{align*}
G(\bbeta^{(t)}) & =\frac{1}{k}\sum_{s\in[k]} \E_{\mb x\sim f_{s}^{*}} \big[\min_{i\in [k]} \big\|\mb x-\bbeta_{i}^{(t)}\big\|^{2} \big] \\
 & \ge\frac{1}{k}\max_{s\in[k]}\E_{\mb x\sim f_{s}^{*}}\left[\min_{i\in[k]}\left|\mb x-\bbeta_{i}^{(t)}\right|^{2}\right]\\
 & \ge\frac{1}{k}\max_{s\in[k]}\E_{\mb x\sim f_{s}^{*}}\left[\min_{i\in[k]}\left(\left|\bbetastar_{s}-\bbeta_{i}^{(t)}\right|-\left|\mb x-\bbetastar_{s}\right|\right)^{2}\right]\\
 & \ge\frac{1}{k}\max_{s\in[k]}\min_{i\in[k]}\left(\left|\bbetastar_{s}-\bbeta_{i}^{(t)}\right|-1\right)^{2}\\
 & \ge\frac{1}{k}\cdot \left(\frac{R}{k^{7}}-1\right)^{2}
\end{align*}
as long as $R$ is sufficiently large. On the other hand, the true
centers $\bbetastar$ satisfy
\begin{align*}
G(\bbeta^{*}) & =\frac{1}{k}\sum_{s\in[k]}\E_{\mb x\sim f_{s}^{*}}\left[\min_{i\in[k]}\left|\mb x-\bbetastar_{i}\right|^{2}\right]\\
& \le\frac{1}{k}\sum_{s\in[k]}\E_{\mb x\sim f_{s}^{*}}\left[\left|\mb x-\bbetastar_{s}\right|^{2}\right]\\
& \le\frac{1}{k}\sum_{s\in[k]}1=1.
\end{align*}
Therefore, taking $R=k^{7}\left(\sqrt{k(\Cgap+1)}+1\right)$, we can
ensure that 
\[
\frac{G(\bbeta^{(t)})-G(\bbetastar)}{G(\bbetastar)}\ge\Cgap,
\]
thereby proving Theorem 3.2.

\subsection{Proof of Lemma~\ref{lem:trap}}
\label{sec:proof_lem_trap}
For each $t\ge0$, define the index sets $I_{1}^{(t)}=\left\{ i\in[k]:\bbeta_{i}^{(t)}\in\ball_{c\delta}(\delta)\right\} $,
$I_{2}^{(t)}=\left\{ i\in[k]:\bbeta_{i}^{(t)}\in\ball_{-c\delta}(\delta)\right\} $
and $I_{3}^{(t)}=\left\{ i\in[k]:\bbeta_{i}^{(t)}\in\overline{\ball_{0}(20c\delta)}\right\} $.
By definition we have $k_{j}^{(t)}=\left|I_{j}^{(t)}\right|,j=1,2,3.$

First consider part 1 of the lemma, where $k=\widetilde{k}$. We prove
the claim by induction. The base case $t=0$ holds trivially. Fix
a $t\ge0$ and assume that $k_{1}^{(t)}=k_{1}^{(0)}$ and $k_{2}^{(t)}=k_{2}^{(0)}$.
For each $i\in I_{1}^{(t)}$, since $\bbeta_{i}^{(t)}\in\ball_{-c\delta}(\delta)$,
we have for all $\mb x\in\ball_{c\delta}(\delta)\cap\supp(f^{*})$,
\[
\left|\bbeta_{i}^{(t)}-\mb x\right|\ge2(c-1)\delta\ge2\delta\ge\min_{j\in I_{2}^{(t)}}\left|\bbeta_{j}^{(t)}-\mb x\right|.
\]
It follows that $\vor_{i}(\bbeta^{(t)})\cap\ball_{c\delta}(\delta)=\emptyset$
and hence $\vor_{i}(\bbeta^{(t)})\subseteq\ball_{-c\delta}(\delta)$.
If $\vor_{i}(\bbeta^{(t)})=\emptyset$, then $\bbeta_{i}^{(t+1)}=\bbeta_{i}^{(t)}\in\ball_{-c\delta}(\delta)$
by specification of the algorithm. If $\vor_{i}(\bbeta^{(t)})=\emptyset$,
then $\bbeta_{i}^{(t+1)}=\mean\left(\vor_{i}(\bbeta^{(t)})\right)\in\vor_{i}(\bbeta^{(t)})\subseteq\ball_{-c\delta}(\delta)$.
We conclude that $\bbeta_{i}^{(t+1)}\in\ball_{-c\delta}(\delta),\forall i\in I_{1}^{(t)}$.
A similar argument shows that $\bbeta_{i}^{(t+1)}\in\ball_{c\delta}(\delta),\forall i\in I_{2}^{(t)}$.
Therefore, we have $I_{1}^{(t+1)}=I_{1}^{(t)},I_{2}^{(t+1)}=I_{2}^{(t)}$
and hence $k_{1}^{(t+1)}=k_{1}^{(t)},k_{2}^{(t+1)}=k_{2}^{(t)}$.
This completes the induction step and proves part 1 of the lemma.

Now consider part 2 of the lemma, where $k<\tilde{k}$. We again prove
the claim by induction. Fix a $t\ge0$. Assume that $k_{1}^{(t)}=k_{1}^{(0)}$
and $k_{2}^{(t)}=k_{2}^{(0)}$, and that for each $\bbetastar_{s}\in\overline{\ball_{0}(20c\delta)}$,
there exists $i\in I_{3}^{(t)}$ such that $\left|\bbeta_{i}^{(t)}-\bbetastar_{s}\right|\le\left|\bbetastar_{s}\right|/10$.
For each $i\in I_{3}^{(t)}$, if $\bbeta_{i}^{(t)}\in(-\infty,-20c\delta)$,
we have for all $\mb x\in\left(\ball_{-c\delta}(\delta)\cup\ball_{c\delta}(\delta)\right)\cap\supp(f^{*}),$
\begin{align*}
\left|\bbeta_{i}^{(t)}-\mb x\right|
&\ge20c\delta-2(c+1)\delta \\
&\ge2\delta\\
&\ge\min_{j\in I_{1}^{(t)}\cup I_{2}^{(t)}}\left|\bbeta_{j}^{(t)}-\mb x\right|.
\end{align*}
It follows that $\vor_{i}(\bbeta^{(t)})\cap\left(\ball_{-c\delta}(\delta)\cup\ball_{c\delta}(\delta)\right)=\emptyset$.
It is also clear that $\vor_{i}(\bbeta^{(t)})\cap\left(20c\delta,\infty\right)=\emptyset$.
Therefore, we have $\vor_{i}(\bbeta^{(t)})\subseteq(-\infty,-20c\delta)$
and hence $\bbeta_{i}^{(t+1)}\in(-\infty,-20c\delta)$. The same argument
applies if $i\in I_{3}^{(t)}$ and $\bbeta_{i}^{(t)}\in(20c\delta,\infty)$.
We conclude that $\bbeta_{i}^{(t+1)}\in\overline{\ball_{0}(20c\delta)},\forall i\in I_{3}^{(t)}.$
A similar argument as above shows that $\bbeta_{i}^{(t+1)}\in\ball_{-c\delta}(\delta),\forall i\in I_{1}^{(t)}$
and $\bbeta_{i}^{(t+1)}\in\ball_{c\delta}(\delta),\forall i\in I_{2}^{(t)}$.
Therefore, we have $I_{1}^{(t+1)}=I_{1}^{(t)},I_{2}^{(t+1)}=I_{2}^{(t)}$
and hence $k_{1}^{(t+1)}=k_{1}^{(t)},k_{2}^{(t+1)}=k_{2}^{(t)}$.
This proves part 2 of the lemma.

\section{Split-Merge Based Algorithms}\label{app:related}

Here, we discuss four related algorithms~\cite{pelleg2000x, muhr2009automatic, morii2006clustering, lei2016robust} in detail. The algorithms~\cite{pelleg2000x, muhr2009automatic, lei2016robust} are specifically designed for clustering problems where the number of clusters is unknown. We summarize all four algorithms within the context of our proposed framework, which involves splitting \ofm and merging \mfo associations.

\subsection{X-means \cite{pelleg2000x}}
\label{app:xmeans}

This method begins with the number of clusters $k$ set to a lower bound and then iteratively splits clusters based on changes in the Bayesian Information Criterion (BIC) score. The BIC score is given by:
\begin{align*}
BIC &\doteq - \frac{n_i}{2} \log 2\pi - \frac{n_i m}{2} \log \sigma^2 \\
&\quad - \frac{n_i - k}{2} + n_i \log \frac{n_i}{n} - \frac{k}{2} \log n,
\end{align*}
where $k$ is the number of clusters, $n_i$ is the size of the $i^{\text{th}}$ cluster, $n$ is the total size of the dataset, $m$ represents the dimension of the dataset, and $\sigma^2$ is given by:
\begin{align*}
\sigma^2 = \frac{1}{n_i - k} \sum_{i} (x_i - c_{y_i})^2,
\end{align*}
where $x_i$ is a data point, $c_{y_i}$ represents the centroid of the cluster to which $x_i$ belongs, and $y_i$ is the cluster index
.

The \xmeans algorithm consists of iteratively executing the following three steps (as illustrated in Figure \ref{fig:x-means}):

\begin{enumerate}
\item \textbf{Step 1:} Split each cluster $\mb\beta_i$ into two clusters, $\mb\beta_{i+}$ and $\mb\beta_{i-}$, by running a local 2-means algorithm.
\item \textbf{Step 2:} Calculate the BIC score for $\mb\beta_i$ and for both $\mb\beta_{i+}$ and $\mb\beta_{i-}$.
\item \textbf{Step 3:} Retain either the original cluster $\mb\beta_i$ or the split clusters $\mb\beta_{i+}$ and $\mb\beta_{i-}$ with the higher BIC score.
\end{enumerate}

The iterations continue until there are no changes in the higher BIC score and cluster centers.

\begin{figure*}
    \centering
    \includegraphics[width=0.7\textwidth]{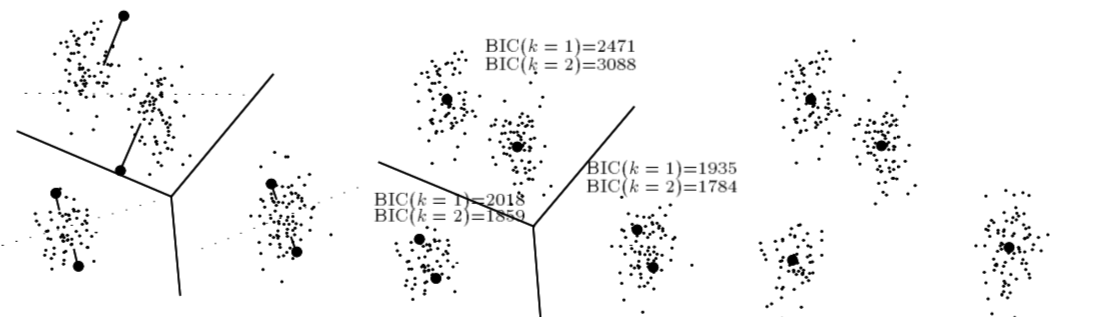}
    \caption{Framework for X-means \cite{pelleg2000x}}
    \label{fig:x-means}
    \end{figure*}

\subsection{Algorithm in \cite{muhr2009automatic}}\label{app:sm-xmeans}

This algorithm extends the \xmeans algorithm by incorporating a merge step. It considers several internal validity indices $V(c)$, such as BIC, Calinski-Harabasz index, Hartigan index, and others, to determine when to stop the split or merge procedures.

The algorithm in \cite{muhr2009automatic} first splits clusters until it is validated to stop and then continues to merge clusters until it is validated to stop (as illustrated in Figure \ref{fig:smx-means}). During the merge step, the internal validity indices select two clusters with the highest average data sample similarity to merge.

\begin{figure*}
    \centering
    \includegraphics[width=0.8\textwidth]{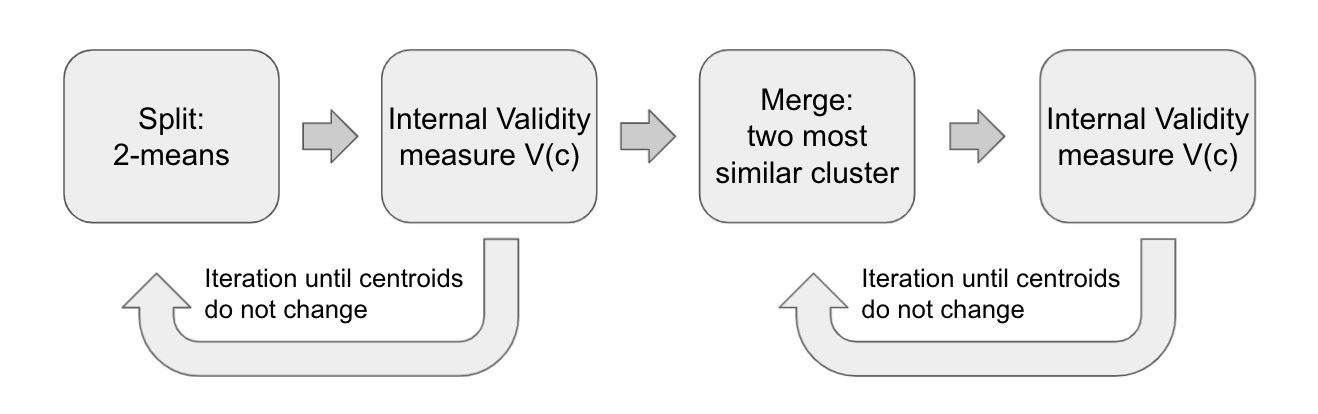}
    \caption{Framework for Algorithm \cite{muhr2009automatic}}
    \label{fig:smx-means}
    \end{figure*}

\subsection{Algorithm in \cite{morii2006clustering}}\label{app:for-kmeans}

This method assumes knowledge of the ground truth cluster number $k$. The algorithm consists of three steps:

\begin{itemize}
\item \textbf{Step 1:} Run Lloyd's \kmeans.
\item \textbf{Step 2:} Split each cluster into $m$ sub-clusters ($m=2,\cdots,M$) and calculate the ratio between successive \kmeans objectives $D$.
\begin{align*}
\label{eq.Fkmeans}
\rho(m) &= D^{(m)} / D^{(m-1)}, \\
\rho_k(m^*) &= \min \{\rho(m), m = 2,...,M \} .
\end{align*}
A cluster will be split if $ \rho_k(m^*) $ is below a certain threshold $\eta$.
\item \textbf{Step 3:} Keep the split cluster that is furthest from neighboring Voronoi regions and merge the rest of the split clusters into neighboring Voronoi regions.
\end{itemize}

This algorithm has very high computational complexity, as it splits the sub-cluster several times with different split numbers and determines the most suitable split number.

\subsection{Algorithm in \cite{lei2016robust}}\label{app:ds-kmeans}

This algorithm splits or merges cluster centers based on {\em intra-cluster dissimilarity} and {\em inter-cluster dissimilarity} defined below:
\begin{align*}
d_\text{inter} &= \left\Vert m_i - m_j\right\Vert^2, \\
d_\text{intra} &= \max \{\left\Vert\ m_i - x_p\right\Vert ^2 \} + \min \{\left\Vert m_i - x_p\right\Vert^2 \}.
\end{align*}

Here, $m_i$ and $m_j$ represent the cluster centers of the $i$-th and $j$-th clusters, while $x_p$ represents the data point in the corresponding cluster. Additionally, we define a threshold:
\begin{align*}
\overline{d} = \frac{1}{A^2_k} \sum_{i=1}^{k}\sum_{j=1}^{k} \left\Vert m_i - m_j\right\Vert^2.
\end{align*}

Here, $A^2_k$ is the number of pairwise cluster centers.

The algorithm \cite{lei2016robust} can be formalized as follows:

\begin{itemize}
\item \textbf{Split:} Iteratively update $\overline{d}$ and $d_\text{intra}$, and split the cluster until $d_\text{intra} < \overline{d}/2$.
\item \textbf{Merge:} Iteratively update $\overline{d}$ and $d_\text{inter}$, and merge the clusters until $d_\text{inter} >\overline{d}/2$.
\end{itemize}

\section{visualization of Benchmark dataset}\label{app:benchmark}
Figure \ref{fig:benchmark} provides the visual representation of the benchmark dataset described in Section \ref{exp:data}.
\begin{figure*}
  \centering
  \includegraphics[width=\textwidth]{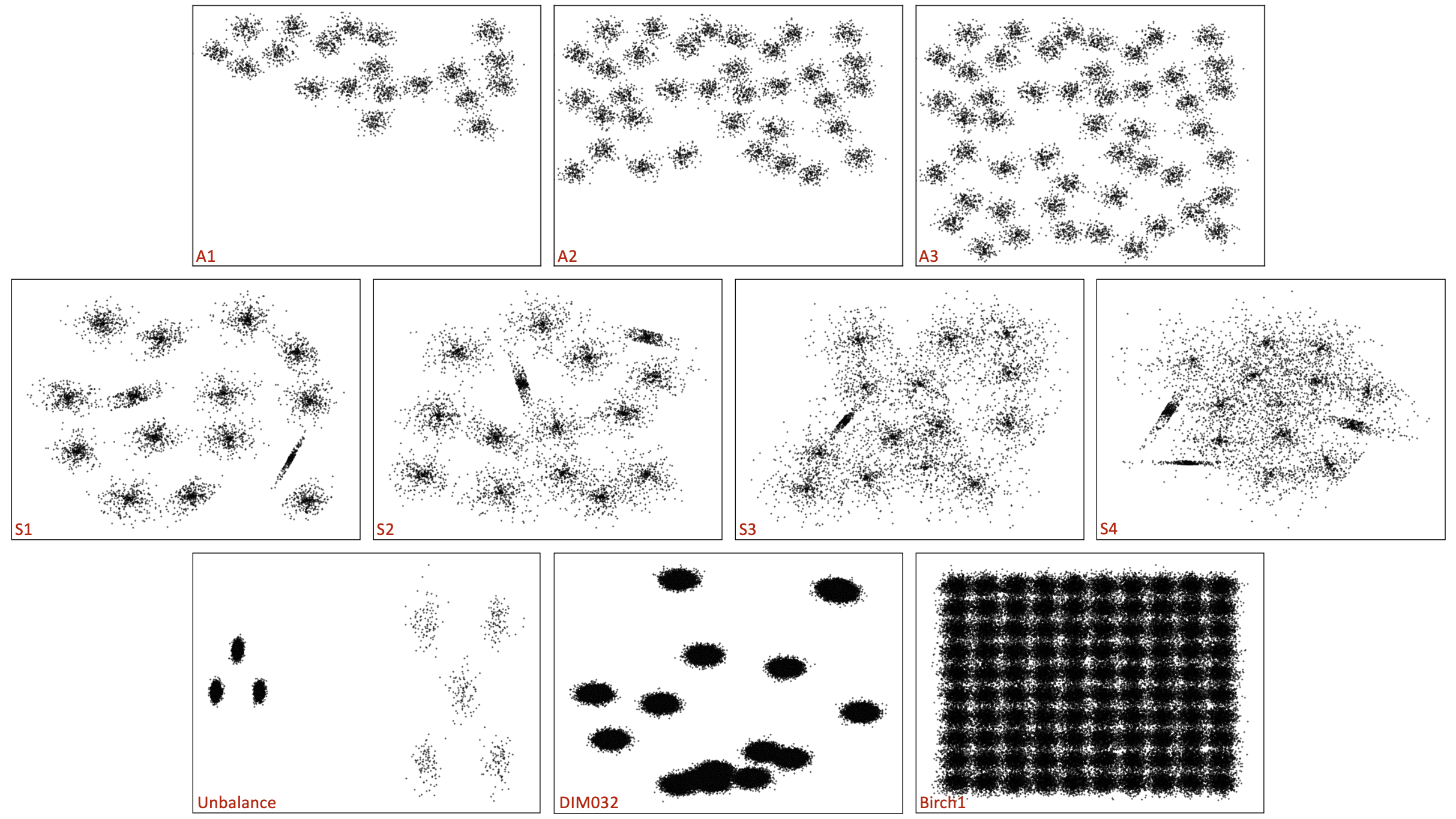}
  \caption{Visualization for the benchmark datasets~\cite{franti2018k}}\label{fig:benchmark}
  \end{figure*}

\section{Subroutines of Fusion-Fission \kmeans}\label{app:six_detect}

The results of $100$ trials on the benchmark dataset for the six different combinations of \ofm and \mfo subroutines of the Fission-Fusion $k$-means (FFkm) are presented in Tables \ref{tab:kmsminusplus} and \ref{tab:kmsminusplus_Ratio}. Our findings indicate that the objective-based FFkm (TD+OI) performs better than others on the heavily overlapped $S3$ and $S4$ datasets, demonstrating a limiting scenario for geometry-based algorithms. However, contributing to the flexibility of the FFkm framework, when scenarios involve high overlap, the combination of the \ofm detection subroutine, Total Deviation (TD), and the \mfo detection subroutine, Objective Increment (OI), denoted as TD+OI, may mitigate this problem.
    
\begin{table}
    \centering
    \caption{Success rate (\%) comparison}
    \scalebox{0.8}{
    \begin{tabular}{|c | c | c | c | c | c | c | c |}
    \hline 
    \textbf{Data Set}  & \textbf{SD+PD} & \textbf{SD+OI} & \textbf{TD+PD} & \textbf{TD+OI} & \textbf{RD+PD} & \textbf{RD+OI}\\ 
    \hline
    A1         & 100        &   100        & 100      & 100        &   100     &    99(0.00)     \\
    A2        & 100        &   100        & 100      & 100        &   100     &    97(0.00)     \\
    A3        & 100        &   100        & 100      & 100        &   100     &    98(0.00)             \\
    S1        & 100        &   100        & 100      & 100        &   100     &    95(0.00)            \\
    S2         & 100        &   100        & 100      & 100        &   100     &    99(0.00)           \\
    S3         & 77(0.02)   &   89(0.00)   & 87(0.01) & 96(0.00)   &   89(0.01)  &    92(0.01)                 \\
    S4        & 31(0.05)   &   39(0.04)   & 43(0.04) & 90(0.01)   &   41(0.05)   &    77(0.02)              \\
    Unbalance  & 100        &   100        & 100      & 100        &   100   &    99(0.00)      \\
    Dim032      & 100        &   100        & 100      & 100        &   100     &    100        \\
    Birch1    & 100        &   100        & 100      & 100        &   100     &    100           \\
    \hline
    \end{tabular}}
    \label{tab:kmsminusplus}
  \end{table}

  \begin{table}
    \centering
    \caption{$\rho$-ratio comparison}
    \scalebox{1.0}{
    \begin{tabular}{|c | c | c | c | c | c | c | c |}
    \hline 
    \textbf{Data Set} & \textbf{SD+PD} & \textbf{SD+OI} & \textbf{TD+PD} & \textbf{TD+OI} & \textbf{RD+PD} & \textbf{RD+OI}\\ 
    \hline
    A1         & $1.00\pm0.00$ & $1.00\pm0.00$ & $1.00\pm0.00$ & $1.00\pm0.00$ & $1.00\pm0.00$ &  $1.00\pm0.02$ \\
    A2         & $1.00\pm0.00$ & $1.00\pm0.00$ & $1.00\pm0.00$ & $1.00\pm0.00$ & $1.00\pm0.00$ &  $1.01\pm0.04$    \\
    A3         & $1.00\pm0.00$ & $1.00\pm0.00$ & $1.00\pm0.00$ & $1.00\pm0.00$ & $1.00\pm0.00$ &  $1.00\pm0.00$            \\
    S1         & $1.00\pm0.00$ & $1.00\pm0.00$ & $1.00\pm0.00$ & $1.00\pm0.00$ & $1.00\pm0.00$ &  $1.03\pm0.13$           \\
    S2        & $1.00\pm0.00$ & $1.00\pm0.00$ & $1.00\pm0.00$ & $1.00\pm0.00$ & $1.00\pm0.00$ &  $1.01\pm0.06$         \\
    S3        & $1.03\pm0.05$ & $1.01\pm0.04$ & $1.01\pm0.04$ & $1.00\pm0.00$ & $1.01\pm0.04$ &   $1.01\pm0.04$              \\
    S4        & $1.05\pm0.04$ & $1.05\pm0.05$ & $1.04\pm0.04$ & $1.01\pm0.02$ & $1.06\pm0.07$ &   $1.03\pm0.06$             \\
    Unbalance & $1.00\pm0.00$ & $1.00\pm0.00$ & $1.00\pm0.00$ & $1.00\pm0.00$ & $1.00\pm0.00$ &     $1.06\pm0.63$      \\
    Dim032    & $1.00\pm0.00$ & $1.00\pm0.00$ & $1.00\pm0.00$ & $1.00\pm0.00$ & $1.00\pm0.00$ &     $1.00\pm0.00$       \\
    Birch1    & $1.00\pm0.00$ & $1.00\pm0.00$ & $1.00\pm0.00$ & $1.00\pm0.00$ & $1.00\pm0.00$ &     $1.00\pm0.00$          \\
    \hline
    \end{tabular}}
    \label{tab:kmsminusplus_Ratio}
  \end{table}

\section{Steps Visulization of Fusion-Fission \kmeans} \label{app:FFkm_detail}
We provide a detailed visualization of each step of Algorithm~\ref{alg:1} (FFkm) in Figure~\ref{fig:demo}. In this visualization, the FFkm framework implements the detection subroutines Standard Deviation (SD) and Pairwise Distance (PD), referred to as FFkm (SD+PD). In this example, FFkm executes iteration $\ell=2$ to recover Lloyd's \kmeans from local minima to ground truth.

  \begin{figure*}
    \centering
    \setlength{\floatsep}{0pt}  % Reducing space between figures
    \setlength{\textfloatsep}{0pt} % Reducing space below/above floats to a minimum
    \setlength{\intextsep}{0pt} % Reducing in-text float space to a minimum
    
    % First row of images
    \subfloat[$\ell=1$, Step 1 (detect \ofm)]{  
      \includegraphics[width=0.333\linewidth]{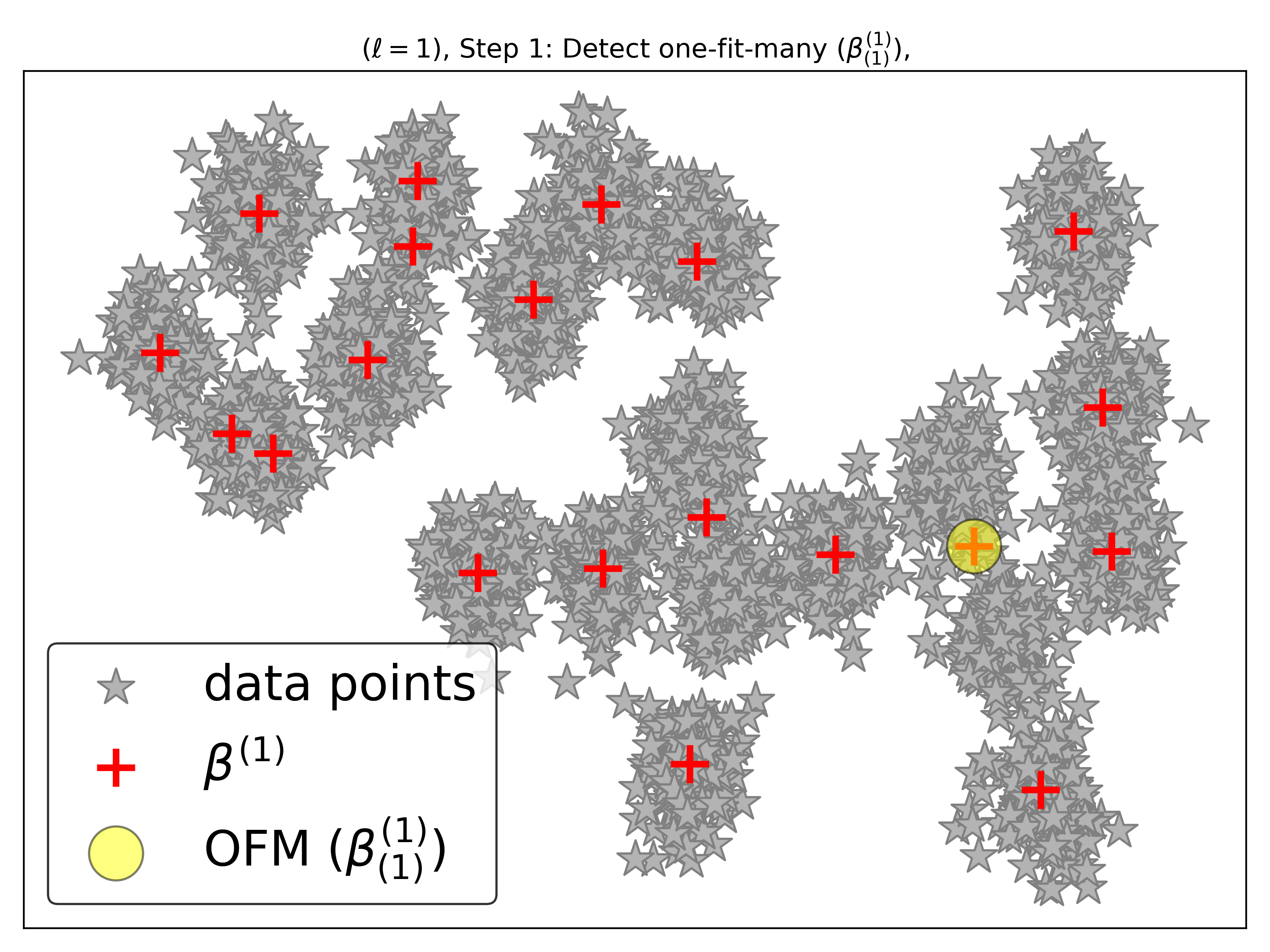}
    }
    \subfloat[$\ell=1$, Step 2a (split \ofm)]{  
      \includegraphics[width=0.333\linewidth]{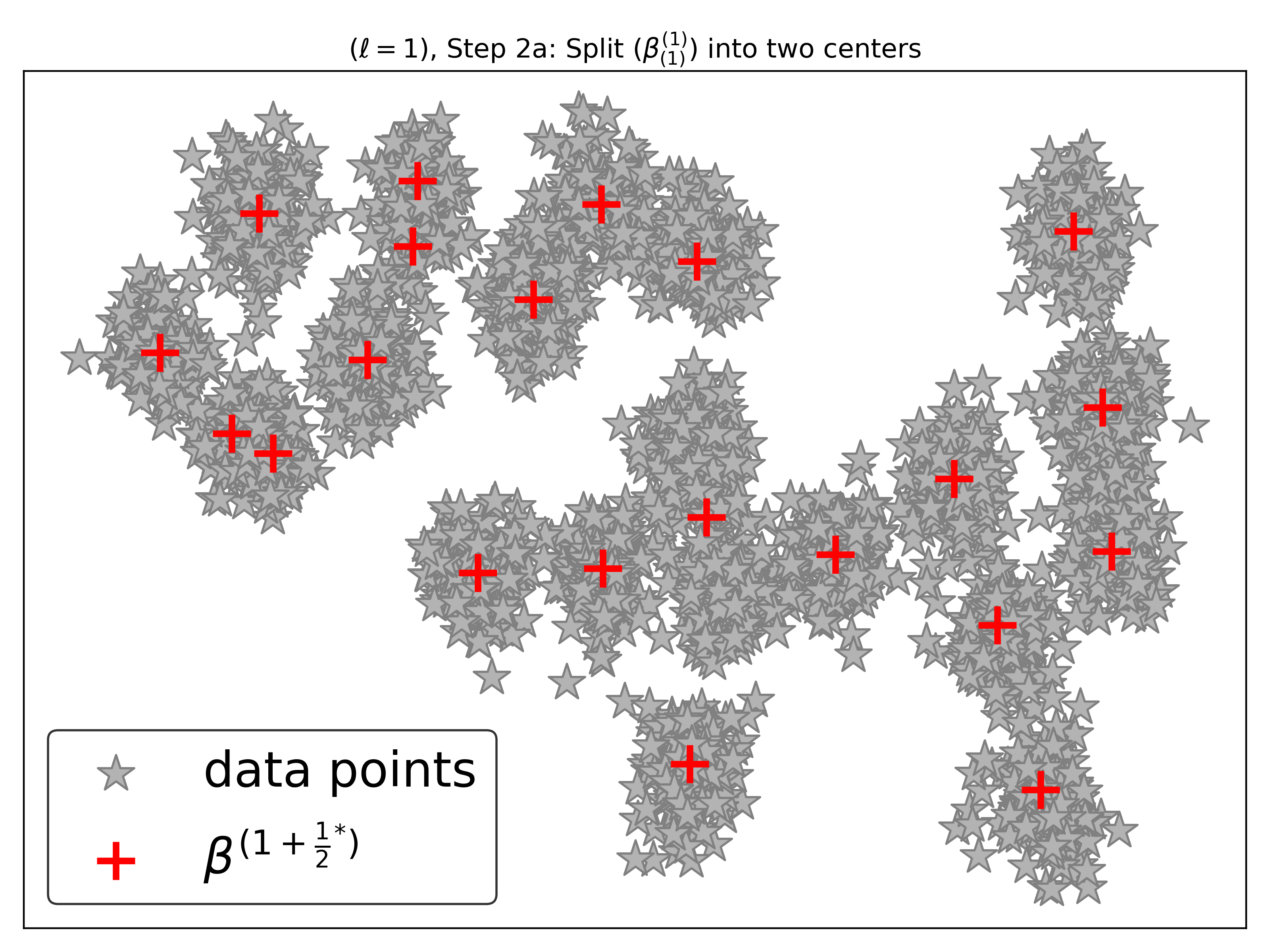}
    }
    %\hspace{0.5mm} % Adjustable space between the figures
    \subfloat[$\ell=1$, Step 2b (detect \mfo]{  
      \includegraphics[width=0.333\linewidth]{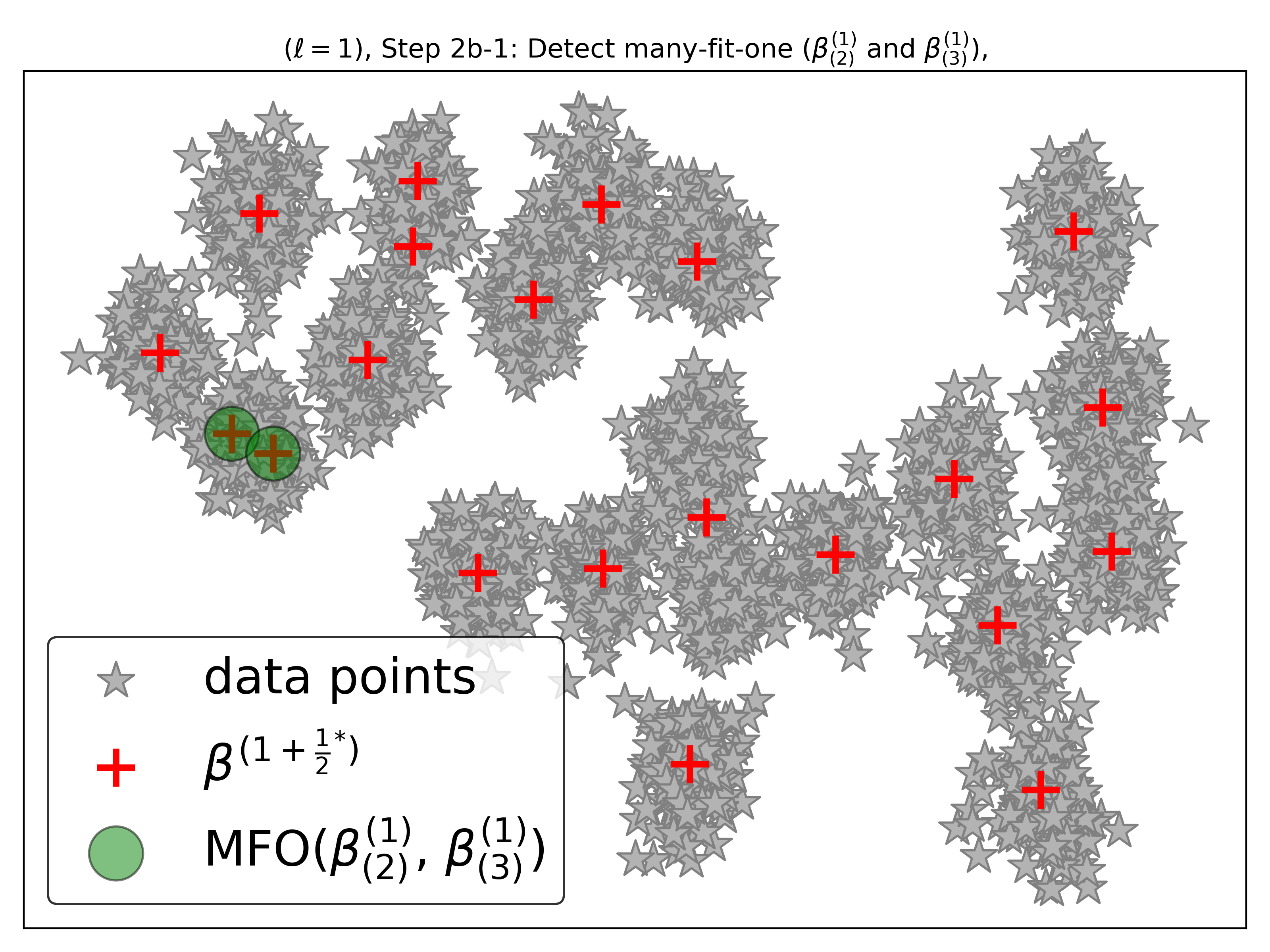}
    }
  
    % First row of images
    \subfloat[$\ell=1$, Step 2b (merge \mfo)]{  
      \includegraphics[width=0.333\linewidth]{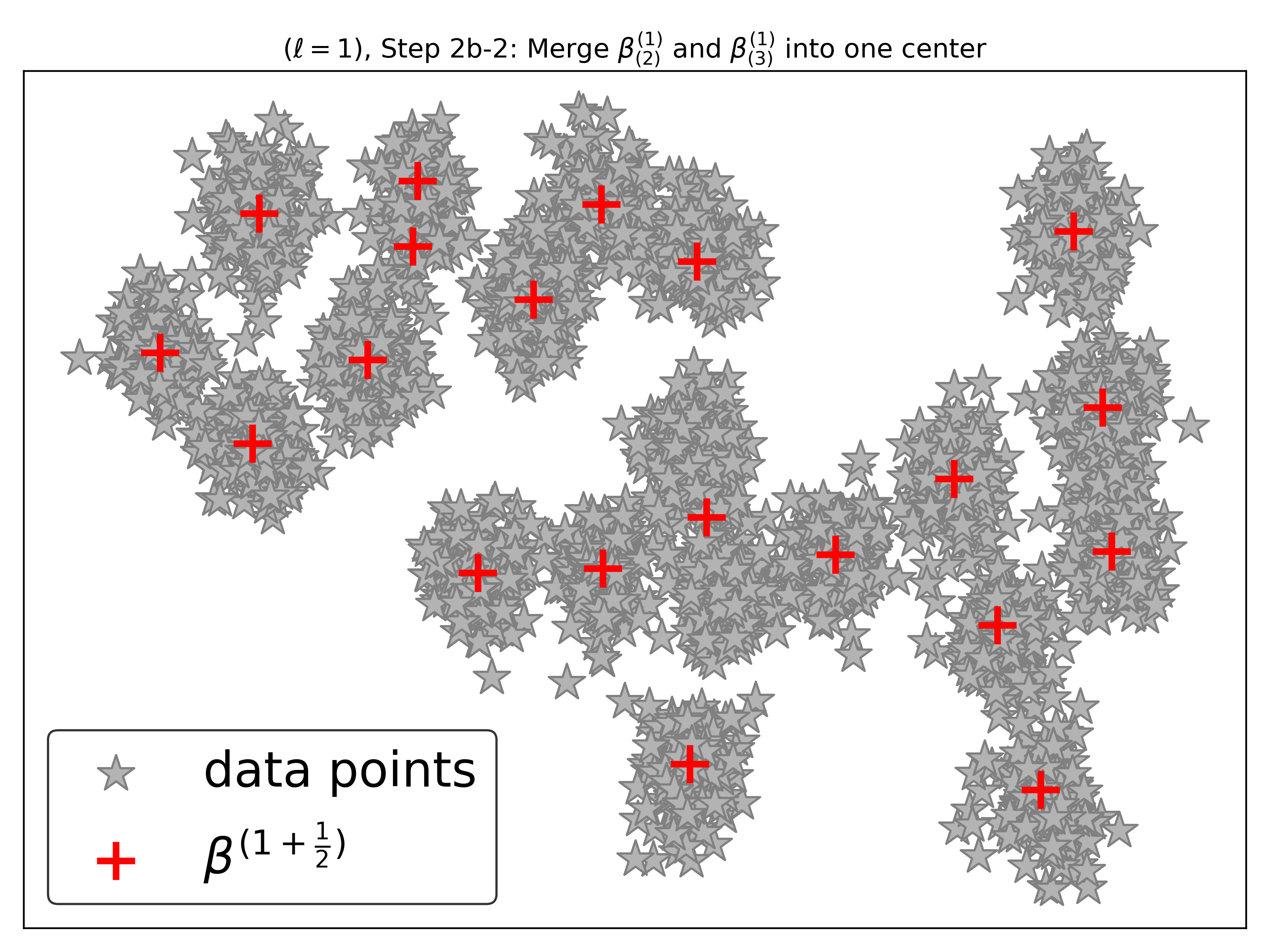}
    }
    \subfloat[$\ell=1$, Step 3]{  
      \includegraphics[width=0.333\linewidth]{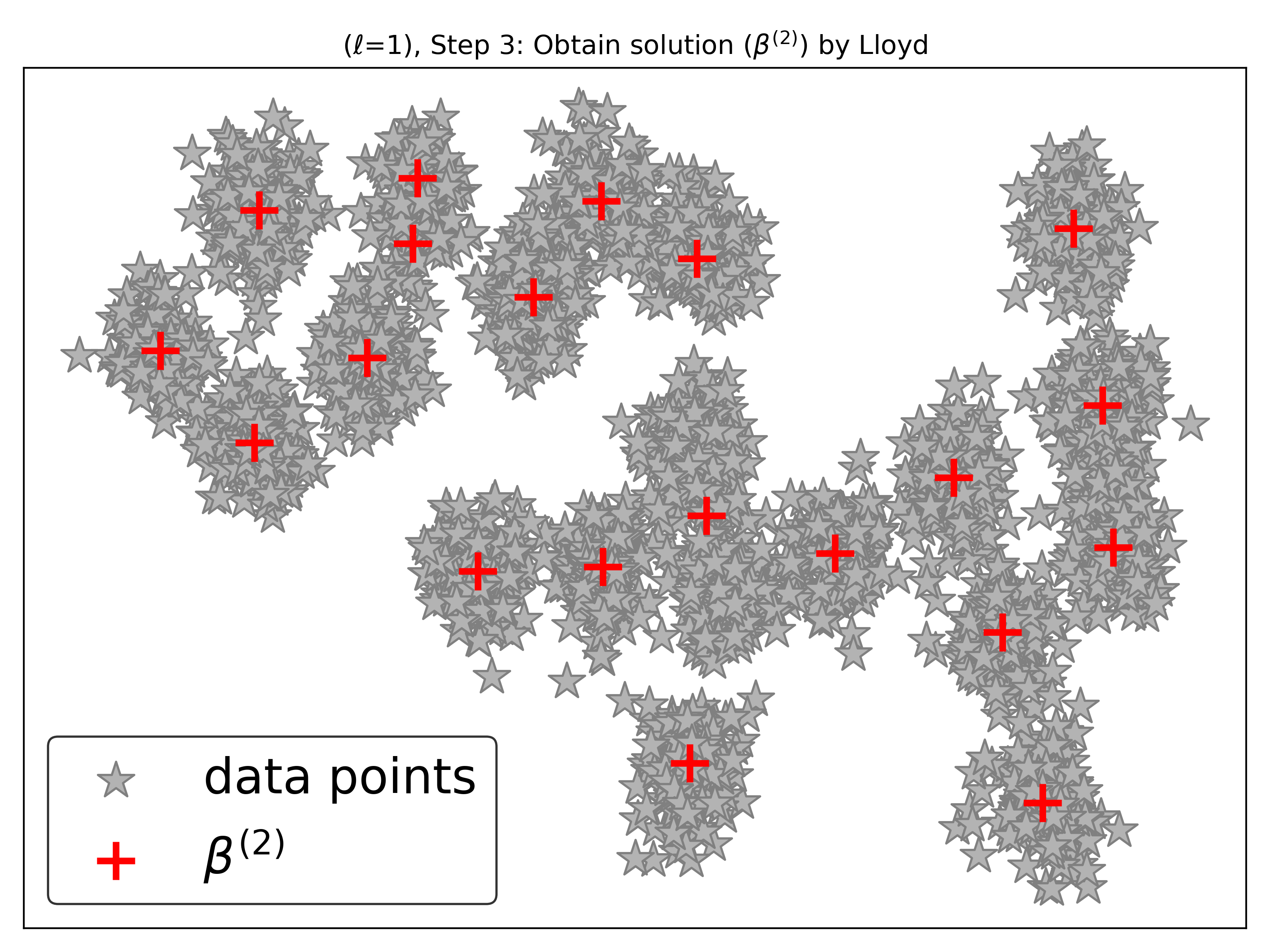}
    }
    %\hspace{0.5mm} % Adjustable space between the figures
    \subfloat[$\ell=2$, Step 1 (detect \ofm)]{  
      \includegraphics[width=0.333\linewidth]{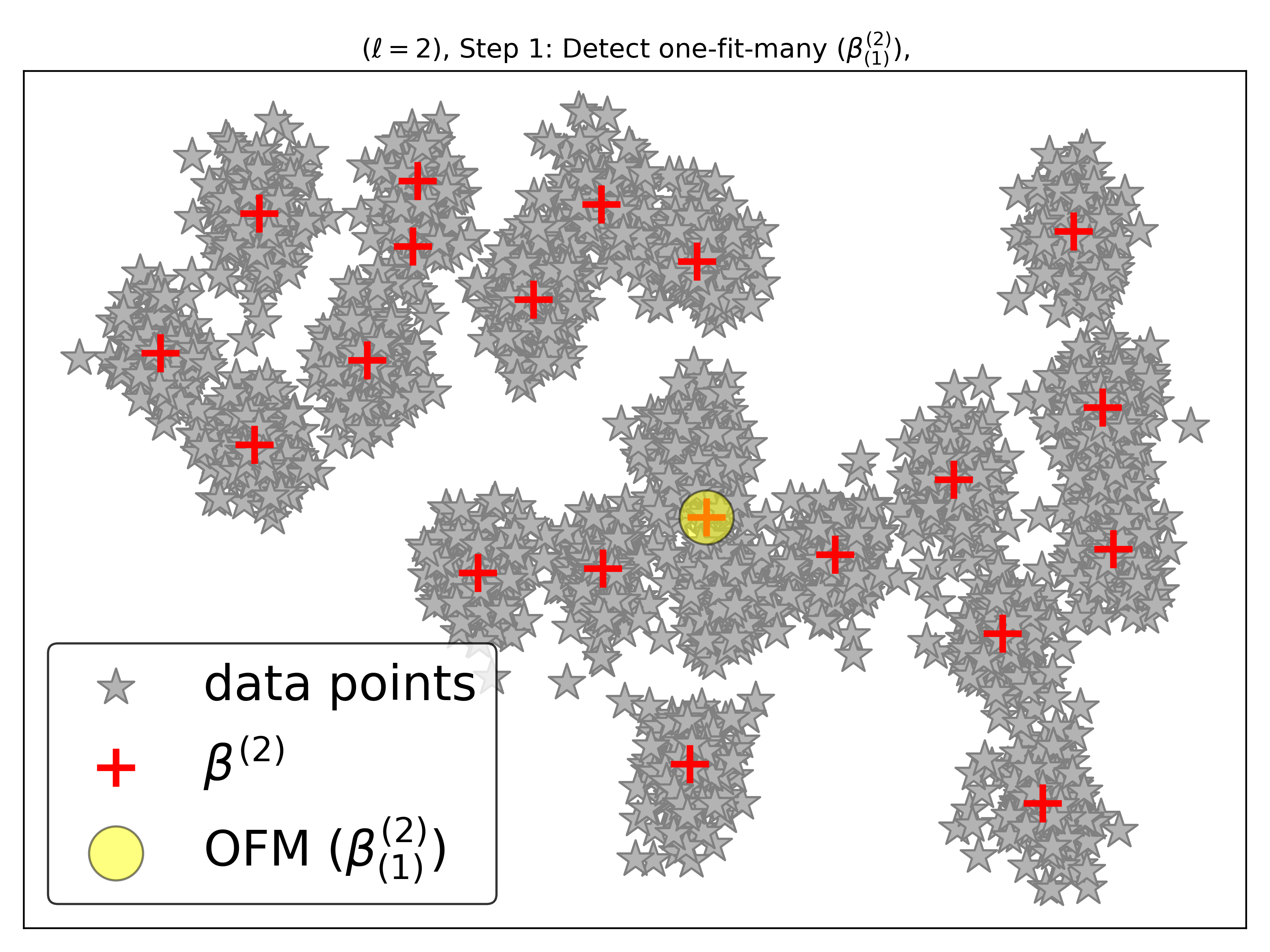}
    }
  
    % Second row of images
    %\hspace{0.5mm} % Adjustable space between the figures
    \subfloat[$\ell=2$, Step 2a (split \ofm)]{  
      \includegraphics[width=0.333\linewidth]{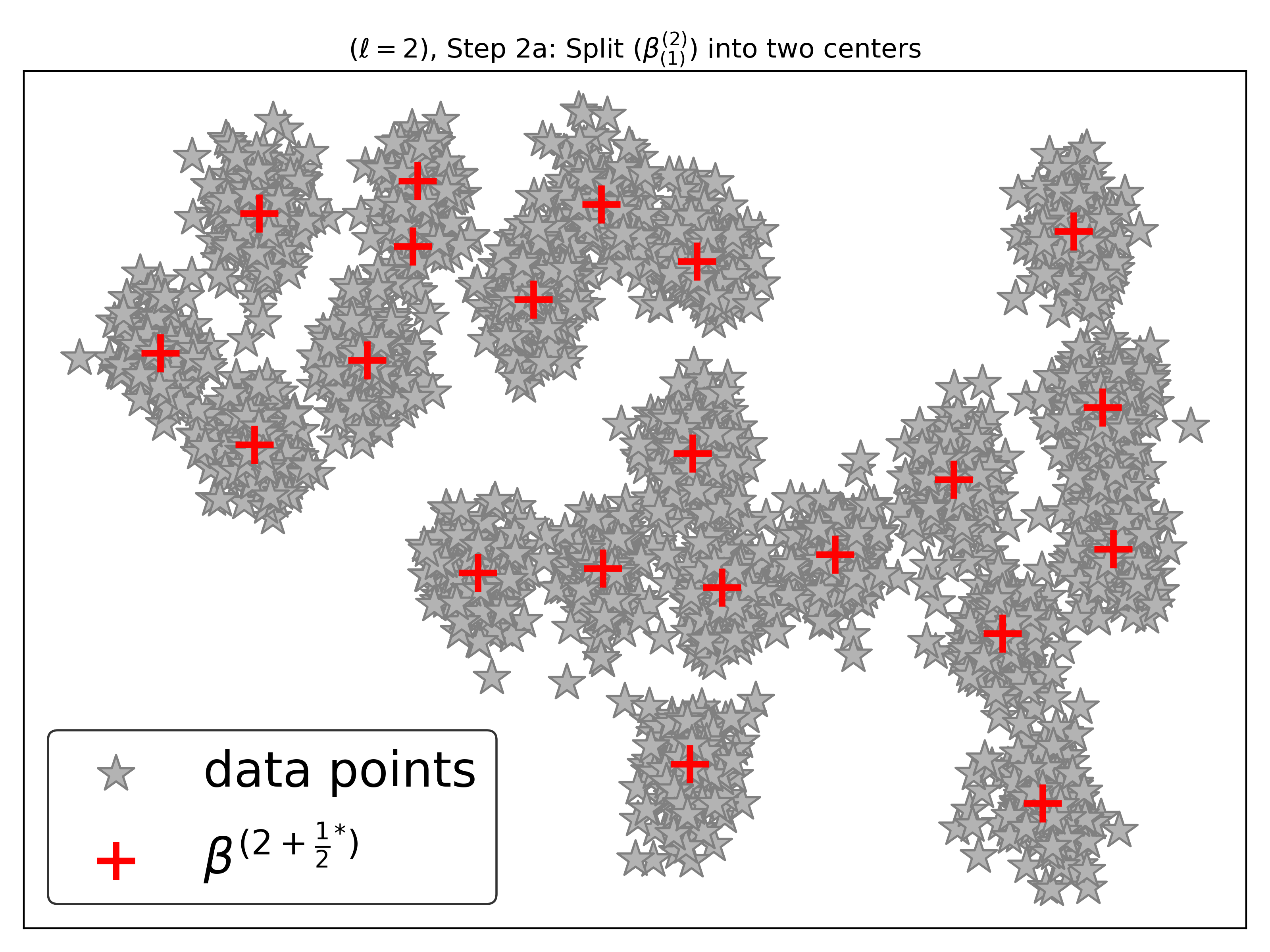}
    }
    \subfloat[$\ell=2$, Step 2b (detect \mfo]{  
      \includegraphics[width=0.333\linewidth]{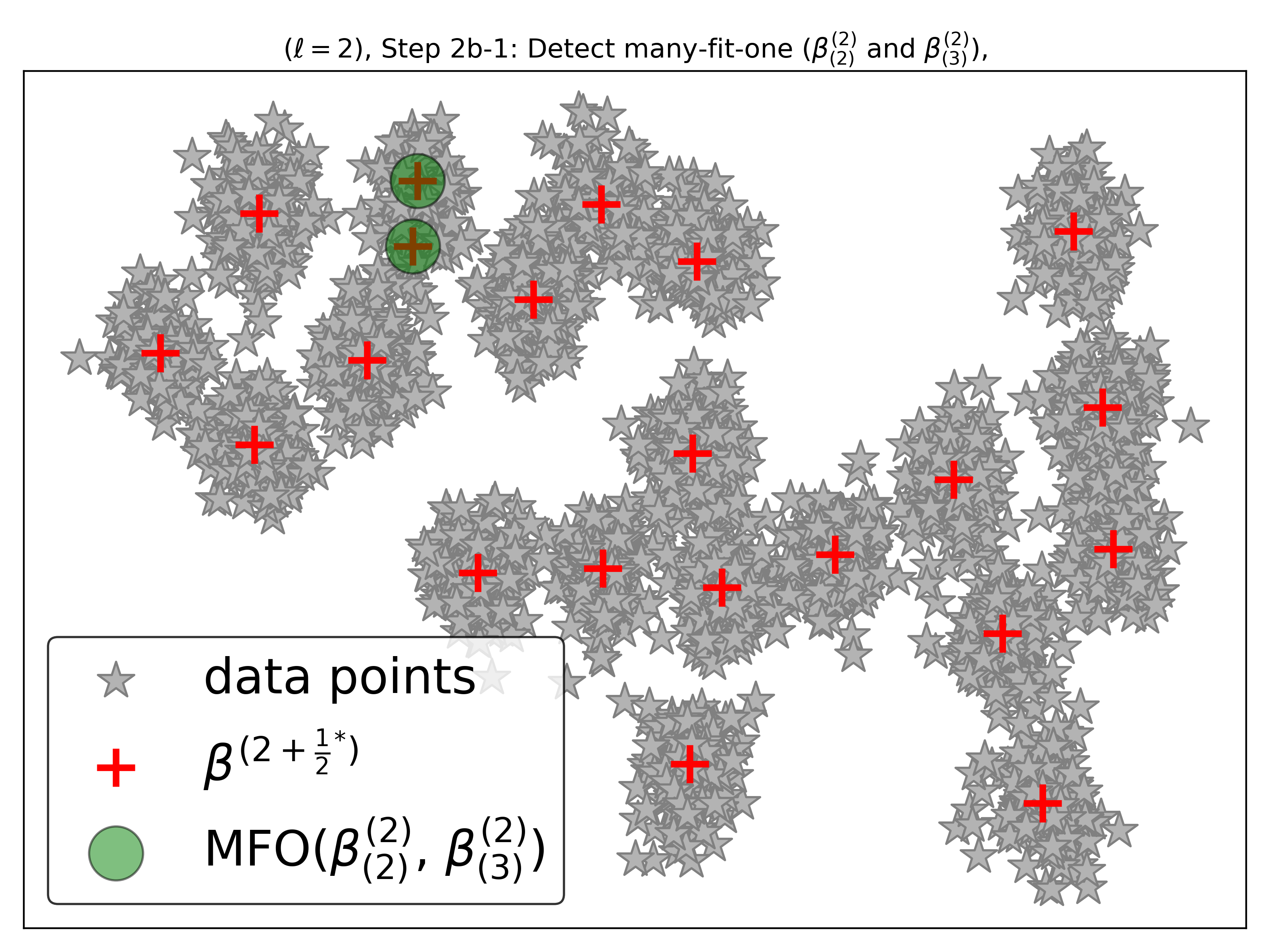}
    }
    %\hspace{0.5mm} % Adjustable space between the figures
    \subfloat[$\ell=2$, Step 2b (merge \mfo)]{  
      \includegraphics[width=0.333\linewidth]{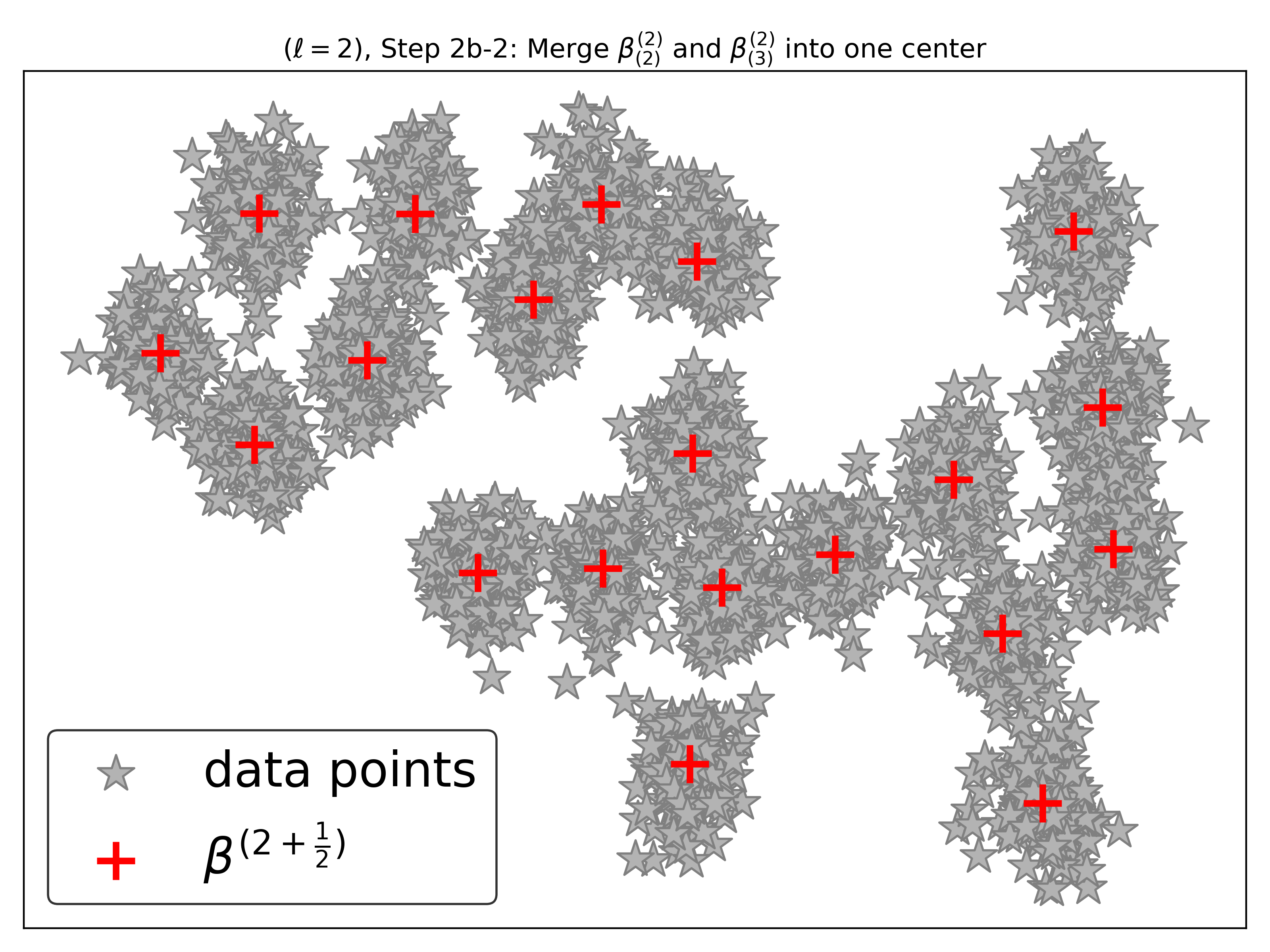}
    }

        % First row of images
    \subfloat[$\ell=2$, Step 3]{  
      \includegraphics[width=0.333\linewidth]{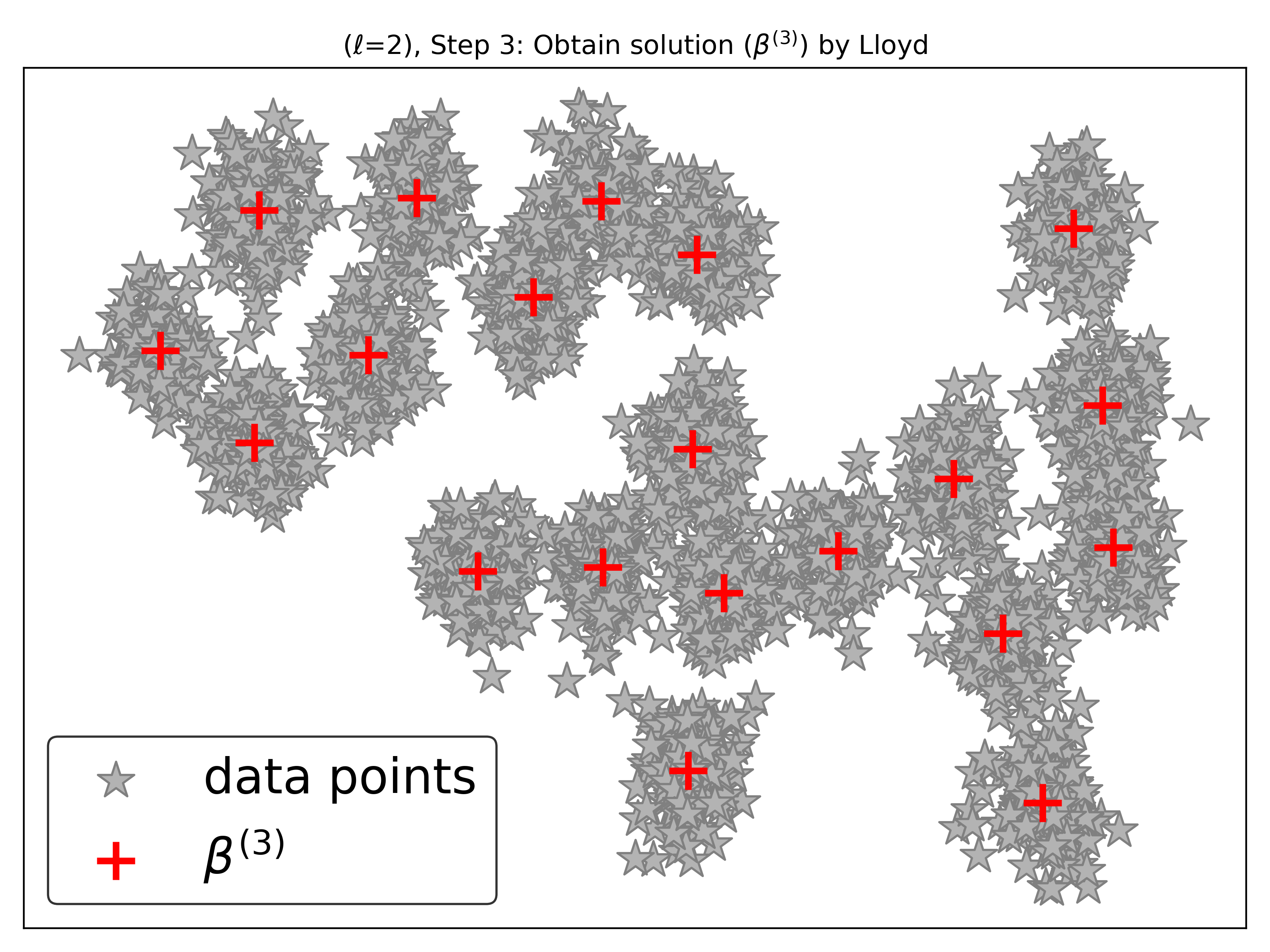}
    }
  
    \caption{(a)-(j) Visualization of the steps of the algorithm \ref{alg:1}: Fission-Fusion k-means (FFkm) at specific iterations. Detection subroutines involving \ofm and \mfo utilize standard deviation (SD) and pairwise distance (PD). (a-e) The steps 1 to 3 of the first iteration ($\ell=1$). (f-j) The steps 1 to 3 of the second iteration ($\ell=2$). Note: the positions of $\beta^{2+\frac{1}{2}}$ in (i) and $\beta^{3}$ in (j) differ slightly, though these differences may not be clearly apparent to the naked eye.}
    \label{fig:demo}
  \end{figure*}

\section{Image Visualization Supplement for Color Quantization} \label{app:images}
Figure~\ref{fig:appf} provides an additional image visualization of Color Quantization as a supplement to Figure~\ref{fig:5}. For Image Earth, except for Lloyd \kmeans, all other algorithms reveal a complete South American continent. For the Image Flower, Red Panda, Babbon, and Peppers, slight differences are difficult to discern with the naked eye but can be compared in Table~\ref{tab:9}.

  \begin{figure}
  \centering
      \includegraphics[width=\textwidth]{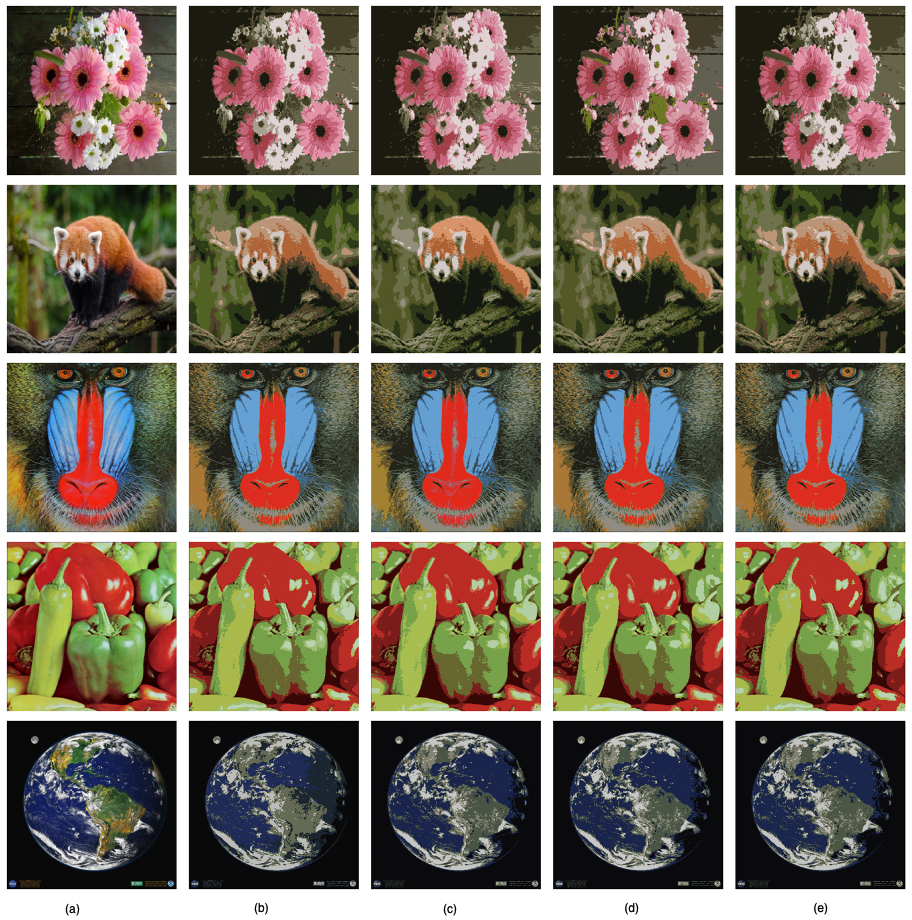} 
    \caption{Results of unsupervised color quantization using different numbers of clusters ($k$ values for colors).  The images are organized in rows from top to bottom: Flower ($k=8$), Red Panda ($k=8$), Baboon ($k=10$), Peppers($k=10$), and Earth ($k=5$).  Each column shows (a) The original image ($k$ is provided in Table~\ref{tab:9}).  (b) The result of Lloyd's \kmeans.  (c) The result of FFkm (SD+PD).  (d) The result of FFkm (TD+OI).  (e) The result of \kmeansmpx.}
  \label{fig:appf}
\end{figure}

\section{Additional experiment results for the Unbalance and S4 datasets} \label{app:discuss}

In Tables~\ref{tab:11} and~\ref{tab:S4com}, we provide additional experiment results for the Unbalance dataset and the S4 dataset, respectively. See Sections~\ref{sec:disc} for the setup.

\begin{table}
  \centering
  \caption{Lloyd's algorithm with over-parameterized $k$ on Unbalance}
  \scalebox{0.88}{
    \begin{tabular}{|c | c | c | c | c | c | c | c | c | c |}
    \hline
    $\textbf k$ & \textbf{$5\ktrue$} & \textbf{$6\ktrue$} & \textbf{$7\ktrue$} & \textbf{$8\ktrue$} & \textbf{$9\ktrue$} &\textbf{$10\ktrue$} & \textbf{$15\ktrue$} &\textbf{$20\ktrue$} \\ \hline
    SR (\%)& 11    & 22    & 38    & 55    & 65    & 75    & 96    & \textbf{99} \\
    AMR   & 0.28  & 0.23  & 0.16  & 0.11  & 0.09  & 0.06  & 0.01  & \textbf{0.00} \\
    $\rho$-ratio & $4.99\pm2.88$  & $4.00\pm2.65$  & $2.97\pm2.30$  & $2.39\pm2.08$  & $2.04\pm1.79$  & $1.65\pm1.34$  & $1.11\pm0.55$  & \textbf{$1.03\pm0.30$} \\ \hline
    \end{tabular}}
  \label{tab:11}
\end{table}

\begin{table}
  \centering
  \caption{$\epsilon$-Radius Detection Subroutine on S4 dataset}
  \scalebox{0.8}{
    \begin{tabular}{|c | c | c | c | c | c | c | c | c | c |}
    \hline
$\delta$ & \textbf{0.001} & \textbf{0.01} & \textbf{0.05}  & \textbf{0.1}   & \textbf{0.25}  & \textbf{0.5}   & \textbf{1} & \textbf{2}     & \textbf{5} \\ \hline
    SR (\%)& {44}    & 41    & 42    & 41    & 44    & 40    & 43    & 39    & 33 \\
    AMR   & {0.04}  & 0.05  & 0.05  & 0.05  & 0.04  & 0.05  & 0.05  & 0.04  & 0.05 \\
    $\rho$-ratio & {$1.05\pm0.07$}  & $1.06\pm0.07$  & $1.06\pm0.07$  & $1.06\pm0.07$  & $1.06\pm0.07$  & $1.07\pm0.08$  & $1.06\pm0.07$  & $1.05\pm0.05$  & $1.05\pm0.04$ \\
    \hline
    \end{tabular}}
  \label{tab:S4com}%
\end{table}%

\end{document}